\documentclass[11pt]{article}
\usepackage[margin=1in]{geometry}
\usepackage{amsmath,amssymb,amsthm}
\usepackage[showonlyrefs]{mathtools}
\usepackage{xcolor,enumitem}
\usepackage{tikz}
\usetikzlibrary{arrows.meta,positioning}
\usepackage{upgreek}
\usepackage[numbers,sort&compress]{natbib}
\usepackage{hyperref}
\definecolor{jade}{HTML}{00A86B}
\hypersetup{hypertexnames=false,colorlinks=true,pdfborder={0 0 0},allcolors=jade}
\definecolor{wpgblue}{RGB}{31,56,100}

\theoremstyle{plain}
\newtheorem{theorem}{Theorem}
\newtheorem{proposition}{Proposition}
\newtheorem{lemma}{Lemma}
\newtheorem{corollary}{Corollary}
\theoremstyle{definition}
\newtheorem{definition}{Definition}
\newtheorem{assumption}{Assumption}
\newtheorem{example}{Example}
\theoremstyle{remark}
\newtheorem{remark}{Remark}

\newcommand{\R}{\mathbb R}
\newcommand{\E}{\mathbb E}
\newcommand{\dd}{\mathrm d}
\newcommand{\KL}{\mathrm{KL}}
\newcommand{\RelEnt}{\mathcal K}
\newcommand{\wdist}{\mathsf d}
\newcommand{\cS}{\mathcal S}
\newcommand{\cA}{\mathcal A}

\newcommand{\Lip}{\mathrm{Lip}}
\newcommand{\osc}{\operatorname{osc}}
\newcommand{\esssup}{\operatorname*{ess\,sup}}
\newcommand{\essinf}{\operatorname*{ess\,inf}}
\newcommand{\Law}{\operatorname{Law}}

\DeclareMathOperator{\diver}{div}

\newcommand{\factref}[1]{(\hyperlink{fact:#1}{#1})}
\newcommand{\itemref}[1]{\hyperref[#1]{(\ref*{#1})}}
\makeatletter
\newcommand{\itemlabel}[2]{\item[(#2)]\phantomsection\def\@currentlabel{#2}\label{#1}}
\makeatother

\begin{document}
\title{Global Convergence of Wasserstein Policy Gradient for Entropy-Regularized Reinforcement Learning}
\author{Zhaoyu Zhu\thanks{University of Illinois Urbana-Champaign. Email: \texttt{zhaoyuz4@illinois.edu}}
\and Rui Gao\thanks{The University of Texas at Austin. Email: \texttt{rui.gao@mccombs.utexas.edu}}
\and Shuang Li\thanks{The Chinese University of Hong Kong, Shenzhen. Email: \texttt{shuang.li@cuhk.edu.cn}}}
\date{}
\maketitle

\begin{abstract}

Wasserstein policy gradient (WPG) optimizes stochastic policies in the optimal-transport geometry of action distributions. In entropy-regularized reinforcement learning, WPG combines transport driven by gradients of the soft Q-function with entropy-induced diffusion. The resulting dynamics are coupled through Bellman policy evaluation, and the policy optimization objective is generally nonconvex. We establish a unified global convergence theory of WPG for infinite-horizon discounted Markov decision processes with general state spaces and continuous action spaces. 
We identify a statewise gradient-dominance inequality by exploiting the Bellman structure together with log-Sobolev inequalities, which we show to hold uniformly along the trajectory. This inequality relates the squared norm of the Wasserstein policy gradient to the value gap in the sup norm.
In continuous time, we prove well-posedness of the flow and exponential convergence of the value gap in the sup norm and of the policy in statewise KL divergence. In discrete time, by additionally controlling the regularity and the discretization error, we obtain geometric convergence to a neighborhood of the optimum with constant step sizes and an $O(\varepsilon^{-1})$ iteration complexity with diminishing step sizes for fixed problem parameters. Finally, we establish faster local convergence rates for both dynamics, governed by the geometry of the optimal policy.
\end{abstract}

\section{Introduction}\label{sec:intro}

Policy gradient (PG) methods have been widely used in reinforcement learning (RL) since the development of REINFORCE~\cite{williams1992reinforce} and the policy gradient theorem~\cite[Theorem~1]{sutton1999policy}. The sensitivity of Euclidean PG to step-size choices and policy parameterization has motivated methods that account for the geometry of policy space. Natural policy gradient (NPG) uses the metric induced by the Fisher information matrix~\cite{amari1998natural,kakade2001natural}. Trust region policy optimization constrains the average Kullback--Leibler (KL) divergence between successive policies~\cite{schulman2015trpo}, while proximal policy optimization uses a clipped objective or a KL penalty to approximate this constraint~\cite{schulman2017ppo}. Policy mirror descent extends NPG to proximal updates based on general Bregman divergences~\cite{lan2023mirror,xiao2022rates,zhan2023regularized}.

Optimal transport offers another geometry for distributions that is natural for many settings involving continuous actions. The Wasserstein distance accounts for distances between actions. Unlike KL-based updates, which reweight action probabilities, transport-based updates move sampled actions directly, making them suitable for expressive generative policies~\cite{lv2025flowrl,koo2026drifting}. Existing approaches to Wasserstein policy optimization include particle approximations of Wasserstein gradient flows~\cite{zhang2018wpg}, kernel approximations of Wasserstein natural gradients~\cite{arbel2020kernelized,moskovitz2021wasserstein}, trust-region formulations~\cite{song2023metric,terpin2022transport}, Wasserstein gradient updates projected onto explicit policy classes~\cite{pfau2025wasserstein}, and approximations of Jordan--Kinderlehrer--Otto (JKO) schemes for implicit policy classes~\cite{zhu2026proximal}.

Entropy regularization is also widely used in RL. Maximum-entropy RL augments expected rewards with policy entropy, encouraging exploration and robustness while giving rise to soft Bellman operators and Gibbs-type policy improvements~\cite{ziebart2010causal,haarnoja2018sac,geist2019regmdp}. In Wasserstein geometry, the entropy term induces diffusion~\cite{jordan1998jko}. Combining this diffusion with transport driven by the soft Q-function gives a Fokker--Planck equation for policy evolution and a corresponding Langevin-type actor update~\cite{richemond2017wasserstein,zhang2018wpg}.

In this paper, we study the convergence of Wasserstein policy gradient (WPG) for entropy-regularized discounted Markov decision processes (MDPs) with general state spaces and actions in $\R^d$. We view a policy as a state-indexed family of action distributions and optimize these distributions directly in Wasserstein geometry. We consider a bounded base reward with a quadratic action penalty and entropy temperature $\tau>0$. We write $Q^\pi$ for the soft Q-function of $\pi$. The continuous-time WPG flow evolves the action distribution $\pi_t(\cdot\mid s)$ at each state $s$ according to the Fokker--Planck equation
\begin{equation}\label{eq:WPGF}\tag{\texttt{WPGF}}
\partial_t\pi_t=-\diver_a\bigl(\pi_t\nabla_aQ^{\pi_t}\bigr)+\tau\Delta_a\pi_t.
\end{equation}
In the occupancy-weighted Wasserstein geometry (Section~\ref{sec:wpg}), \eqref{eq:WPGF} can be interpreted as the steepest-ascent flow of the entropy-regularized objective. As in NPG, the occupancy weights cancel from the ascent direction, yielding a statewise update. The dynamics across states remain coupled through policy evaluation.
The corresponding action process and its Euler--Maruyama discretization, with step size $\eta_k>0$, are
\begin{align}
\dd A^s_t&=\nabla_aQ^{\pi_t}(s,A^s_t)\,\dd t+\sqrt{2\tau}\,\dd B^s_t,&\pi_t(\cdot\mid s)&=\Law(A^s_t),\label{eq:particle-flow}\\
A^s_{k+1}&=A^s_k+\eta_k\nabla_aQ^{\pi_k}(s,A^s_k)+\sqrt{2\tau\eta_k}\,\xi^s_k,&\pi_k(\cdot\mid s)&=\Law(A^s_k),\label{eq:particle-update}
\end{align}
where $B^s$ is a standard Brownian motion and $\xi^s_k\sim\mathcal N(0,I_d)$ is independent of $A^s_k$. Equivalently, the discrete-time WPG update is
\begin{equation}\label{eq:WPG}\tag{\texttt{WPG}}
\pi_{k+1}(\cdot\mid s)=\Big(\big(\mathrm{Id}+\eta_k\nabla_aQ^{\pi_k}(s,\cdot)\big)_{\#}\,\pi_k(\cdot\mid s)\Big)*\mathcal N(0,2\tau\eta_kI_d),
\end{equation}
where $\#$ denotes pushforward and $*$ denotes convolution. Each step transports the action distribution along the gradient of the current soft Q-function and applies Gaussian convolution. Since the Gibbs policy $p^\pi_s\propto\exp(Q^\pi(s,\cdot)/\tau)$ satisfies $\nabla_aQ^\pi(s,\cdot)=\tau\nabla_a\log p^\pi_s$, \eqref{eq:WPG} takes the form of a Langevin step. 

The convergence theory of WPG is less developed than that of other PG methods. Under suitable parameterizations, including tabular direct and softmax policies, PG enjoys global convergence guarantees~\cite{agarwal2021policy,bhandari2021linear,bhandari2024global,mei2020global}. NPG and KL-based policy mirror descent coincide for tabular softmax policies and admit linear convergence guarantees under entropy regularization or with increasing step sizes~\cite{cen2022natural,ju2026strongly,khodadadian2022natural,lan2023mirror,lan2023block,xiao2022rates,yuan2023loglinear,zhan2023regularized}. Related results extend to general state and action spaces~\cite{ju2026general} and continuous time~\cite{kerimkulov2025fisherrao,muller2024geometry}. Convergence results have also been established for regularized policy gradient methods with two-layer neural networks~\cite{leahy2022meanfield,yamamoto2024actorcritic}. These results raise the question of whether analogous convergence guarantees can be established for WPG:
\begin{quote}
\itshape Do the WPG flow~\eqref{eq:WPGF} and its discrete-time update~\eqref{eq:WPG} admit global convergence guarantees with explicit, nonasymptotic rates?
\end{quote}

The main challenge is that WPG evolves policies in Wasserstein geometry, so convergence analyses developed for PG and NPG do not apply directly. PG analyses use gradient-dominance inequalities formulated for Euclidean gradients, whereas NPG and policy mirror descent exploit Bregman geometry. Wasserstein geometry is not of Bregman type; moreover, classical evolution variational inequality arguments for Wasserstein gradient flows~\cite{ambrosio2005gradientflows} require geodesic convexity, which the RL objective need not~possess.

Related results on Wasserstein policy optimization study different schemes or objectives. Zhang et al.~\citep{zhang2018wpg} model policy optimization as Wasserstein gradient flows over parameter- and policy-distribution spaces, approximating the associated JKO schemes with particles. Their convergence guarantees assume a fixed Gibbs target and thus do not cover RL settings where the Q-function changes with the policy. Song et al.~\citep{song2023metric} study optimal-transport trust-region updates on finite state and action spaces without entropy regularization. Zhu et al.~\citep{zhu2026proximal} propose an operator-splitting method with Gaussian diffusion for entropy-regularized RL. Their theoretical analysis, however, concerns the Wasserstein proximal update and cannot apply to the WPG update.

A closer point of comparison is the literature on mean-field Langevin dynamics~\cite{chizat2018global,mei2018meanfield}, because the WPG update has a Langevin form. For convex, smooth probability functionals, functional inequalities and information-theoretic tools can yield gradient-dominance conditions that support global convergence~\cite{chizat2022mean,nitanda2022mfld}. These convexity-based analyses do not directly extend to WPG: the Bellman recursion makes the RL objective generally nonconcave in the policy~\cite{agarwal2021policy}. In addition, because the entropy term enters the soft Q-function recursively, it affects the regularity of the drift $\nabla_a Q^{\pi_t}$. A convergence analysis must therefore address both objective nonconcavity and the regularity of this policy-dependent drift. We note that mean-field Langevin dynamics have also been applied to entropy-regularized MDPs with neural-network policies~\cite{leahy2022meanfield,yamamoto2024actorcritic}, but their transport and diffusion act on network parameters rather than on the action distributions $\pi(\cdot\mid s)$ themselves.

\paragraph{Contributions.}
To the best of our knowledge, this work establishes the first unified global convergence theory for the continuous-time WPG flow~\eqref{eq:WPGF} and its explicit discretization~\eqref{eq:WPG} in entropy-regularized discounted MDPs with general state spaces and continuous actions.
\begin{itemize}[itemsep=0.5em]
\item
\emph{Bellman--Wasserstein gradient dominance.}
Despite the nonconcavity of the objective in the policy, by exploiting the Bellman structure of the RL objective, we establish a statewise gradient-dominance inequality on Wasserstein space, which also rules out nonoptimal stationary points (Proposition~\ref{prop:dominance}). 
Along the continuous-time flow, the values increase at least at the rate given by the squared norm of the Wasserstein policy gradient, so this inequality yields a differential inequality for the value gap. In discrete time, the same Bellman structure converts a one-step KL estimate into a recursion for the value gap.

\item
\emph{Global convergence of the continuous-time WPG flow.}
We first establish existence and uniqueness of a weak solution to the WPG flow~\eqref{eq:WPGF} (Theorem~\ref{thm:wellposed}). We then establish a uniform log-Sobolev inequality (LSI) for the moving family of Gibbs policies by showing monotone value improvement along the flow (Lemmas~\ref{lem:gibbs-bounds} and~\ref{lem:ct-dissipation}). Using the gradient-dominance inequality with this uniform constant, we prove exponential decay of the value gap in the sup norm (Theorem~\ref{thm:continuous}). The statewise KL divergence to the optimal policy decays at the same rate, uniformly over states.

\item
\emph{Global convergence of discrete-time WPG.}
For the discrete-time update, we additionally control the discretization error and use an inductive argument to establish uniform value bounds. These bounds ensure uniform regularity of the policy-dependent drift and a uniform LSI constant for the moving family of Gibbs policies along the iterates (Proposition~\ref{prop:dt-stability}). For constant step sizes $\eta$, we prove geometric convergence of the value gap up to an $O(\eta)$ error (Theorem~\ref{thm:discrete}). With diminishing step sizes, we obtain an $O(\varepsilon^{-1})$ iteration complexity for fixed problem parameters (Corollary~\ref{cor:diminishing}).

\item
\emph{Improved local rates.}
Near the optimum, we improve the global rates in two ways: we replace the uniform LSI constant required for all Gibbs policies along the trajectory with the state-uniform LSI constant of the optimal policy, and we remove the factor $1-\gamma$ (Theorems~\ref{thm:local-rate} and~\ref{thm:dt-local}). If the optimal soft Q-function $Q^\star$ is $\mu$-strongly concave in the action, the flow converges at rate $2\mu$, and discrete-time WPG can attain rates arbitrarily close to the same rate (Corollaries~\ref{cor:concave-optimum} and~\ref{cor:dt-concave}).
\end{itemize}

After an earlier version of this paper was submitted for conference review, we learned of the work \cite{siska2026note}. It establishes global convergence of the average value gap for the continuous-time flow on finite state spaces, assuming sufficiently regular solutions and coverage. Our work treats general state spaces, establishes well-posedness of weak solutions, and proves global and local sup-norm convergence in both continuous and discrete time without additional coverage assumptions. Remark~\ref{rem:concurrent-work} compares the convergence bounds and proof techniques in detail.

\paragraph{Organization.} Section~\ref{sec:model} introduces the model, the WPG dynamics and the standing assumptions. Section~\ref{sec:mechanism} develops the Bellman structure that both dynamics use, including the Bellman--Wasserstein gradient-dominance inequality. Section~\ref{sec:ct} establishes the well-posedness of the flow and proves its convergence. Section~\ref{sec:dt} analyzes discrete-time WPG, and Section~\ref{sec:local} proves the local rates of both. Auxiliary results and technical proofs are collected in the appendices.
Throughout, $|\cdot|$ is the Euclidean norm, $\mathbf 1$ is the constant function, and for a state-indexed function $f$ we write $f_s:=f(s,\cdot)$ for the function of the action at state $s$.

\section{Model Setup and Wasserstein Policy Gradient}\label{sec:model}

This section defines the entropy-regularized MDP and the two WPG dynamics studied in the paper. Section~\ref{sec:mdp} introduces the objective~\eqref{eq:objective} and the admissible policy class~\eqref{eq:admissible}. Section~\ref{sec:bellman} defines the soft Bellman operators, value functions, Gibbs policies, and Bellman resolvent. Section~\ref{sec:wpg} derives the Wasserstein policy gradient~\eqref{eq:velocity} and the continuous-time flow~\eqref{eq:WPGF}; Appendix~\ref{app:metric} gives the rigorous gradient representation. Section~\ref{sec:standing} states the assumptions used in the convergence analysis.

\subsection{MDP Setup}\label{sec:mdp}

We consider an infinite-horizon discounted MDP $(\cS,\cA,P,r,\gamma,\rho_0)$ with action space $\cA=\R^d$, discount factor $\gamma\in(0,1)$, initial distribution $\rho_0$, transition kernel $P(\dd s'\mid s,a)$ and measurable reward $r:\cS\times\cA\to\R$. A (stationary Markov) policy is a conditional distribution $\pi(\dd a\mid s)$ that admits a density, still denoted $\pi(\cdot\mid s)$ or $\pi_s$, with respect to Lebesgue measure. Given $\pi$, a trajectory $(s_n,a_n)_{n\ge0}$ is generated by $s_0\sim\rho_0$, $a_n\sim\pi_{s_n}$ and $s_{n+1}\sim P(\cdot\mid s_n,a_n)$, and $\E_{\rho_0,\pi}$ denotes expectation under the law of this trajectory.

The state space $\cS$ is a standard Borel space with a $\sigma$-finite reference measure $\mathfrak m$, and we assume that ${P(\dd s'\mid s,a)=p(s'\mid s,a)\,\mathfrak m(\dd s')}$ and $\rho_0\ll\mathfrak m$. Counting measure covers finite and countable state spaces; Lebesgue measure covers Euclidean ones. Functions of the state are identified up to $\mathfrak m$-null sets, statewise statements hold for $\mathfrak m$-a.e.\ $s$, and essential bounds over state--action pairs refer to $\mathfrak m\otimes\dd a$. We write ${L^\infty_{\cS}:=L^\infty(\cS,\mathfrak m)}$ with norm $\|\cdot\|_\infty$. Every discounted occupancy measure, defined in~\eqref{eq:occupancy} below, is then absolutely continuous with respect to $\mathfrak m$, so these conventions are compatible with policy evaluation.

\looseness=-1 Actions take values in the unbounded space $\R^d$, and we penalize their size through the regularized reward $\tilde r(s,a):=r(s,a)-\beta|a|^2/2$ with $\beta>0$. For a temperature $\tau>0$, the objective is
\begin{equation}\label{eq:objective}
J_{\rho_0}(\pi):=\E_{\rho_0,\pi}\Bigl[\sum_{n\ge0}\gamma^n\bigl(\tilde r(s_n,a_n)-\tau\log\pi(a_n\mid s_n)\bigr)\Bigr].
\end{equation}
Without the penalty, the objective would be unbounded above, since spreading an action law over ever larger regions of $\R^d$ makes its entropy arbitrarily large. The penalty also provides the log-Sobolev inequalities used in the analysis (Lemma~\ref{lem:gibbs-bounds}). Such control costs are common in continuous-control benchmarks~\cite{lillicrap2015continuous,brockman2016openai,tunyasuvunakool2020dm_control}.

\looseness=-1 For probability densities $\nu$ and $p$ with $p>0$, the Kullback--Leibler divergence, or relative entropy, is ${\KL(\nu\,\|\,p):=\int\nu\log(\nu/p)\,\dd a\in[0,\infty]}$, with $0\log0=0$. Between policy families, it is taken statewise. Let $\rho_\beta=\mathcal N(0,(\tau/\beta)I_d)$ denote the Gaussian reference, whose density and normalizing constant are
\begin{equation}\label{eq:gaussian-reference}
\rho_\beta(a)=Z_\beta^{-1}e^{-\beta|a|^2/(2\tau)},\qquad Z_\beta=(2\uppi\tau/\beta)^{d/2},
\end{equation}
the Gibbs density of the penalty at temperature $\tau$. Pointwise, $\tilde r-\tau\log\pi=r-\tau\log(\pi/\rho_\beta)+\tau\log Z_\beta$, and therefore
\begin{equation}\label{eq:onestep-reward}
\int\bigl(\tilde r(s,a)-\tau\log\pi(a\mid s)\bigr)\pi(a\mid s)\,\dd a=\int r(s,a)\pi(a\mid s)\,\dd a+\tau\log Z_\beta-\tau\KL(\pi_s\,\|\,\rho_\beta).
\end{equation}
Thus, up to the additive constant $\tau\log Z_\beta/(1-\gamma)$, $J_{\rho_0}$ is the KL-regularized objective with Gaussian reference $\rho_\beta$, and finite relative entropy with respect to $\rho_\beta$ is the natural integrability condition. Accordingly, the class of admissible policies is
\begin{equation}\label{eq:admissible}
\Pi:=\Bigl\{\pi:\ \pi_s\ll\dd a,\ (s,a)\mapsto\pi(a\mid s)\ \text{jointly measurable},\ \esssup_{s\in\cS}\KL(\pi_s\,\|\,\rho_\beta)<\infty\Bigr\}.
\end{equation}
By Lemma~\ref{lem:entropy-integrability}\itemref{itm:entropy-integrability-ii} in Appendix~\ref{app:entropy}, every $\pi\in\Pi$ has a finite second moment in the action and a finite integral $\int|\log\pi_s|\,\pi_s\,\dd a$, uniformly in the state, so, for bounded $r$, every term in~\eqref{eq:onestep-reward} is finite. 

\subsection{Soft Bellman Operators and Value Functions}\label{sec:bellman}

For a bounded measurable $V:\cS\to\R$ and $\pi\in\Pi$, define the transition operators $P$ and $P^\pi$, the one-step regularized reward $c^\pi$, the \emph{soft Q-function} $Q_V$ of $V$ and the soft evaluation operator $T^\pi$ by
\begin{equation}\label{eq:soft-operators}
\begin{aligned}
(PV)(s,a)&:=\int V(s')\,P(\dd s'\mid s,a),&(P^\pi V)(s)&:=\int(PV)(s,a)\,\pi(a\mid s)\,\dd a,\\
c^\pi(s)&:=\int\bigl(\tilde r(s,a)-\tau\log\pi(a\mid s)\bigr)\,\pi(a\mid s)\,\dd a,&Q_V(s,a)&:=\tilde r(s,a)+\gamma(PV)(s,a),\\
T^\pi V&:=c^\pi+\gamma P^\pi V=\int(Q_V-\tau\log\pi)\,\pi\,\dd a.
\end{aligned}
\end{equation}
Thus $Q_V(s,a)$ is the reward of action $a$ at $s$ when the future is valued by $V$, and $T^\pi V$ evaluates the one-step objective at $\pi_s$. The soft optimality operator and the \emph{Gibbs policy} of $V$ are
\begin{equation}\label{eq:gibbs}
T^\star V:=\tau\log\int e^{Q_V/\tau}\,\dd a,\qquad G[V](a\mid s):=e^{(Q_V(s,a)-(T^\star V)(s))/\tau}.
\end{equation}
By Lemma~\ref{lem:bellman-identities}\itemref{itm:bellman-identities-i} in Section~\ref{sec:mechanism}, $T^\star V$ is the maximum of $T^\pi V$ over $\pi\in\Pi$, attained only at $G[V]$. For bounded $r$ and $\pi\in\Pi$, $c^\pi$ is bounded by~\eqref{eq:onestep-reward}, and $T^\pi$ and $T^\star$ are monotone $\gamma$-contractions on $L^\infty_{\cS}$ (Lemma~\ref{lem:soft-bellman} in Appendix~\ref{app:entropy}).

For $\pi\in\Pi$, the \emph{value} $V^\pi$ is the fixed point of $T^\pi$: the expected discounted regularized reward as a function of the initial state, so that $J_{\rho_0}(\pi)=\int V^\pi\,\dd\rho_0$. Let $V^\star$ be the fixed point of $T^\star$, $Q^\star:=Q_{V^\star}$ and $\pi^\star:=G[V^\star]$. We call $V^\star$ the \emph{optimal value}, $Q^\star$ the \emph{optimal soft Q-function} and $\pi^\star$ the \emph{optimal policy}. Lemma~\ref{lem:bellman-identities} also shows that $\pi^\star\in\Pi$, $V^{\pi^\star}=V^\star$ and $V^\pi\le V^\star$ for all $\pi\in\Pi$. For $\pi\in\Pi$, we set
\begin{align}
Q^\pi&:=Q_{V^\pi},&p^\pi&:=G[V^\pi],&p^\pi_s&:=p^\pi(\cdot\mid s),\label{eq:gibbs-policy}\\
R^\pi&:=T^\star V^\pi-V^\pi,&W^\pi&:=V^\star-V^\pi,&E^\pi&:=\|W^\pi\|_\infty.\label{eq:residual}
\end{align}
We call $Q^\pi$ the \emph{soft Q-function} of $\pi$ and $p^\pi$ its \emph{Gibbs policy}, which is optimal for one step when the future is valued by $V^\pi$. We call $R^\pi$ the \emph{Bellman residual}, $W^\pi$ the \emph{value gap} and $E^\pi$ the \emph{value gap in the sup norm}. Since $V^\pi=T^\pi V^\pi$, $R^\pi(s)$ is the gain in the one-step objective at $s$ when $\pi_s$ is replaced by $p^\pi_s$. Along a path $(\pi_t)$ or a sequence $(\pi_k)$, we write $V_t:=V^{\pi_t}$, $R_t:=R^{\pi_t}$, $W_t:=W^{\pi_t}$ and $E(t):=E^{\pi_t}$, and similarly $V_k$, $R_k$, $W_k$ and $E_k$. Identities and inequalities along the flow that involve time derivatives hold for almost every $t>0$. The \emph{Bellman resolvent} and the normalized discounted occupancy measure are
\begin{align}
(I-\gamma P^\pi)^{-1}&=\sum_{n\ge0}\gamma^n(P^\pi)^n,\label{eq:resolvent}\\
d^\pi_{\rho_0}&:=(1-\gamma)\sum_{n\ge0}\gamma^n\rho_0(P^\pi)^n.\label{eq:occupancy}
\end{align}
The resolvent is a positive operator on $L^\infty_{\cS}$ that maps $\mathbf 1$ to $(1-\gamma)^{-1}\mathbf 1$, satisfies $(I-\gamma P^\pi)^{-1}f\ge f$ for $f\ge0$, and turns the one-step reward into the value: $V^\pi=(I-\gamma P^\pi)^{-1}c^\pi$ (Lemma~\ref{lem:soft-bellman}).

\subsection{Wasserstein Policy Gradient Dynamics}\label{sec:wpg}

We first derive the WPG velocity~\eqref{eq:velocity} and the flow~\eqref{eq:WPGF} formally. Appendix~\ref{app:metric} gives the rigorous treatment.

For a smooth family of policies $(\pi_\varepsilon)$ with $\pi_\varepsilon|_{\varepsilon=0}=\pi$ and $\delta\pi:=\frac{\dd}{\dd\varepsilon}\pi_\varepsilon|_{\varepsilon=0}$, the policy gradient theorem~\cite[Theorem~1]{sutton1999policy} gives, formally,
\begin{equation}\label{eq:pg}
\delta J_{\rho_0}(\pi)[\delta\pi]=\frac{1}{1-\gamma}\int_{\cS}\int_{\cA}\varphi^\pi(s,a)\,\delta\pi(a\mid s)\,\dd a\,d^\pi_{\rho_0}(\dd s),\qquad\varphi^\pi:=Q^\pi-\tau\log\pi.
\end{equation}
Wasserstein geometry describes a policy change by moving sampled actions. At a fixed state $s$, let $A^s\sim\pi_s$, move it with a velocity field $v$ to $A^s_\varepsilon:=A^s+\varepsilon v(s,A^s)$, and let $\pi_\varepsilon(\cdot\mid s)$ be the law of $A^s_\varepsilon$. The density then changes by $\delta\pi=-\diver_a(\pi v)$, and, under sufficient smoothness, integration by parts in $a$ turns~\eqref{eq:pg} into
\begin{equation}\label{eq:transport-derivative}
\frac{\dd}{\dd\varepsilon}J_{\rho_0}(\pi_\varepsilon)\Big|_{\varepsilon=0}=\frac{1}{1-\gamma}\int_{\cS}\int_{\cA}\big\langle\nabla_a\varphi^\pi(s,a),v(s,a)\big\rangle\,\pi(a\mid s)\,\dd a\,d^\pi_{\rho_0}(\dd s).
\end{equation}
Optimal transport measures the size of this policy change by the distance that the actions travel. For probability measures with finite second moments, ${W_2^2(\mu,\nu):=\inf\{\E|X-Y|^2:\ X\sim\mu,\ Y\sim\nu\}}$ is the squared Wasserstein distance on the set $\mathcal P_2(\R^d)$ of such measures, and the coupling $(A^s,A^s_\varepsilon)$ gives
\[
W_2^2\bigl(\pi_\varepsilon(\cdot\mid s),\pi_s\bigr)\le\E|A^s_\varepsilon-A^s|^2=\varepsilon^2\int|v(s,a)|^2\,\pi(a\mid s)\,\dd a.
\]
Weighting these statewise costs as in~\eqref{eq:transport-derivative} gives the occupancy-weighted kinetic energy of~$v$
\begin{equation}\label{eq:kinetic}
\frac{1}{1-\gamma}\int_{\cS}\int_{\cA}|v(s,a)|^2\,\pi(a\mid s)\,\dd a\,d^\pi_{\rho_0}(\dd s).
\end{equation}

With $\mu_\pi(\dd s,\dd a):=(1-\gamma)^{-1}d^\pi_{\rho_0}(\dd s)\,\pi(a\mid s)\,\dd a$, the first-order improvement~\eqref{eq:transport-derivative} is $\langle\nabla_a\varphi^\pi,v\rangle_{L^2(\mu_\pi)}$, and the kinetic energy~\eqref{eq:kinetic} is $\|v\|^2_{L^2(\mu_\pi)}$. By the Cauchy--Schwarz inequality, among velocity fields with a given kinetic energy, the first-order improvement is largest for the positive multiple of $\nabla_a\varphi^\pi$ with that energy. Hence the steepest-ascent direction of $J_{\rho_0}$ in this Wasserstein geometry is the \emph{WPG velocity}
\begin{equation}\label{eq:velocity}
v^\pi(s,a):=\nabla_a\varphi^\pi(s,a)=\nabla_aQ^\pi(s,a)-\tau\nabla_a\log\pi(a\mid s).
\end{equation}
As in NPG~\cite{kakade2001natural}, the shared occupancy weights make the ascent direction statewise. If $\rho_0$ is equivalent to $\mathfrak m$, then so is $d^\pi_{\rho_0}$, because $d^\pi_{\rho_0}\ge(1-\gamma)\rho_0$; hence~\eqref{eq:velocity} is the ascent direction of $J_{\rho_0}$ at $\mathfrak m$-a.e.\ state. Otherwise, $J_{\rho_0}$ does not depend on the policy at states of zero occupancy, and WPG extends its gradient there by the same formula (Remark~\ref{rem:zero-occupancy} in Appendix~\ref{app:metric}). Transporting $\pi_t$ with the velocity $v^{\pi_t}$ gives, by mass conservation, the continuity equation ${\partial_t\pi_t=-\diver_a(\pi_tv^{\pi_t})}$. Since $\pi_t\nabla_a\log\pi_t=\nabla_a\pi_t$, the entropy part of the velocity becomes the diffusion $\tau\Delta_a\pi_t$, and the continuity equation becomes~\eqref{eq:WPGF}:
\[
\partial_t\pi_t=-\diver_a\bigl(\pi_t\nabla_aQ^{\pi_t}\bigr)+\tau\diver_a\bigl(\pi_t\nabla_a\log\pi_t\bigr)=-\diver_a\bigl(\pi_t\nabla_aQ^{\pi_t}\bigr)+\tau\Delta_a\pi_t.
\]

For a positive density $p$ and a probability density $\nu$ on $\R^d$, define the relative Fisher information
\begin{equation}\label{eq:fisher}
\mathcal I(\nu\,\|\,p):=4\int_{\R^d}\left|\nabla_a\sqrt{\nu/p}\right|^2p\,\dd a.
\end{equation}
The value is $+\infty$ unless $\sqrt{\nu/p}$ has the indicated weak derivative with finite weighted norm. When it is finite, the relative score
\[
\nabla_a\log(\nu/p):=\frac{2\nabla_a\sqrt{\nu/p}}{\sqrt{\nu/p}}\qquad\text{on }\{\nu>0\}
\]
gives ${\mathcal I(\nu\,\|\,p)=\int|\nabla_a\log(\nu/p)|^2\nu\,\dd a}$ without requiring a classical logarithmic derivative. For policy families, $\mathcal I(\pi\,\|\,p^\pi)$ denotes the state function $s\mapsto \mathcal I(\pi_s\,\|\,p^\pi_s)$. Propositions~\ref{prop:first-variation} and~\ref{prop:weighted-gradient} in Appendix~\ref{app:metric} justify the derivation of~\eqref{eq:WPGF} above for admissible policies with finite Fisher information, defined using weak derivatives, and Corollary~\ref{cor:gradient-along-flow} there proves the gradient-flow identity
\[
\frac{\dd}{\dd t}J_{\rho_0}(\pi_t)=\frac{\tau^2}{1-\gamma}\int_{\cS}\mathcal I\bigl(\pi_t(\cdot\mid s)\,\|\,p^{\pi_t}_s\bigr)\,d^{\pi_t}_{\rho_0}(\dd s) \ \ge\ 0.
\]
Since $p^\pi_s$ is proportional to $\exp(Q^\pi(s,\cdot)/\tau)$ by~\eqref{eq:gibbs}, the score identity for the Gibbs policy
\begin{equation}\label{eq:gibbs-score}
\nabla_aQ^\pi(s,\cdot)=\tau\nabla_a\log p^\pi_s
\end{equation}
gives $v^\pi=-\tau\nabla_a\log(\pi_s/p^\pi_s)$. Hence, by~\eqref{eq:fisher}, the right-hand side of the gradient-flow identity is the kinetic energy~\eqref{eq:kinetic} of~$v^{\pi_t}$, so the objective increases at the rate of the squared norm of its gradient.

\subsection{Standing Assumptions}\label{sec:standing}

The following assumptions are used throughout, where $\pi_0$ is the initial policy of~\eqref{eq:WPGF} or~\eqref{eq:WPG}.

\begin{assumption}[Boundedness and initialization]\label{ass:base}
The reward bounds $r_-:=\essinf_{s,a}r(s,a)$ and $r_+:=\esssup_{s,a}r(s,a)$ are finite. The initial policy satisfies $\RelEnt_0:=\esssup_s\KL(\pi_0(\cdot\mid s)\,\|\,\rho_\beta)<\infty$.
\end{assumption}

\begin{assumption}[Action regularity]\label{ass:smooth}
\begin{enumerate}
\itemlabel{itm:smooth-a}{a} For a.e.\ $s$, $a\mapsto r(s,a)$ is $C^1$, and uniformly in $(s,a,\bar a)$,
\[
|\nabla_ar(s,a)|\le G_r,\qquad|\nabla_ar(s,a)-\nabla_ar(s,\bar a)|\le L_r|a-\bar a|.
\]
\itemlabel{itm:smooth-b}{b} For a.e.\ $s$, the density $p(s'\mid s,a)$ is jointly measurable, is a probability density in $s'$ for every $a$, and is $C^1$ in $a$ for all $s'$ outside one $\mathfrak m$-null set; moreover, uniformly in $(s,a,\bar a)$,
\[
\begin{aligned}
\int|\nabla_ap(s'\mid s,a)|\,\mathfrak m(\dd s')\le G_p,\quad
\int|\nabla_ap(s'\mid s,a)-\nabla_ap(s'\mid s,\bar a)|\,\mathfrak m(\dd s')\le L_p|a-\bar a|.
\end{aligned}
\]
\end{enumerate}
\end{assumption}

An example of a transition kernel that satisfies Assumption~\ref{ass:smooth}\itemref{itm:smooth-b} is as follows.
\begin{example}[Additive-noise transitions]\label{ex:additive}
Let $\cS=\R^{d'}$ with Lebesgue measure $\mathfrak m$, and let ${S'=F(s,a)+\Sigma\xi}$, where $\Sigma$ is an invertible $d'\times d'$ matrix and the noise $\xi$, independent of $(s,a)$, has a $C^2$ density $q$. Then $\Sigma\xi$ has the density ${q_\Sigma(x):=|\det\Sigma|^{-1}q(\Sigma^{-1}x)}$, and we assume $\|\nabla q_\Sigma\|_{L^1}<\infty$ and $\|D^2q_\Sigma\|_{L^1}<\infty$, with the Hessian measured in operator norm. Suppose $F$ is jointly measurable and $C^1$ in the action, with $|D_aF(s,a)|\le M_1$ and ${|D_aF(s,a)-D_aF(s,\bar a)|\le M_2|a-\bar a|}$ in operator norm, uniformly in $s$. Then
\[
p(y\mid s,a)=q_\Sigma\bigl(y-F(s,a)\bigr),\qquad\nabla_ap(y\mid s,a)=-D_aF(s,a)^\top\nabla q_\Sigma\bigl(y-F(s,a)\bigr).
\]
The $L^1$ norm is translation invariant, and
\[
|F(s,a)-F(s,\bar a)|\le M_1|a-\bar a|,\qquad\|\nabla q_\Sigma(\cdot-x)-\nabla q_\Sigma(\cdot-y)\|_{L^1}\le\|D^2q_\Sigma\|_{L^1}|x-y|.
\]
Hence Assumption~\ref{ass:smooth}\itemref{itm:smooth-b} holds with ${G_p=M_1\|\nabla q_\Sigma\|_{L^1}}$ and ${L_p=M_2\|\nabla q_\Sigma\|_{L^1}+M_1^2\|D^2q_\Sigma\|_{L^1}}$. Nondegenerate Gaussian noise satisfies these conditions.
\end{example}

\section{Bellman Structure and Wasserstein Gradient Dominance}\label{sec:mechanism}

This section develops the main connections between policy improvement and the value gap. Section~\ref{sec:bellman-identities} establishes the relevant Bellman identities and value bounds. Section~\ref{sec:dominance} combines these results with log-Sobolev inequalities for Gibbs policies to establish a Wasserstein gradient-dominance inequality. These results provide the foundation for the continuous- and discrete-time convergence analyses in later sections.

With $\RelEnt_0$ as in Assumption~\ref{ass:base}, set
\begin{equation}\label{eq:value-range}
\begin{gathered}
V_-:=\frac{r_-+\tau\log Z_\beta-\tau\RelEnt_0}{1-\gamma},\qquad V_+:=\frac{r_++\tau\log Z_\beta}{1-\gamma},\\
\Delta_r:=r_+-r_-,\qquad\Delta_V:=\esssup_sV^\star(s)-\essinf_sV^{\pi_0}(s).
\end{gathered}
\end{equation}
Here $V_+$ is the optimal soft value for the constant reward $r_+$. The bound $V_-$ accounts for the worst reward $r_-$ and for the cost $\tau\RelEnt_0$ per step of the divergence of $\pi_0$ from $\rho_\beta$. By Lemma~\ref{lem:bellman-identities}\itemref{itm:bellman-identities-iii} below, ${V_-\le V^{\pi_0}}$ and ${V^\star\le V_+}$. These two bounds give
\[
\Delta_V\le V_+-V_-=\frac{\Delta_r+\tau\RelEnt_0}{1-\gamma}.
\]

\subsection{Bellman Structure}\label{sec:bellman-identities}

The following lemma relates the Bellman residual to both a statewise KL divergence and the value gap, extending previous results~\cite{geist2019regmdp,kakade2002approximately,lan2023mirror} to general state spaces and unbounded continuous actions.

\begin{lemma}[Bellman identities and value bounds]\label{lem:bellman-identities}
\begin{enumerate}
\renewcommand{\theenumi}{\roman{enumi}}
\itemlabel{itm:bellman-identities-i}{i}
For every $\pi\in\Pi$ and bounded $V$,
\begin{equation}\label{eq:bellman-gibbs}
T^\pi V=T^\star V-\tau\KL(\pi\,\|\,G[V]).
\end{equation}
Moreover, $G[V]\in\Pi$, so $T^\star V=\max_{\pi\in\Pi}T^\pi V$, attained only at $G[V]$. In particular, $\pi^\star\in\Pi$ and $V^{\pi^\star}=V^\star$.
\itemlabel{itm:bellman-identities-ii}{ii}
For every $\pi\in\Pi$,
\begin{align}
R^\pi&=\tau\KL(\pi\,\|\,p^\pi)\ge0,\label{eq:residual-kl}\\
R^\pi&\ge W^\pi-\gamma E^\pi\mathbf 1,\label{eq:residual-lower}\\
(I-\gamma P^\pi)W^\pi&=\tau\KL(\pi\,\|\,\pi^\star),\label{eq:optimal-gap}
\end{align}
and hence $V^\pi\le V^\star$.
\itemlabel{itm:bellman-identities-iii}{iii}
$V_-\le V^{\pi_0}\le V^\star\le V_+$ and $V^\star\ge V_+-\Delta_r/(1-\gamma)$.
\end{enumerate}
\end{lemma}

\begin{proof}
\emph{Part~\itemref{itm:bellman-identities-i}.} Fix $s$. By~\eqref{eq:gibbs}, $\log G[V](a\mid s)=[Q_V(s,a)-(T^\star V)(s)]/\tau$. Inserting this into the KL divergence and using $\int\pi_s\,\dd a=1$ gives
\[
\begin{aligned}
\tau\KL\bigl(\pi_s\,\|\,G[V](\cdot\mid s)\bigr)
&=\int\tau\log\pi_s\,\pi_s\,\dd a-\int\bigl(Q_V(s,\cdot)-(T^\star V)(s)\bigr)\pi_s\,\dd a\\
&=(T^\star V)(s)-\int\bigl(Q_V(s,\cdot)-\tau\log\pi_s\bigr)\pi_s\,\dd a=(T^\star V)(s)-(T^\pi V)(s),
\end{aligned}
\]
which is~\eqref{eq:bellman-gibbs}. By Lemma~\ref{lem:gibbs-bounds}\itemref{itm:gibbs-bounds-i} below, $\log(G[V]/\rho_\beta)$ is bounded, so $G[V]\in\Pi$. Since the KL divergence is nonnegative and vanishes only when $\pi_s=G[V](\cdot\mid s)$, \eqref{eq:bellman-gibbs} gives $T^\pi V\le T^\star V$ for every $\pi\in\Pi$, with equality only at $\pi=G[V]$. Therefore $T^\star V$ is the maximum of $T^\pi V$ over $\Pi$, attained only at $G[V]$. For $V=V^\star$ and $\pi=\pi^\star=G[V^\star]$, the KL divergence vanishes, so~\eqref{eq:bellman-gibbs} gives $T^{\pi^\star}V^\star=T^\star V^\star=V^\star$. Thus $V^\star$ is the fixed point of $T^{\pi^\star}$, that is, $V^{\pi^\star}=V^\star$.

\emph{Part~\itemref{itm:bellman-identities-ii}.} Taking $V=V^\pi$ in~\eqref{eq:bellman-gibbs} and using $T^\pi V^\pi=V^\pi$ gives~\eqref{eq:residual-kl}. Since $T^\star$ is a $\gamma$-contraction in $\|\cdot\|_\infty$ with fixed point $V^\star$ (Lemma~\ref{lem:soft-bellman}),
\[
R^\pi=T^\star V^\pi-V^\pi=W^\pi-(T^\star V^\star-T^\star V^\pi)\ge W^\pi-\gamma E^\pi\mathbf 1,
\]
which is~\eqref{eq:residual-lower}. Taking $V=V^\star$ in~\eqref{eq:bellman-gibbs} gives $T^\pi V^\star=V^\star-\tau\KL(\pi\,\|\,\pi^\star)$. Since $T^\pi$ is affine with linear part $\gamma P^\pi$ and $T^\pi V^\pi=V^\pi$, it follows that
\[
W^\pi-\tau\KL(\pi\,\|\,\pi^\star)=T^\pi V^\star-T^\pi V^\pi=\gamma P^\pi W^\pi,
\]
which is~\eqref{eq:optimal-gap}. Applying the positive resolvent~\eqref{eq:resolvent} to~\eqref{eq:optimal-gap} gives $W^\pi\ge0$, that is, $V^\pi\le V^\star$.

\emph{Part~\itemref{itm:bellman-identities-iii}.} By~\eqref{eq:onestep-reward}, every $\pi\in\Pi$ satisfies
\[
r_-+\tau\log Z_\beta-\tau\KL(\pi_s\,\|\,\rho_\beta)\le c^\pi(s)\le r_++\tau\log Z_\beta.
\]
The resolvent~\eqref{eq:resolvent} is positive and maps $\mathbf 1$ to $(1-\gamma)^{-1}\mathbf 1$, and $V^\pi=(I-\gamma P^\pi)^{-1}c^\pi$. For $\pi=\pi_0$, Assumption~\ref{ass:base} bounds the KL term by $\RelEnt_0$, so the lower bound gives $V^{\pi_0}\ge V_-$. Part~\itemref{itm:bellman-identities-ii} gives $V^{\pi_0}\le V^\star$. By part~\itemref{itm:bellman-identities-i}, $\pi^\star\in\Pi$ and $V^{\pi^\star}=V^\star$, so the upper bound gives $V^\star\le V_+$. For the policy $\pi(\cdot\mid s)=\rho_\beta$, the KL term vanishes, so part~\itemref{itm:bellman-identities-ii} and the lower bound give
\[
V^\star\ge V^{\rho_\beta}\ge\frac{r_-+\tau\log Z_\beta}{1-\gamma}=V_+-\frac{\Delta_r}{1-\gamma}.\qedhere
\]
\end{proof}

For finite MDPs, entropy-regularized policy optimization is convex in the state--action occupancy measure~\cite{neu2017unified}, but the objective need not be concave in the policy. In the policy variable, the Bellman comparison~\eqref{eq:residual-lower} replaces a convexity bound on the value gap. Identity~\eqref{eq:optimal-gap} converts bounds on the value gap into KL bounds, which give the policy bounds in Theorems~\ref{thm:continuous} and~\ref{thm:discrete}.

\subsection{Wasserstein Gradient Dominance}\label{sec:dominance}
We say that a positive probability density $p$ on $\R^d$ satisfies a log-Sobolev inequality with constant $\alpha>0$, written $\mathrm{LSI}(\alpha)$, if
\begin{equation}\label{eq:lsi}
\mathcal I(\nu\,\|\,p)\ge2\alpha\KL(\nu\,\|\,p),\quad\text{for every probability density }\nu.
\end{equation}
Since $\nabla_a\log(\nu/p)$ is the Wasserstein gradient of the functional $\nu\mapsto\KL(\nu\,\|\,p)$, \eqref{eq:lsi} is a gradient-dominance inequality for this functional. We recall two classical LSI facts under this convention:
\begin{enumerate}
\item[\factref{F1}]\raisebox{\ht\strutbox}{\hypertarget{fact:F1}{}}\emph{Gaussian measures}~\cite{gross1975lsi}. The measure $\mathcal N(0,\sigma^2I_d)$ satisfies $\mathrm{LSI}(1/\sigma^2)$; in particular, $\rho_\beta$ satisfies $\mathrm{LSI}(\beta/\tau)$.
\item[\factref{F2}]\raisebox{\ht\strutbox}{\hypertarget{fact:F2}{}}\emph{Bounded perturbations}~\cite{holley1987lsi}. If $p$ satisfies $\mathrm{LSI}(\alpha)$ and $q\propto e^\psi p$ with ${\osc\psi\le\Omega}$, then $q$ satisfies $\mathrm{LSI}(\alpha e^{-\Omega})$. Here and below, ${\osc\psi:=\esssup_a\psi(a)-\essinf_a\psi(a)}$ is the essential oscillation.
\end{enumerate}

\begin{lemma}[Gibbs policies under value bounds]\label{lem:gibbs-bounds}
Let $r$ be bounded, let $V$ be bounded and measurable with $v_-\le V\le v_+$, and set $\Delta_v:=v_+-v_-$. Then, for a.e.\ $s\in \cS$:
\begin{enumerate}
\renewcommand{\theenumi}{\roman{enumi}}
\itemlabel{itm:gibbs-bounds-i}{i} the logarithmic density ratio is bounded:
\begin{equation}\label{eq:gibbs-ratio}
\bigl|\log\bigl(G[V](a\mid s)/\rho_\beta(a)\bigr)\bigr|\le(\Delta_r+\gamma\Delta_v)/\tau\qquad\text{for a.e.\ }a\in\R^d;
\end{equation}
\itemlabel{itm:gibbs-bounds-ii}{ii} $G[V](\cdot\mid s)$ satisfies an LSI with constant
\begin{equation}\label{eq:alpha}
\frac\beta\tau\exp\Bigl(-\frac{\Delta_r+\gamma\Delta_v}{\tau}\Bigr).
\end{equation}
\end{enumerate}
\end{lemma}

\begin{proof}
With $\psi_s:=r(s,\cdot)+\gamma(PV)(s,\cdot)$, we have $Q_V(s,\cdot)=\psi_s-\frac\beta2|\cdot|^2$. Since $\rho_\beta\propto e^{-\beta|a|^2/(2\tau)}$, \eqref{eq:gibbs} then shows that $G[V](\cdot\mid s)$ is proportional to $e^{\psi_s/\tau}\rho_\beta$. Since $P(\cdot\mid s,a)$ is a probability measure, $\psi_s$ takes values in $[r_-+\gamma v_-,r_++\gamma v_+]$, so ${\osc_a\psi_s\le\Delta_r+\gamma\Delta_v}$. For part~\itemref{itm:gibbs-bounds-i}, since $\rho_\beta$ is a probability density, the logarithm of the normalizing constant $\int e^{\psi_s/\tau}\rho_\beta\,\dd a$ lies between the essential infimum and supremum of $\psi_s/\tau$. Hence $|\log(G[V](\cdot\mid s)/\rho_\beta)|\le\osc_a\psi_s/\tau$, which gives~\eqref{eq:gibbs-ratio}. For part~\itemref{itm:gibbs-bounds-ii}, by~\factref{F1}, $\rho_\beta$ satisfies $\mathrm{LSI}(\beta/\tau)$, so~\factref{F2} with $\psi=\psi_s/\tau$ and $\Omega=(\Delta_r+\gamma\Delta_v)/\tau$ gives~\eqref{eq:alpha}.
\end{proof}

We now combine the LSI with Lemma~\ref{lem:bellman-identities}\itemref{itm:bellman-identities-ii} to bound the value gap by relative Fisher information.

\begin{proposition}[Bellman--Wasserstein gradient dominance]\label{prop:dominance}
Let $\pi\in\Pi$.
\begin{enumerate}
\renewcommand{\theenumi}{\roman{enumi}}
\itemlabel{itm:dominance-i}{i}
If, for some $\alpha>0$, $p^\pi_s$ satisfies $\mathrm{LSI}(\alpha)$ for $\mathfrak m$-a.e.\ $s$, then, for $\mathfrak m$-a.e.\ $s$,
\begin{equation}\label{eq:dominance}
\tau^2\mathcal I(\pi_s\,\|\,p^\pi_s)\ \ge\ 2\tau\alpha\,R^\pi(s)\ \ge\ 2\tau\alpha\big(W^\pi(s)-\gamma E^\pi\big).
\end{equation}
\itemlabel{itm:dominance-ii}{ii}
If $\mathcal I(\pi_s\,\|\,p^\pi_s)=0$ for $\mathfrak m$-a.e.\ $s$, then $\pi=\pi^\star$.
\end{enumerate}
\end{proposition}

\begin{proof}
\emph{Part~\itemref{itm:dominance-i}.} The LSI~\eqref{eq:lsi} of $p^\pi_s$, applied with $\nu=\pi_s$, and the identity~\eqref{eq:residual-kl} give
\[
\tau^2\mathcal I(\pi_s\,\|\,p^\pi_s)\ge2\alpha\,\tau^2\KL(\pi_s\,\|\,p^\pi_s)=2\tau\alpha\,R^\pi(s).
\]
The second inequality is~\eqref{eq:residual-lower}.

\emph{Part~\itemref{itm:dominance-ii}.} Since $V^\pi$ is bounded, Lemma~\ref{lem:gibbs-bounds}\itemref{itm:gibbs-bounds-ii} gives a common LSI constant for $p^\pi_s$. Hence part~\itemref{itm:dominance-i} and~\eqref{eq:residual-kl} give $R^\pi=0$, that is, $V^\pi=T^\star V^\pi$. By Lemma~\ref{lem:soft-bellman}, $V^\pi=V^\star$. By~\eqref{eq:residual-kl}, ${\pi=p^\pi=G[V^\star]=\pi^\star}$.
\end{proof}

When finite, ${\tau^2\mathcal I(\pi_s\,\|\,p^\pi_s)}$ is the squared $L^2(\pi_s)$ norm of the WPG velocity $v^\pi(s,\cdot)$ by~\eqref{eq:velocity} and~\eqref{eq:gibbs-score}. For $\varepsilon>0$, at states with $W^\pi(s)\ge E^\pi-\varepsilon$, part~\itemref{itm:dominance-i} bounds this squared norm below by $2\tau\alpha((1-\gamma)E^\pi-\varepsilon)$. Part~\itemref{itm:dominance-i} is therefore a gradient-dominance inequality for the value gap in the sup norm: a small gradient at states near the worst value gap forces that gap to be small.

\section{Global Convergence of the WPG Flow}\label{sec:ct}

This section establishes well-posedness and global convergence of the continuous-time WPG flow \eqref{eq:WPGF}. We first state the well-posedness result, with its proof given in Appendix~\ref{app:wellposed}. Section~\ref{sec:ct-main} presents the exponential convergence theorem, and Section~\ref{sec:ct-convergence} proves it by relating value improvement along the flow to the Wasserstein gradient-dominance inequality established in Section~\ref{sec:dominance}. 

The WPG flow is understood in the following distributional sense. 

\begin{definition}[Weak solution]\label{def:weak-solution}
A weak solution of~\eqref{eq:WPGF} is a jointly measurable family of probability densities $(\pi_t)_{t\ge0}$ such that, outside one $\mathfrak m$-null set of states, $t\mapsto\pi_t(\cdot\mid s)$ is continuous for weak convergence of probability measures (narrowly continuous) on $[0,\infty)$ with initial law $\pi_0(\cdot\mid s)$, and
\begin{equation}\label{eq:weak-flow}
\int_0^\infty\!\int_{\R^d}
\bigl(\partial_t\psi+\nabla_a Q^{\pi_t}\cdot\nabla_a\psi+\tau\Delta_a\psi\bigr)\pi_t\,\dd a\,\dd t=0,
\quad \psi\in C_c^\infty((0,\infty)\times\R^d).
\end{equation}
It is also required to satisfy, for every finite $T$,
\begin{equation}\label{eq:weak-class}
\sup_{0\le t\le T}\esssup_s\KL(\pi_t(\cdot\mid s)\,\|\,\rho_\beta)<\infty.
\end{equation}
\end{definition}

To state Theorem~\ref{thm:wellposed}, we use the uniform distance ${\wdist(\pi,\widehat\pi):=\esssup_s W_2(\pi_s,\widehat\pi_s)}$. Lemma~\ref{lem:entropy-integrability}\itemref{itm:entropy-integrability-ii} ensures this distance is finite for every pair of admissible policies.

\begin{theorem}[Well-posedness of \eqref{eq:WPGF}]\label{thm:wellposed}
Under Assumptions~\ref{ass:base} and~\ref{ass:smooth}, equation~\eqref{eq:WPGF} has a unique weak solution $(\pi_t)_{t\ge0}$, which satisfies the following properties.
\begin{enumerate}
\renewcommand{\theenumi}{\roman{enumi}}
\itemlabel{itm:wellposed-i}{i} For every $T<\infty$, $\wdist(\pi_t,\pi_q)\le C_T|t-q|^{1/2}$ on $[0,T]$. As $t\downarrow0$,
\begin{equation}\label{eq:initial-traces}
\KL(\pi_t(\cdot\mid s)\,\|\,\rho_\beta)\longrightarrow\KL(\pi_0(\cdot\mid s)\,\|\,\rho_\beta),
\qquad V_t(s)\longrightarrow V_0(s).
\end{equation}
\itemlabel{itm:wellposed-ii}{ii} The map $t\mapsto V_t$ is locally Lipschitz in $L^\infty_{\cS}$ on $(0,\infty)$, and the relative Fisher information satisfies ${\sup_{\varepsilon\le t\le T}\esssup_s\mathcal I(\pi_t(\cdot\mid s)\,\|\,\rho_\beta)<\infty}$ for $0<\varepsilon<T$. The weak velocity ${v_t=\nabla_aQ^{\pi_t}-\tau\nabla_a\log\pi_t}$ has a jointly measurable version. Almost every statewise policy curve is $W_2$-absolutely continuous on $[0,T]$ with square-integrable metric speed and velocity $v_t$.
\itemlabel{itm:wellposed-iii}{iii} For every $T<\infty$ and every other solution $\widehat\pi$ such that both initial policies have a common upper bound on their relative entropy with respect to $\rho_\beta$,
\begin{align}
\sup_{0\le t\le T}\wdist(\pi_t,\widehat\pi_t)
&\le C_T\wdist(\pi_0,\widehat\pi_0),\label{eq:flow-stability}\\
\|V^{\pi_t}-V^{\widehat\pi_t}\|_\infty
&\le C_T(1+t^{-1/2})\wdist(\pi_0,\widehat\pi_0),\qquad 0<t\le T.\label{eq:value-stability-initial}
\end{align}
\end{enumerate}
The constants depend only on $T$, the model constants, and the common initial entropy bound.
\end{theorem}

A standard way to establish well-posedness is through suitable Lipschitz control of the evolving dynamics~\cite{huang2020ddsde,chizat2022mean,kerimkulov2025fisherrao}. 
In~\eqref{eq:WPGF}, the Bellman recursion makes the drift $\nabla_aQ^\pi$ depend on the policy's relative entropy through its value function $V^\pi$. This complicates Lipschitz control of the drift with respect to the policy in Wasserstein distance, since relative entropy is not continuous under Wasserstein convergence, even on uniformly entropy-bounded sets. We instead exploit the smoothing effect of diffusion to establish the necessary stability of Bellman evaluation under uniformly finite initial relative entropy. Using this estimate, we construct a fixed point on value-function paths, initially clipping the value input to control the drift. We then establish uniform bounds on the evolving policy values, which allow us to remove the clipping and prove global well-posedness. Appendix~\ref{app:wellposed} provides the complete proof. 
A similar bootstrap argument underlies our discrete-time stability and convergence analysis detailed in Section \ref{sec:dt}.

\subsection{Main Convergence Result for Continuous-Time WPG Flow}\label{sec:ct-main}

Theorem~\ref{thm:continuous} below shows that the value gap decays exponentially along the flow. Its rate comes from a log-Sobolev inequality that holds uniformly along the trajectory. By monotone value improvement (see~\eqref{eq:monotone} below), the values along the solution of Theorem~\ref{thm:wellposed} lie between $\essinf_sV^{\pi_0}(s)$ and $\esssup_sV^\star(s)$ at all times. By Lemma~\ref{lem:gibbs-bounds}\itemref{itm:gibbs-bounds-ii}, every Gibbs policy $p^{\pi_t}_s$ along this solution satisfies $\mathrm{LSI}(\alpha)$ with
\begin{equation}\label{eq:alphac}
\alpha:=\frac\beta\tau\exp\Bigl(-\frac{\Delta_r+\gamma\Delta_V}{\tau}\Bigr)\ge\frac\beta\tau\exp\Bigl(-\frac{\Delta_r}{\tau(1-\gamma)}-\frac{\gamma\RelEnt_0}{1-\gamma}\Bigr).
\end{equation}

\begin{theorem}[Exponential convergence of \eqref{eq:WPGF}]\label{thm:continuous}
Let Assumptions~\ref{ass:base} and~\ref{ass:smooth} hold, and let $(\pi_t)_{t\ge0}$ be the solution of~\eqref{eq:WPGF} from Theorem~\ref{thm:wellposed}. Then, for all $t\ge0$,
\begin{equation}\label{eq:ct-rate}
E(t)\le e^{-\lambda t}E(0),\qquad\lambda:=2\tau\alpha(1-\gamma).
\end{equation}
Moreover, the performance gap and the statewise KL divergence to $\pi^\star$ satisfy
\begin{equation}\label{eq:ct-consequences}
0\le J_{\rho_0}(\pi^\star)-J_{\rho_0}(\pi_t)\le E(t),\qquad\tau\esssup_s\KL\bigl(\pi_t(\cdot\mid s)\,\|\,\pi^\star(\cdot\mid s)\bigr)\le E(t).
\end{equation}
\end{theorem}

The rate $\lambda$ reflects both the Bellman comparison and a uniform log-Sobolev bound. The factor $1-\gamma$ arises from applying~\eqref{eq:residual-lower} at a state with nearly maximal value gap. The local analysis in Section~\ref{sec:local-rate} removes this factor near the optimum. The constant $\alpha$ in~\eqref{eq:alphac} follows from the Holley--Stroock perturbation bound~\factref{F2}, which depends only on the oscillation of the perturbation. Since $V_t$ remains between $V^{\pi_0}$ and $V^\star$, this bound is uniform along the flow and depends on the initial policy only through $\essinf_s V^{\pi_0}(s)$, which enters $\Delta_V$. Replacing $\Delta_V$ by its upper bound $V_+-V_-$ gives the explicit lower bound in~\eqref{eq:alphac} in terms of $\RelEnt_0$. Such worst-case LSI bounds can yield rates that deteriorate exponentially, as in other LSI-based analyses~\cite{nitanda2022mfld,chizat2022mean}. The convergence proof also applies with $\alpha$ replaced by any common LSI constant of the Gibbs policies $p^{\pi_t}_s$. The worst-case constant~\eqref{eq:alphac} shows that one always exists.

A useful comparison is the Fisher--Rao policy-gradient flow, which is the continuous-time counterpart of NPG in policy space and the continuous-time limit of KL-based policy mirror descent. For entropy-regularized discounted MDPs, \cite[Theorem~2.8]{kerimkulov2025fisherrao} establishes global exponential convergence at rate $\tau$ (see also \cite[Corollary~31]{muller2024geometry} for finite MDPs), under the assumption that the initial density ratio is uniformly bounded above and away from zero, which is stronger than Assumption~\ref{ass:base}. Their analysis determines the policy through a differential equation and a normalized exponential representation. Implementing this approach requires evaluating an action-space normalization integral or approximating the resulting distribution; either can be challenging.

\subsection{Proof of Theorem \ref{thm:continuous}}\label{sec:ct-convergence}\label{sec:ct-theorem}

\paragraph{Proof strategy.} Unlike the standard convex analysis for mean-field Langevin dynamics \citep{nitanda2022mfld,chizat2022mean}, our proof does not require the functional objective to be concave. Instead, it exploits the Bellman structure of the soft values to derive the following statewise chain:
\begin{equation}\label{eq:mechanism}
\begin{aligned}
\dot V_t&\overset{\textup{(a)}}{=}(I-\gamma P^{\pi_t})^{-1}\big(\tau^2\mathcal I(\pi_t\,\|\,p^{\pi_t})\big)\overset{\textup{(b)}}{\ge}\tau^2\mathcal I(\pi_t\,\|\,p^{\pi_t})\\
&\overset{\textup{(c)}}{\ge}2\tau^2\alpha\KL(\pi_t\,\|\,p^{\pi_t})\overset{\textup{(d)}}{=}2\tau\alpha R_t\overset{\textup{(e)}}{\ge}2\tau\alpha\big(W_t-\gamma E(t)\mathbf 1\big).
\end{aligned}
\end{equation}
Step~(a) is the dissipation identity~\eqref{eq:energy} of Lemma~\ref{lem:ct-dissipation}. Step~(b) uses the positivity of the Bellman resolvent from Lemma~\ref{lem:soft-bellman}. Step~(c) is the LSI of Lemma~\ref{lem:gibbs-bounds}\itemref{itm:gibbs-bounds-ii} with the constant $\alpha$ of~\eqref{eq:alphac}. Steps~(d) and~(e) are the residual identity~\eqref{eq:residual-kl} and the Bellman comparison~\eqref{eq:residual-lower} of Lemma~\ref{lem:bellman-identities}\itemref{itm:bellman-identities-ii}. Proposition~\ref{prop:dominance}\itemref{itm:dominance-i} combines steps~(c)--(e) into a Bellman--Wasserstein gradient-dominance inequality. Integrating the chain at each state before taking the essential supremum yields Theorem~\ref{thm:continuous}.

\begin{lemma}[Statewise dissipation and monotone value improvement]\label{lem:ct-dissipation}
Under Assumptions~\ref{ass:base} and~\ref{ass:smooth}, let $(\pi_t)_{t\ge0}$ be the solution of~\eqref{eq:WPGF} from Theorem~\ref{thm:wellposed}. The value $V_t(s)$ is locally absolutely continuous in $t>0$, and its statewise derivative satisfies, for almost every $t>0$,
\begin{equation}\label{eq:energy}
\begin{aligned}
(I-\gamma P^{\pi_t})\dot V_t=g_t,\ \ \textrm{ where }
g_t(s):=\tau^2\mathcal I(\pi_t(\cdot\mid s)\,\|\,p^{\pi_t}_s)
=\int|v_t(s,a)|^2\pi_t(a\mid s)\,\dd a.
\end{aligned}
\end{equation}
Consequently, for all $t\ge0$,
\begin{equation}\label{eq:monotone}
\dot V_t=(I-\gamma P^{\pi_t})^{-1}g_t\ge g_t\ge0,
\qquad V_0\le V_t\le V^\star.
\end{equation}
\end{lemma}

\begin{proof}
We verify the conditions of Lemma~\ref{lem:value-derivative} in Appendix~\ref{app:chain}, which gives local absolute continuity of $V_t(s)$ and justifies the differentiation and integration by parts below. It suffices to establish uniform bounds on the value, velocity energy, entropy, and Fisher information over compact time intervals.
To verify these conditions, Theorem~\ref{thm:wellposed}\itemref{itm:wellposed-ii} and the entropy bound~\eqref{eq:weak-class} provide uniform entropy and Fisher-information bounds. By~\eqref{eq:onestep-reward} and the Bellman resolvent, the entropy bound also gives uniform bounds on $c^{\pi_t}$ and $V_t$. Next, the definitions~\eqref{eq:soft-operators} and~\eqref{eq:gibbs-policy}, together with~\eqref{eq:gaussian-reference}, give
\[
v_t=\nabla_ar+\gamma\nabla_a(PV_t)-\tau\nabla_a\log(\pi_t/\rho_\beta).
\]
Assumption~\ref{ass:smooth} and the bounds on $V_t$ therefore yield the required uniform velocity-energy bound. Combining these estimates with the continuity equation from Theorem~\ref{thm:wellposed}\itemref{itm:wellposed-ii} verifies the conditions of Lemma~\ref{lem:value-derivative}. Consequently, $V_t(s)$ is locally Lipschitz, hence locally absolutely continuous, in $t>0$.

For almost every $t>0$, differentiating the Bellman equation $V_t=c^{\pi_t}+\gamma P^{\pi_t}V_t$ gives
\[
(I-\gamma P^{\pi_t})\dot V_t=\partial_tc^{\pi_t}+\gamma(\partial_tP^{\pi_t})V_t.
\]
Integrating by parts along $\partial_t\pi_t=-\diver_a(\pi_tv_t)$, with $c^{\pi_t}$ written as in~\eqref{eq:onestep-reward}, we obtain
\begin{equation}\label{eq:pg-statewise}
\begin{aligned}
\partial_tc^{\pi_t}+\gamma(\partial_tP^{\pi_t})V_t
&=\int\bigl(\nabla_ar+\gamma\nabla_a(PV_t)-\tau\nabla_a\log(\pi_t/\rho_\beta)\bigr)\cdot v_t\,\pi_t\,\dd a\\
&=\int|v_t|^2\pi_t\,\dd a=\tau^2\mathcal I(\pi_t\,\|\,p^{\pi_t}).
\end{aligned}
\end{equation}
The second equality follows from the formula for $v_t$, and the third from the score identity~\eqref{eq:gibbs-score}, which gives $v_t=-\tau\nabla_a\log(\pi_t/p^{\pi_t})$. This proves~\eqref{eq:energy}. Applying the Bellman resolvent gives the identity in~\eqref{eq:monotone}.

By Lemma~\ref{lem:soft-bellman}, the Bellman resolvent is positive and dominates the identity on nonnegative functions, so $\dot V_t\ge g_t\ge0$. Integrating statewise on positive-time intervals and then using the initial-time limit~\eqref{eq:initial-traces} gives $V_t\ge V_0$ for all $t\ge0$. Finally, Lemma~\ref{lem:bellman-identities}\itemref{itm:bellman-identities-ii} gives $V_t\le V^\star$.
\end{proof}

Lemma~\ref{lem:ct-dissipation} also identifies the stationary solutions of~\eqref{eq:WPGF}. Under Assumptions~\ref{ass:base} and~\ref{ass:smooth}, the constant path ${\pi_t\equiv\pi^\star}$ is the only time-independent weak solution of~\eqref{eq:WPGF}. By Lemma~\ref{lem:bellman-identities}\itemref{itm:bellman-identities-i}, $\pi^\star\in\Pi$ and ${V^{\pi^\star}=V^\star}$, so $p^{\pi^\star}=G[V^\star]=\pi^\star$. By~\eqref{eq:gibbs-score}, the WPG velocity of $\pi^\star$ vanishes. Hence the constant path is a weak solution. Conversely, a time-independent weak solution has $\pi_0\in\Pi$ by~\eqref{eq:weak-class}. By the uniqueness in Theorem~\ref{thm:wellposed}, it is the solution from $\pi_0$, so Lemma~\ref{lem:ct-dissipation} applies. Since $\dot V_t=0$, \eqref{eq:energy} gives $\mathcal I(\pi_t\,\|\,p^{\pi_t})=0$, and Proposition~\ref{prop:dominance}\itemref{itm:dominance-ii} gives $\pi_t=\pi^\star$.

\begin{proof}[Proof of Theorem~\ref{thm:continuous}]
Recall that $W_t=V^\star-V_t\ge0$ and that every Gibbs policy $p^{\pi_t}_s$ satisfies $\mathrm{LSI}(\alpha)$ with $\alpha$ defined in~\eqref{eq:alphac}.

By~\eqref{eq:monotone} and Proposition~\ref{prop:dominance}\itemref{itm:dominance-i} with this constant, for almost every $t>0$ and a.e.\ $s$,
\begin{equation}\label{eq:ct-ode}
\dot W_t=-\dot V_t\le-g_t\le-2\tau\alpha\bigl(W_t-\gamma E(t)\mathbf 1\bigr).
\end{equation}
We integrate~\eqref{eq:ct-ode} at each state before taking the essential supremum. This order matters: the essential supremum defining $E(t)$ need not be attained, and the state at which it is nearly attained may change with time. By Lemma~\ref{lem:ct-dissipation}, $W_t(s)$ is locally absolutely continuous in $t>0$. By Theorem~\ref{thm:wellposed}\itemref{itm:wellposed-ii}, $E$ is continuous on compact time intervals bounded away from zero. Fix $0<t_0<t$. Outside one state null set, multiplying~\eqref{eq:ct-ode} by $e^{2\tau\alpha q}$ and integrating over $q\in[t_0,t]$ gives
\[
e^{2\tau\alpha t}W_t\le e^{2\tau\alpha t_0}W_{t_0}+2\tau\alpha\gamma\Bigl(\int_{t_0}^te^{2\tau\alpha q}E(q)\,\dd q\Bigr)\mathbf 1.
\]
Since $W_t\ge0$, the essential supremum of $W_t$ is $E(t)$, and $W_{t_0}\le E(t_0)\mathbf 1$. Taking the essential supremum of both sides therefore gives
\[
u(t)\le u(t_0)+2\tau\alpha\gamma\int_{t_0}^tu(q)\,\dd q,\qquad u(q):=e^{2\tau\alpha q}E(q).
\]
Since $t>t_0$ is arbitrary and $u$ is continuous, Gr\"onwall's inequality gives $u(t)\le e^{2\tau\alpha\gamma(t-t_0)}u(t_0)$, that is,
\[
E(t)\le e^{-2\tau\alpha(1-\gamma)(t-t_0)}E(t_0).
\]
By monotone value improvement, $E(t_0)\le E(0)$. Letting $t_0\downarrow0$ gives~\eqref{eq:ct-rate}.

Finally, since $J_{\rho_0}(\pi)=\int V^\pi\,\dd\rho_0$, integrating $0\le W_t\le E(t)\mathbf 1$ against $\rho_0$ gives the performance bound in~\eqref{eq:ct-consequences}. For the policy bound,~\eqref{eq:optimal-gap} gives
\[
0\le\tau\KL(\pi_t\,\|\,\pi^\star)=W_t-\gamma P^{\pi_t}W_t\le W_t\le E(t)\mathbf 1.\qedhere
\]
\end{proof}

\begin{remark}[Comparison with \cite{siska2026note}]\label{rem:concurrent-work}
For a finite state space, \cite{siska2026note} also analyzes the convergence of the flow~\eqref{eq:WPGF}. They assume that it has a sufficiently regular solution (their Assumption~3.2), whereas Theorem~\ref{thm:wellposed} proves that~\eqref{eq:WPGF} has a unique weak solution. In their cost-minimization notation, the argument bounds the value gap ${V^{\pi_t}(\rho)-V^{\pi^\star}(\rho)}$, averaged over an initial distribution $\rho$, by an occupancy-weighted KL divergence (their Proposition~3.1). Retaining the factor $1-\gamma$ from their inequality~(26) in their passage from~(39) to~(40) gives a bound of the form
\[
V^{\pi_t}(\rho)-V^{\pi^\star}(\rho)
\le\bar\kappa\Bigl(\int_{\cS}\bigl(V^{\pi_0}(s)-V^{\pi^\star}(s)\bigr)\,d_\rho^{\pi^\star}(\dd s)\Bigr)e^{-2\kappa\alpha\tau(1-\gamma)t},
\]
where $d_\rho^{\pi^\star}$ is the normalized discounted occupancy measure under $\pi^\star$, $\alpha$ is their LSI constant, and
\[
\bar\kappa:=\sup_{s\in\cS}\frac{\dd\rho}{\dd d_\rho^{\pi^\star}}(s),\qquad
\kappa:=\inf_{s\in\cS}\frac{\dd\rho}{\dd d_\rho^{\pi^\star}}(s)\le1.
\]
Thus their estimate is distribution-dependent and requires coverage-type den\-sity-ratio control, in particular $\kappa>0$, to obtain a positive exponential rate. By contrast, Theorem~\ref{thm:continuous} proves a state-uniform bound and therefore controls $J_{\rho_0}(\pi^\star)-J_{\rho_0}(\pi_t)$ for any initial state distribution $\rho_0\ll\mathfrak m$. The proof mechanisms are also different: occupancy-weighted analyses of policy gradient methods compare averages under different state distributions and involve distribution-mismatch or density-ratio coefficients~\cite{agarwal2021policy,xiao2022rates,ding2022global}, whereas we use the residual identity~\eqref{eq:residual-kl} and the Bellman comparison~\eqref{eq:residual-lower} to compare the residual with the value gap at each state. This removes the density-ratio constant $\kappa$ from the rate, although with a different initial gap.
\end{remark}

\section{Global Convergence of Discrete-Time WPG}\label{sec:dt}

This section establishes global convergence of \eqref{eq:WPG}. Section~\ref{sec:dt-results} presents the main convergence results. Unlike the continuous-time flow, a discrete update need not improve the value at every step, so the analysis must control the regularity and discretization error. \mbox{Sections~\ref{sec:dt-bellman}--\ref{sec:dt-proofs}} develop the proofs.

\subsection{Main Convergence Results for Discrete-Time WPG}\label{sec:dt-assumptions}\label{sec:dt-results}

To allow for decreases caused by discretization, we enlarge the interval $[V_-,V_+]$ of~\eqref{eq:value-range} by lowering~$V_-$~by $\tau$ and set
\begin{equation}\label{eq:dt-constants}
\begin{aligned}
\ell&:=V_--\tau,&\qquad\bar\Delta_V&:=V_+-\ell=V_+-V_-+\tau,\\
L&:=\beta+L_r+\frac\gamma2L_p\bar\Delta_V,&\bar\alpha&:=\frac\beta\tau\exp\Bigl(-\frac{\Delta_r+\gamma\bar\Delta_V}{\tau}\Bigr).
\end{aligned}
\end{equation}
By Lemmas~\ref{lem:gibbs-bounds} and~\ref{lem:drift}, $\bar\alpha$ is a uniform LSI constant for every Gibbs policy $G[V]$ with $\ell\le V\le V_+$, and $L$ is a Lipschitz constant for its associated drift. The rate and the coefficient of the error term in Theorem~\ref{thm:discrete} below are determined by
\begin{equation}\label{eq:dt-constants2}
\bar\lambda:=\frac12\bar\alpha\tau(1-\gamma),\qquad C:=\frac{6dL^2}{\bar\alpha(1-\gamma)^2}.
\end{equation}

\begin{theorem}[Geometric convergence of \eqref{eq:WPG} to a neighborhood of the optimum]\label{thm:discrete}
Let Assumptions~\ref{ass:base} and~\ref{ass:smooth} hold. For every constant step size $0<\eta\le\tau/C$, the iterates of \eqref{eq:WPG} with $\eta_k\equiv\eta$ satisfy, for all $k\ge0$,
\begin{equation}\label{eq:dt-rate}
E_k\le e^{-\bar\lambda\eta k}E_0+C\eta.
\end{equation}
Moreover, every iterate is admissible with $\ell\le V^{\pi_k}\le V_+$, and
\begin{equation}\label{eq:dt-consequences}
0\le J_{\rho_0}(\pi^\star)-J_{\rho_0}(\pi_k)\le E_k,\qquad\tau\esssup_s\KL\bigl(\pi_k(\cdot\mid s)\,\|\,\pi^\star(\cdot\mid s)\bigr)\le E_k.
\end{equation}
\end{theorem}

Step sizes that decrease to zero can make the error term $C\eta$ in~\eqref{eq:dt-rate} vanish, so that the iterates converge to the optimum. The following corollary gives such a schedule.


\begin{corollary}[Convergence with diminishing step sizes]\label{cor:diminishing}
Let Assumptions~\ref{ass:base} and~\ref{ass:smooth} hold. Choose an integer $k_0\ge2C/(\bar\lambda\tau)$ and apply \eqref{eq:WPG} with step sizes
\[
\eta_k:=\frac{2}{\bar\lambda(k+k_0)},\qquad k\ge0.
\]
Then for all $k\ge0$,
\begin{equation}\label{eq:dt-schedule}
E_k\le\frac{C_2}{k+k_0},\qquad C_2:=\max\Bigl\{k_0E_0,\frac{4C}{\bar\lambda}\Bigr\}.
\end{equation}
\end{corollary}

The same bound holds for the performance gap and for $\tau$ times the statewise KL divergence to $\pi^\star$. In particular, $E_k\le\varepsilon$ whenever $k\ge C_2/\varepsilon$, giving an $O(\varepsilon^{-1})$ iteration complexity.
The constant in~\eqref{eq:dt-schedule} depends on the worst-case LSI constant $\bar\alpha$. 

For comparison, NPG and policy mirror descent admit linear convergence rates under entropy regularization \cite{cen2022natural,lan2023mirror}. Their exact updates involve exponential reweighting using the soft Q-function, which is straightforward for tabular settings but generally requires normalization or sampling from the resulting distribution for continuous actions. 



\begin{remark}[Policy convergence from value convergence]\label{rem:policy-bounds}
In addition to the KL bounds in Theorems~\ref{thm:continuous} and~\ref{thm:discrete}, the value gap controls the total variation and Wasserstein distances to the optimal policy. For every $\pi\in\Pi$ and a.e.\ $s$,
\[
\tfrac12\|\pi_s-\pi^\star_s\|_{L^1}\le\Bigl(\frac{E^\pi}{2\tau}\Bigr)^{1/2},\qquad W_2^2(\pi_s,\pi^\star_s)\le\frac{2E^\pi}{\tau\alpha}.
\]
Indeed, \eqref{eq:optimal-gap} gives $\tau\KL(\pi_s\,\|\,\pi^\star_s)=W^\pi(s)-\gamma(P^\pi W^\pi)(s)$, which is at most $E^\pi$ because $W^\pi\ge0$ by Lemma~\ref{lem:bellman-identities}\itemref{itm:bellman-identities-ii}. Pinsker's inequality gives the first bound. Since $V^\star\ge V^{\pi_0}$ by Lemma~\ref{lem:bellman-identities}\itemref{itm:bellman-identities-ii}, $V^\star$ takes values between $\essinf_sV^{\pi_0}(s)$ and $\esssup_sV^\star(s)$, so Lemma~\ref{lem:gibbs-bounds}\itemref{itm:gibbs-bounds-ii} shows that $\pi^\star_s$ satisfies $\mathrm{LSI}(\alpha)$. Hence $\pi^\star_s$ satisfies Talagrand's inequality $W_2^2(\nu,\pi^\star_s)\le2\KL(\nu\,\|\,\pi^\star_s)/\alpha$~\cite[Theorem~1]{otto2000generalization}, which gives the second bound. Any LSI constant $\alpha_\star$ of $\pi^\star_s$ may replace $\alpha$. With $\pi=\pi_t$ or $\pi=\pi_k$, these bounds turn the value-gap guarantees of Theorems~\ref{thm:continuous} and~\ref{thm:discrete} and Corollary~\ref{cor:diminishing} into guarantees for the policies, and they apply in the same way to the local results of Section~\ref{sec:local}.
\end{remark}

The proof in the remainder of this section is organized as follows. Section~\ref{sec:dt-bellman} relates one-step policy changes to value improvement. Section~\ref{sec:onestep} develops KL estimates for the discretization error, and Section~\ref{sec:stability} establishes uniform bounds on policy values and derives a recursion for the value gap. Finally, Section~\ref{sec:dt-proofs} uses this recursion to prove convergence for constant and diminishing step sizes.

\subsection{The Bellman Comparison for Discrete-Time WPG}\label{sec:dt-bellman}

We compare consecutive policies through the following lemma.

\begin{lemma}[Soft performance difference]\label{lem:pdl}
For $\pi,\pi^+\in\Pi$,
\begin{equation}\label{eq:finite-difference}
(I-\gamma P^{\pi^+})(V^{\pi^+}-V^\pi)=T^{\pi^+}V^\pi-V^\pi=\tau\bigl(\KL(\pi\,\|\,p^\pi)-\KL(\pi^+\,\|\,p^\pi)\bigr).
\end{equation}
\end{lemma}

For MDPs with finite state and action spaces, the performance difference lemma was established without regularization in~\cite[Lemma~6.1]{kakade2002approximately} and with entropy regularization in~\cite[Lemma~2]{lan2023mirror}. Lemma~\ref{lem:pdl} extends it to general state spaces, unbounded continuous actions and all policies with bounded relative entropy. We prove the statewise version in our setting for completeness.

\begin{proof}
Both policies are admissible, so $V^\pi$ and $V^{\pi^+}$ are bounded. Since $T^{\pi^+}V=c^{\pi^+}+\gamma P^{\pi^+}V$ for bounded $V$ and $V^{\pi^+}=T^{\pi^+}V^{\pi^+}$,
\[
V^{\pi^+}-T^{\pi^+}V^\pi=T^{\pi^+}V^{\pi^+}-T^{\pi^+}V^\pi=\gamma P^{\pi^+}(V^{\pi^+}-V^\pi).
\]
Expanding the left-hand side of~\eqref{eq:finite-difference} and inserting this identity gives
\[
\begin{aligned}
(I-\gamma P^{\pi^+})(V^{\pi^+}-V^\pi)&=(V^{\pi^+}-V^\pi)-\gamma P^{\pi^+}(V^{\pi^+}-V^\pi)\\
&=(V^{\pi^+}-V^\pi)-(V^{\pi^+}-T^{\pi^+}V^\pi)=T^{\pi^+}V^\pi-V^\pi,
\end{aligned}
\]
which is the first equality in~\eqref{eq:finite-difference}.

For the second equality, apply~\eqref{eq:bellman-gibbs} with $V=V^\pi$, for which $G[V]=p^\pi$, to the policies $\pi^+$ and $\pi$. Since $T^\pi V^\pi=V^\pi$, this gives
\[
T^{\pi^+}V^\pi=T^\star V^\pi-\tau\KL(\pi^+\,\|\,p^\pi),\qquad V^\pi=T^\star V^\pi-\tau\KL(\pi\,\|\,p^\pi).
\]
Since $T^{\pi^+}V^\pi$, $T^\star V^\pi$ and $V^\pi$ are bounded, both KL divergences are finite. Subtracting the second identity from the first gives
\[
T^{\pi^+}V^\pi-V^\pi=\tau\bigl(\KL(\pi\,\|\,p^\pi)-\KL(\pi^+\,\|\,p^\pi)\bigr).\qedhere
\]
\end{proof}

The one-step Bellman improvement $T^{\pi^+}V^\pi-V^\pi$ need not be nonnegative, and for $\pi^+=p^\pi$ it equals~$R^\pi$. To relate a single WPG update to the flow, fix a sufficiently regular policy $\pi\in\Pi$, and let $\pi^+=\pi^+(\eta)$ be its update~\eqref{eq:WPG} with step size $\eta$. As $\eta\downarrow0$, the update direction converges to that of~\eqref{eq:WPGF}, so the map $\eta\mapsto\KL(\pi^+\,\|\,p^\pi)$ has derivative $-\tau\mathcal I(\pi\,\|\,p^\pi)$ at $\eta=0$. By~\eqref{eq:finite-difference}, $(T^{\pi^+}V^\pi-V^\pi)/\eta$ therefore tends to $\tau^2\mathcal I(\pi\,\|\,p^\pi)$ as $\eta\downarrow0$. Applying the Bellman resolvent identifies the resulting instantaneous value-improvement rate with~\eqref{eq:monotone}. For a finite step, Lemma~\ref{lem:ula} bounds $\KL(\pi^+\,\|\,p^\pi)$ from above, which by~\eqref{eq:finite-difference} gives a lower bound on the one-step Bellman improvement.

The next lemma converts a statewise KL estimate into value bounds. The proof of part~\itemref{itm:improvement-i} is the one-step counterpart of the chain~\eqref{eq:mechanism} along the flow. Part~\itemref{itm:improvement-ii} gives a lower bound on the new value, which Proposition~\ref{prop:dt-stability} uses to keep the values above $\ell$.

\begin{lemma}[Value bounds from a statewise KL estimate]\label{lem:improvement}
Let $\pi,\pi^+\in\Pi$, $V:=V^\pi$, $V^+:=V^{\pi^+}$ and $R:=T^\star V-V$. Suppose that, for some $\theta\in[0,1]$ and $\delta\ge0$,
\[
\KL(\pi^+\,\|\,p^\pi)\le(1-\theta)\KL(\pi\,\|\,p^\pi)+\delta\qquad\text{for a.e.\ }s.
\]
\begin{enumerate}
\renewcommand{\theenumi}{\roman{enumi}}
\itemlabel{itm:improvement-i}{i} The value gap in the sup norm satisfies
\begin{equation}\label{eq:contraction}
\|V^\star-V^+\|_\infty\le\bigl[1-(1-\gamma)\theta\bigr]\|V^\star-V\|_\infty+\frac{\tau\delta}{1-\gamma}.
\end{equation}
\itemlabel{itm:improvement-ii}{ii} If $V\ge\ell$, then
\begin{equation}\label{eq:lower-boundary}
V^+\ge\ell+\theta\bigl(r_-+\tau\log Z_\beta-(1-\gamma)\ell\bigr)-\frac{\tau\delta}{1-\gamma}.
\end{equation}
\end{enumerate}
\end{lemma}

\begin{proof}
\emph{Part~\itemref{itm:improvement-i}.} By~\eqref{eq:finite-difference}, the assumption and~\eqref{eq:residual-kl},
\[
T^{\pi^+}V-V=\tau\bigl\{\KL(\pi\,\|\,p^\pi)-\KL(\pi^+\,\|\,p^\pi)\bigr\}\ge\tau\theta\KL(\pi\,\|\,p^\pi)-\tau\delta=\theta R-\tau\delta.
\]
By~\eqref{eq:finite-difference}, $V^+-V=(I-\gamma P^{\pi^+})^{-1}(T^{\pi^+}V-V)$. Applying the resolvent to this inequality, with the properties listed after~\eqref{eq:resolvent} and $R\ge0$, gives
\begin{equation}\label{eq:improvement}
V^+-V\ge\theta(I-\gamma P^{\pi^+})^{-1}R-\frac{\tau\delta}{1-\gamma}\mathbf 1\ge\theta R-\frac{\tau\delta}{1-\gamma}\mathbf 1.
\end{equation}
\looseness=-1 By~\eqref{eq:residual-lower}, $R\ge(V^\star-V)-\gamma\|V^\star-V\|_\infty\mathbf 1$. Subtracting~\eqref{eq:improvement} from $V^\star-V$ and inserting this bound gives
\[
V^\star-V^+\le(1-\theta)(V^\star-V)+\gamma\theta\|V^\star-V\|_\infty\mathbf 1+\frac{\tau\delta}{1-\gamma}\mathbf 1.
\]
Since $V^\star-V^+\ge0$ by Lemma~\ref{lem:bellman-identities}\itemref{itm:bellman-identities-ii} and $V^\star-V\le\|V^\star-V\|_\infty\mathbf 1$, taking the essential supremum gives~\eqref{eq:contraction}.

\emph{Part~\itemref{itm:improvement-ii}.} If $V\ge\ell$, then $Q_V(s,a)\ge r_--\frac\beta2|a|^2+\gamma\ell$, so
\[
\int e^{Q_V(s,a)/\tau}\,\dd a\ge e^{(r_-+\gamma\ell)/\tau}Z_\beta,\qquad\text{that is,}\qquad T^\star V\ge r_-+\tau\log Z_\beta+\gamma\ell
\]
by~\eqref{eq:gibbs}. Since $R=T^\star V-V$, \eqref{eq:improvement} reads $V^+\ge(1-\theta)V+\theta T^\star V-\tau\delta/(1-\gamma)$. As $V\ge\ell$ and $\theta\in[0,1]$,
\[
\begin{aligned}
V^+&\ge(1-\theta)\ell+\theta\bigl(r_-+\tau\log Z_\beta+\gamma\ell\bigr)-\frac{\tau\delta}{1-\gamma}
=\ell+\theta\bigl(r_-+\tau\log Z_\beta-(1-\gamma)\ell\bigr)-\frac{\tau\delta}{1-\gamma},
\end{aligned}
\]
which is~\eqref{eq:lower-boundary}.
\end{proof}

\subsection{Drift Regularity and One-Step KL Estimates}\label{sec:onestep}

We use the one-step interpolation argument of Vempala and Wibisono~\citep[Lemma~3]{vempala2019} to bound the KL divergence to the current Gibbs policy after one step. This argument hinges on regularity bounds for the drift induced by a bounded value function. We first establish these bounds in Lemma~\ref{lem:drift} and then control the discretization error in Lemma~\ref{lem:ula}.

\begin{lemma}[Drift regularity]\label{lem:drift}
Let Assumption~\ref{ass:smooth} hold, let $V$ be bounded and measurable with ${v_-\le V\le v_+}$, and set $\Delta_v:=v_+-v_-$. Then, for a.e.\ $s\in\cS$, the drift $b_V:=\nabla_aQ_V$ splits as ${b_V(s,a)=-\beta a+h_V(s,a)}$ with
\begin{equation}\label{eq:drift}
\begin{aligned}
\sup_{a\in\R^d}|h_V(s,a)|\le G_r+\frac\gamma2G_p\Delta_v,\quad
\Lip_a(h_V)\le L_r+\frac\gamma2L_p\Delta_v,\quad
\Lip_a(b_V)\le\beta+L_r+\frac\gamma2L_p\Delta_v.
\end{aligned}
\end{equation}
\end{lemma}

\begin{proof}
Fix $s$ and $a,z\in\R^d$. For $\mathfrak m$-a.e.\ $s'$, the map $t\mapsto p(s'\mid s,a+tz)$ is $C^1$, and the fundamental theorem of calculus gives
\[
\begin{aligned}
p(s'\mid s,a+z)-p(s'\mid s,a)-\nabla_ap(s'\mid s,a)\cdot z
=\int_0^1\bigl(\nabla_ap(s'\mid s,a+tz)-\nabla_ap(s'\mid s,a)\bigr)\cdot z\,\dd t.
\end{aligned}
\]
Integrating over $s'$ and using Assumption~\ref{ass:smooth}\itemref{itm:smooth-b},
\[
\big\|p(\cdot\mid s,a+z)-p(\cdot\mid s,a)-\nabla_ap(\cdot\mid s,a)\cdot z\big\|_{L^1(\mathfrak m)}\le\int_0^1L_pt|z|^2\,\dd t=\tfrac12L_p|z|^2.
\]
Since $V$ is bounded, this bound shows that $a\mapsto(PV)(s,a)$ is differentiable, with
\[
\nabla_a(PV)(s,a)=\int V(s')\nabla_ap(s'\mid s,a)\,\mathfrak m(\dd s').
\]
For $V\equiv1$, $(P\mathbf 1)(s,a)=1$ gives $\int\nabla_ap(s'\mid s,a)\,\mathfrak m(\dd s')=0$. Hence, subtracting the midpoint of $[v_-,v_+]$,
\begin{equation}\label{eq:centered-drift}
b_V(s,a)=-\beta a+\nabla_ar(s,a)+\gamma\int\bigl(V(s')-\tfrac12(v_-+v_+)\bigr)\nabla_ap(s'\mid s,a)\,\mathfrak m(\dd s'),
\end{equation}
so $h_V(s,a)$ is the sum of the last two terms. Since $|V-\tfrac12(v_-+v_+)|\le\Delta_v/2$, Assumption~\ref{ass:smooth} gives the bound ${|h_V(s,a)|\le G_r+\frac\gamma2G_p\Delta_v}$ and
\[
\begin{aligned}
|h_V(s,a)-h_V(s,\bar a)|&\le L_r|a-\bar a|+\frac\gamma2\Delta_v\int|\nabla_ap(s'\mid s,a)-\nabla_ap(s'\mid s,\bar a)|\,\mathfrak m(\dd s')\\
&\le\Bigl(L_r+\frac\gamma2L_p\Delta_v\Bigr)|a-\bar a|.
\end{aligned}
\]
Adding the linear term $-\beta a$ gives the bound on $\Lip_a(b_V)$ in~\eqref{eq:drift}.
\end{proof}

\begin{lemma}[One-step KL estimate]\label{lem:ula}
Let $\pi_k\in\Pi$ and $\alpha,L>0$. Suppose that, for a.e.\ $s\in\cS$, the Gibbs policy $p^{\pi_k}_s$ satisfies $\mathrm{LSI}(\alpha)$ and the drift $\nabla_aQ^{\pi_k}(s,\cdot)$ is $L$-Lipschitz. Then, for every $\zeta\in(0,\frac12]$ and $0<\eta_k\le\zeta/(2L)$, the policy $\pi_{k+1}$ given by~\eqref{eq:WPG} satisfies
\begin{equation}\label{eq:onestep-zeta}
\KL(\pi_{k+1}\,\|\,p^{\pi_k})\le e^{-2(1-\zeta)\alpha\tau\eta_k}\KL(\pi_k\,\|\,p^{\pi_k})+\frac{3}{2\zeta}\,dL^2\eta_k^2.
\end{equation}
In particular, for $\zeta=\frac12$ and $0<\eta_k\le1/(4L)$,
\begin{equation}\label{eq:onestep}
\KL(\pi_{k+1}\,\|\,p^{\pi_k})\le e^{-\alpha\tau\eta_k}\KL(\pi_k\,\|\,p^{\pi_k})+3dL^2\eta_k^2.
\end{equation}
\end{lemma}

\begin{proof}
Fix $s$, and write $\eta:=\eta_k$, $\mu:=\pi_k(\cdot\mid s)$, $\rho:=p^{\pi_k}_s$ and $b:=\nabla_aQ^{\pi_k}(s,\cdot)$. The proof interpolates one \eqref{eq:WPG} step by the laws $(\mu_t)_{0\le t\le\eta}$. Specifically, let $X_0\sim\mu$, let $B$ be a standard Brownian motion independent of $X_0$, and set, for $t\in[0,\eta]$,
\[
X_t:=X_0+tb(X_0)+\sqrt{2\tau}\,B_t,\qquad\mu_t:=\Law(X_t),\qquad\mathcal H(t):=\KL(\mu_t\,\|\,\rho).
\]
Since $\sqrt{2\tau}B_\eta\sim\mathcal N(0,2\tau\eta I_d)$ is independent of $X_0$, \eqref{eq:WPG} gives $\mu_\eta=\pi_{k+1}(\cdot\mid s)$. 

By~\eqref{eq:gibbs-score}, $b=\tau\nabla\log\rho$. Since $\pi_k\in\Pi$, $\mu$ has a finite second moment by Lemma~\ref{lem:entropy-integrability}\itemref{itm:entropy-integrability-ii}. By~\eqref{eq:residual-kl}, $\tau\KL(\pi_k\,\|\,p^{\pi_k})=R_k=T^\star V_k-V_k$, which is bounded by Lemma~\ref{lem:soft-bellman}. Thus $\rho$, $b$ and $\mu$ satisfy the standing hypotheses of Appendix~\ref{app:ula}. 
Write $\mathcal I_t:=\mathcal I(\mu_t\,\|\,\rho)$, $D_t:=\E|X_t-X_0|^2$ and ${M_b(t):=\E|b(X_t)|^2}$. By Lemmas~\ref{lem:ula-chain} and~\ref{lem:ula-endpoint} in Appendix~\ref{app:ula}, $\mathcal I_t<\infty$ for $t>0$, and $\mathcal H$ is locally absolutely continuous on $(0,\eta]$ with $\mathcal H(t)\to\mathcal H(0)$ as $t\downarrow0$. Lemma~\ref{lem:ula-chain}\itemref{itm:ula-chain-iii} gives, for a.e.\ $t$,
\[
\mathcal H'(t)=-\tau \mathcal I_t+\int\mu_t(y)\bigl(\E[b(X_0)\mid X_t=y]-b(y)\bigr)\cdot\nabla\log(\mu_t/\rho)(y)\,\dd y,
\]
and the tower property turns the second term into an expectation:
\[
\mathcal H'(t)=-\tau \mathcal I_t+\E\big\langle b(X_0)-b(X_t),\nabla\log(\mu_t/\rho)(X_t)\big\rangle.
\]
The first term is the dissipation of $\KL(\cdot\,\|\,\rho)$ along the Langevin diffusion with drift~$b$, and the second is the error from evaluating the drift at $X_0$ instead of $X_t$.
Young's inequality ${\langle x,y\rangle\le|x|^2/(2\zeta\tau)+\zeta\tau|y|^2/2}$ and the Lipschitz bound $|b(X_0)-b(X_t)|\le L|X_t-X_0|$ give, for a.e.\ $t$,
\begin{equation}\label{eq:kprime}
\mathcal H'(t)\le-\Bigl(1-\frac\zeta2\Bigr)\tau \mathcal I_t+\frac{L^2}{2\zeta\tau}D_t.
\end{equation}

Next, we bound the displacement $D_t$ and the drift moment. 
Since $X_0$ and $B_t$ are independent, $D_t=t^2\E|b(X_0)|^2+2\tau dt$. Since $|b(X_0)|^2\le2|b(X_t)|^2+2L^2|X_t-X_0|^2$, taking expectations gives $\E|b(X_0)|^2\le2M_b(t)+2L^2D_t$, so
\[
D_t\le2t^2M_b(t)+2L^2t^2D_t+2\tau dt.
\]
Since $t\le\eta\le\zeta/(2L)$, we have $tL\le\zeta/2\le1/4$, hence $2L^2t^2\le\frac18$, and absorbing this term gives
\[
D_t\le\frac87\bigl(2t^2M_b(t)+2\tau dt\bigr)\le4t^2M_b(t)+4\tau dt.
\]
The Lipschitz drift has a weak divergence with $|\diver b|\le dL$ almost everywhere. By Lemma~\ref{lem:ula-chain}\itemref{itm:ula-chain-ii} and~\itemref{itm:ula-chain-iv}, the integrals below are finite, and
\[
\int\mu_t\,\diver b\,\dd y=-\int b\cdot\nabla\mu_t\,\dd y.
\]
Substituting $\nabla\log\mu_t=b/\tau+\nabla\log(\mu_t/\rho)$ into this identity and applying Young's inequality gives
\[
\begin{aligned}
M_b(t)&=-\tau\int\mu_t\,\diver b\,\dd y-\tau\int\mu_t b\cdot\nabla\log(\mu_t/\rho)\,\dd y
\le\tau dL+\frac12M_b(t)+\frac{\tau^2}{2}\mathcal I_t.
\end{aligned}
\]
Hence $M_b(t)\le\tau^2\mathcal I_t+2\tau dL$, a temperature-scaled version of~\cite[Lemma~24]{balasubramanian2022}. This bound keeps $\alpha$ out of the step-size condition. Substituting both bounds gives
\[
\frac{L^2}{2\zeta\tau}D_t\le\frac{2L^2t^2}{\zeta}\,\tau \mathcal I_t+\frac{4dL^3t^2+2dL^2t}{\zeta}.
\]
Since $tL\le\zeta/2$, the first term is at most $\frac\zeta2\tau \mathcal I_t$. Since $tL\le\frac14$, the term $4dL^3t^2$ is at most $dL^2t$.

Inserting these bounds into~\eqref{eq:kprime} and using the LSI, $\mathcal I_t\ge2\alpha\mathcal H(t)$, we obtain
\begin{equation}\label{eq:kprime2}
\mathcal H'(t)\le-(1-\zeta)\tau \mathcal I_t+\frac3\zeta\,dL^2t\le-2(1-\zeta)\alpha\tau\mathcal H(t)+\frac3\zeta\,dL^2t.
\end{equation}
Multiplying~\eqref{eq:kprime2} by $e^{2(1-\zeta)\alpha\tau t}$ gives
\[
\frac{\dd}{\dd t}\big(e^{2(1-\zeta)\alpha\tau t}\mathcal H(t)\big)\le\frac3\zeta dL^2t\,e^{2(1-\zeta)\alpha\tau t},\quad\text{for a.e.\ }t\in(0,\eta].
\]
Since $\mathcal H$ is locally absolutely continuous, we may integrate this inequality over $[\varepsilon,\eta]$ for $\varepsilon>0$. Letting $\varepsilon\downarrow0$ and using $e^{-2(1-\zeta)\alpha\tau(\eta-t)}\le1$ gives
\begingroup\postdisplaypenalty=10000
\[
\begin{aligned}
\mathcal H(\eta)&\le e^{-2(1-\zeta)\alpha\tau\eta}\mathcal H(0)+\frac3\zeta\,dL^2\int_0^\eta e^{-2(1-\zeta)\alpha\tau(\eta-t)}t\,\dd t\\
&\le e^{-2(1-\zeta)\alpha\tau\eta}\mathcal H(0)+\frac{3}{2\zeta}\,dL^2\eta^2,
\end{aligned}
\]
which is~\eqref{eq:onestep-zeta}; for $\zeta=\frac12$ it is~\eqref{eq:onestep}.
\endgroup
\end{proof}

\subsection{Stability Control and Uniform LSI}\label{sec:stability}

Lemma~\ref{lem:ula} bounds the KL divergence after one WPG step in terms of an LSI constant and a Lipschitz bound for the drift. In the stability argument, the value bounds ensure that these constants can be chosen uniformly along the iterates. We denote the contraction and error coefficients of~\eqref{eq:onestep} with $\alpha=\bar\alpha$ by
\begin{equation}\label{eq:delta}
\theta_\eta:=1-e^{-\bar\alpha\tau\eta},\qquad\delta_\eta:=3dL^2\eta^2.
\end{equation}

\begin{proposition}[Global stability and the value-gap recursion]\label{prop:dt-stability}
Let Assumptions~\ref{ass:base} and~\ref{ass:smooth} hold, and let $(\eta_k)_{k\ge0}$ be a deterministic sequence with $0<\eta_k\le\tau/C$. Every iterate of~\eqref{eq:WPG} is admissible and satisfies
\[
\ell\le V^{\pi_k}\le V_+,\qquad k\ge0.
\]
With $\theta_\eta$ as in~\eqref{eq:delta}, for all $k\ge0$,
\begin{equation}\label{eq:gap-recursion}
E_{k+1}\le\bigl[1-(1-\gamma)\theta_{\eta_k}\bigr]E_k+\bar\lambda C\eta_k^2.
\end{equation}
\end{proposition}

\begin{figure}[t]
\centering
\begin{tikzpicture}[
  node distance=6mm and 9mm,
  box/.style={draw=wpgblue, thick, rounded corners=3pt, align=flush center, font=\footnotesize, inner sep=3pt, text width=46mm, minimum height=11mm, fill=wpgblue!4},
  arr/.style={-{Stealth[length=2.6mm]}, thick, wpgblue}
]
\node[box] (a) {$\pi_k\in\Pi$ and\\ $V^{\pi_k}\in[\ell,V_+]$};
\node[box, right=of a] (b) {uniform constants:\\ $\mathrm{LSI}(\bar\alpha)$, drift Lipschitz $L$\\ {\footnotesize Lemma~\ref{lem:gibbs-bounds}\itemref{itm:gibbs-bounds-ii} and Lemma~\ref{lem:drift}}};
\node[box, below=of b] (c) {one-step KL estimate~\eqref{eq:kl-step}\\ {\footnotesize Lemma~\ref{lem:ula}}};
\node[box, below=of c] (d) {$\pi_{k+1}\in\Pi$\\ {\footnotesize Lemmas~\ref{lem:gibbs-bounds}\itemref{itm:gibbs-bounds-i} and~\ref{lem:entropy-integrability}\itemref{itm:entropy-integrability-iii}}};
\node[box, left=of d] (e) {$V^{\pi_{k+1}}\in[\ell,V_+]$\\ {\footnotesize Lemma~\ref{lem:bellman-identities}\itemref{itm:bellman-identities-ii}\itemref{itm:bellman-identities-iii}; Lemma~\ref{lem:improvement}\itemref{itm:improvement-ii}; \eqref{eq:budget}}};
\node[box, below=of d, dashed] (f) {Value-gap recursion~\eqref{eq:gap-recursion}\\ {\footnotesize Lemma~\ref{lem:bellman-identities}\itemref{itm:bellman-identities-ii} and Lemma~\ref{lem:improvement}\itemref{itm:improvement-i}}};
\draw[arr] (a) -- (b);
\draw[arr] (b) -- (c);
\draw[arr] (c) -- (d);
\draw[arr] (d) -- (e);
\draw[arr] (e.north) -- node[midway, left, font=\footnotesize, text=black]{$k\leftarrow k+1$} (a.south);
\draw[arr, dashed] (d) -- (f);
\end{tikzpicture}
\caption{One iteration of the induction in the proof of Proposition~\ref{prop:dt-stability}. The loop keeps the values in $[\ell,V_+]$, so that the LSI constant from Lemma~\ref{lem:gibbs-bounds}\itemref{itm:gibbs-bounds-ii} and the drift bound from Lemma~\ref{lem:drift} stay uniform over all iterations along the trajectory. Each iteration also yields the value-gap recursion~\eqref{eq:gap-recursion}.}
\label{fig:loop}
\end{figure}

\begin{proof}
The proof closes an induction on the iterates, preserving the admissibility and value bounds at every step, as summarized in Figure~\ref{fig:loop}. Assuming that $\pi_k$ is admissible and $V_k\in[\ell,V_+]$, Lemma~\ref{lem:gibbs-bounds}\itemref{itm:gibbs-bounds-ii} and Lemma~\ref{lem:drift} provide uniform LSI and drift constants, so Lemma~\ref{lem:ula} gives the one-step KL estimate~\eqref{eq:kl-step}. The change-of-reference identity~\eqref{eq:change-reference} then proves admissibility of $\pi_{k+1}$. Lemma~\ref{lem:improvement}\itemref{itm:improvement-ii}, together with~\eqref{eq:budget}, preserves the lower value bound, while Lemma~\ref{lem:improvement}\itemref{itm:improvement-i} yields the recursion~\eqref{eq:gap-recursion}.

We first establish bounds on the step size and discretization error needed to preserve the value bounds.
Let $0<\eta\le\tau/C$. Then
\begin{equation}\label{eq:step-bounds}
\bar\alpha\tau\eta\le\eta L\le\frac16,\qquad\theta_\eta\ge\frac12\bar\alpha\tau\eta,\qquad\bar\lambda\eta\le\frac1{12}.
\end{equation}
Indeed, the exponential factor in the definition of $\bar\alpha$ is at most one, so $\bar\alpha\tau\le\beta$, and $\beta\le L$ by the definition of $L$. With the identity $\tau/C=\bar\alpha\tau(1-\gamma)^2/(6dL^2)$, this gives
\[
\bar\alpha\tau\eta\le\eta L\le\frac{\tau L}{C}=\frac{\bar\alpha\tau}{L}\cdot\frac{(1-\gamma)^2}{6d}\le\frac16.
\]
In particular, $\eta L\le\frac14$, as Lemma~\ref{lem:ula} requires. The bound on $\theta_\eta$ follows from $1-e^{-x}\ge x/2$ for $0\le x\le1$, applied with $x=\bar\alpha\tau\eta\le\frac16$. The bound on $\bar\lambda\eta$ follows from $\bar\lambda\eta=\frac12(1-\gamma)\bar\alpha\tau\eta$ and $\bar\alpha\tau\eta\le\frac16$. Multiplying the bound on $\theta_\eta$ by $1-\gamma$ gives $(1-\gamma)\theta_\eta\ge\bar\lambda\eta$. The definitions of $\delta_\eta$, $\bar\lambda$ and $C$ give $\tau\delta_\eta/(1-\gamma)=\bar\lambda C\eta^2$. Together with $(1-\gamma)\theta_\eta\ge\bar\lambda\eta$ and $C\eta\le\tau$, this identity yields the bound needed to preserve the lower value bound $\ell=V_--\tau$,
\begin{equation}\label{eq:budget}
\frac{\tau\delta_\eta}{1-\gamma}=\bar\lambda\eta\cdot C\eta\le(1-\gamma)\theta_\eta\cdot\tau.
\end{equation}
The step bound $\eta\le\tau/C$ ensures~\eqref{eq:budget}, which prevents discretization error from pushing the values below $\ell$.

We now carry out the induction. For $k=0$, $\pi_0\in\Pi$ by Assumption~\ref{ass:base}, and $V_0\in[V_-,V_+]\subset[\ell,V_+]$ by Lemma~\ref{lem:bellman-identities}\itemref{itm:bellman-identities-iii}.
Suppose that $\pi_k\in\Pi$ and $\ell\le V_k\le V_+$. By Lemmas~\ref{lem:gibbs-bounds}\itemref{itm:gibbs-bounds-ii} and~\ref{lem:drift} with $\Delta_v=\bar\Delta_V$, every $p^{\pi_k}_s$ satisfies $\mathrm{LSI}(\bar\alpha)$, and the drift $\nabla_aQ^{\pi_k}(s,\cdot)$ is $L$-Lipschitz. Hence~\eqref{eq:onestep} holds with $\bar\alpha$ in place of $\alpha$, that is,
\begin{equation}\label{eq:kl-step}
\KL(\pi_{k+1}\,\|\,p^{\pi_k})\le(1-\theta_{\eta_k})\KL(\pi_k\,\|\,p^{\pi_k})+\delta_{\eta_k}.
\end{equation}
Next, we check that $\pi_{k+1}$ is admissible. The discrete-time WPG step yields the Gaussian-mixture density
\[
\pi_{k+1}(a\mid s)=\int\frac{\exp\bigl(-|a-x-\eta_k\nabla_aQ^{\pi_k}(s,x)|^2/(4\tau\eta_k)\bigr)}{(4\uppi\tau\eta_k)^{d/2}}\,\pi_k(x\mid s)\,\dd x,
\]
with the drift $\nabla_aQ^{\pi_k}$ of Lemma~\ref{lem:drift}. This density is jointly measurable, strictly positive and integrates to one. By the change-of-reference identity~\eqref{eq:change-reference},
\[
\KL(\pi_{k+1}\,\|\,\rho_\beta)=\KL(\pi_{k+1}\,\|\,p^{\pi_k})+\int\log\frac{p^{\pi_k}}{\rho_\beta}\,\dd\pi_{k+1},
\]
and Lemma~\ref{lem:gibbs-bounds}\itemref{itm:gibbs-bounds-i} bounds the last integrand by $(\Delta_r+\gamma\bar\Delta_V)/\tau$. With~\eqref{eq:kl-step}, \eqref{eq:residual-kl} and the boundedness of $R_k$ (Lemma~\ref{lem:soft-bellman}), this gives
\[
\esssup_s\KL\bigl(\pi_{k+1}(\cdot\mid s)\,\|\,\rho_\beta\bigr)\le(1-\theta_{\eta_k})\frac{\|R_k\|_\infty}{\tau}+\delta_{\eta_k}+\frac{\Delta_r+\gamma\bar\Delta_V}{\tau}<\infty,
\]
so $\pi_{k+1}\in\Pi$. Its value $V_{k+1}$ is well defined, and Lemma~\ref{lem:bellman-identities}\itemref{itm:bellman-identities-ii} and~\itemref{itm:bellman-identities-iii} give $V_{k+1}\le V^\star\le V_+$.

The preceding estimates establish admissibility of the updated policy and the upper value bound. To complete the induction, we show that the lower bound $V_{k+1}\ge\ell$ is also preserved. Since both policies are admissible and $V_k\ge\ell$, Lemma~\ref{lem:improvement}\itemref{itm:improvement-ii} applies to~\eqref{eq:kl-step}. By the definition of $V_-$ in~\eqref{eq:value-range}, $r_-+\tau\log Z_\beta=(1-\gamma)V_-+\tau\RelEnt_0$, so the margin $V_--\ell=\tau$ enters the coefficient of $\theta$ in~\eqref{eq:lower-boundary}:
\[
r_-+\tau\log Z_\beta-(1-\gamma)\ell=\tau\RelEnt_0+(1-\gamma)(V_--\ell)=\tau\RelEnt_0+\tau(1-\gamma).
\]
Inserting this coefficient into~\eqref{eq:lower-boundary} gives
\begin{equation}\label{eq:lower-preserved}
V_{k+1}\ge\ell+\theta_{\eta_k}\tau\RelEnt_0+\Bigl(\tau(1-\gamma)\theta_{\eta_k}-\frac{\tau\delta_{\eta_k}}{1-\gamma}\Bigr)\ge\ell,
\end{equation}
because $\RelEnt_0\ge0$ and the term in parentheses is nonnegative by~\eqref{eq:budget}. This closes the induction: every iterate is admissible, with values in $[\ell,V_+]$, so the one-step estimates remain valid at each iteration.

With admissibility and the value bounds established, it remains to derive the value-gap recursion. By~\eqref{eq:kl-step}, Lemma~\ref{lem:improvement}\itemref{itm:improvement-i} applies with $\theta=\theta_{\eta_k}$ and $\delta=\delta_{\eta_k}$. Using $\|V^\star-V_k\|_\infty=E_k$ and $\tau\delta_{\eta_k}/(1-\gamma)=\bar\lambda C\eta_k^2$ from~\eqref{eq:budget}, its bound~\eqref{eq:contraction} gives~\eqref{eq:gap-recursion}.
\end{proof}

\subsection{Proofs of Convergence for Constant and Diminishing Step Sizes}\label{sec:dt-proofs}\label{sec:diminishing}

Both results follow from Proposition~\ref{prop:dt-stability}. Let $0<\eta_k\le\tau/C$ for all $k$. Equation~\eqref{eq:step-bounds} and its derivation give $(1-\gamma)\theta_{\eta_k}\ge\bar\lambda\eta_k$ and $\bar\lambda\eta_k\le\frac1{12}$. Since $E_k\ge0$, replacing the coefficient $1-(1-\gamma)\theta_{\eta_k}$ of $E_k$ in~\eqref{eq:gap-recursion} by the larger coefficient $1-\bar\lambda\eta_k$ gives
\begin{equation}\label{eq:diminishing}
E_{k+1}\le(1-\bar\lambda\eta_k)E_k+\bar\lambda C\eta_k^2.
\end{equation}
Since $\bar\lambda\eta_k\le\frac1{12}$, the factor $1-\bar\lambda\eta_k$ is positive. Hence~\eqref{eq:diminishing} remains valid when $E_k$ on its right-hand side is replaced by an upper bound on $E_k$, so it can be iterated.

\begin{proof}[Proof of Theorem~\ref{thm:discrete}]
Proposition~\ref{prop:dt-stability} shows that every iterate is admissible with $\ell\le V^{\pi_k}\le V_+$. For $\eta_k\equiv\eta$, iterating~\eqref{eq:diminishing} gives
\[
E_k\le(1-\bar\lambda\eta)^kE_0+\bar\lambda C\eta^2\sum_{j=0}^{k-1}(1-\bar\lambda\eta)^j\le(1-\bar\lambda\eta)^kE_0+C\eta\le e^{-\bar\lambda\eta k}E_0+C\eta,
\]
which is~\eqref{eq:dt-rate}. Here the middle inequality bounds the sum by the full geometric series, which equals $1/(\bar\lambda\eta)$, and the last one uses $1-\bar\lambda\eta\le e^{-\bar\lambda\eta}$. Since $\pi_k\in\Pi$, Lemma~\ref{lem:bellman-identities}\itemref{itm:bellman-identities-ii} gives $W_k\ge0$, and $W_k\le E_k\mathbf 1$ by the definition of $E_k$. As in the proof of Theorem~\ref{thm:continuous}, integrating these bounds against $\rho_0$ gives the performance bound in~\eqref{eq:dt-consequences}, and~\eqref{eq:optimal-gap} gives $\tau\KL(\pi_k\,\|\,\pi^\star)\le W_k$, hence the bound on the KL divergence.
\end{proof}

\begin{remark}[Comparison with the Wasserstein proximal update]\label{rem:wppg}
Zhu et al.~\citep[Theorem~1]{zhu2026proximal} analyze the exact Wasserstein proximal update. Their exact-$Q$ convergence theorem sets the step size to $\eta=1/(\gamma\lambda\tau)$, where $\lambda$ is the uniform transport-entropy ($T_2$) constant along the optimization trajectory. By contrast, we analyze the explicit update~\eqref{eq:WPG}, whose drift $\nabla_aQ^{\pi_k}$ is re-evaluated through the Bellman equation at every iterate. The proof is also entirely different. It does not use the optimality of a proximal step, but combines the Bellman comparison~\eqref{eq:residual-lower} with a log-Sobolev inequality and controls the discretization error uniformly along the iterates (Lemma~\ref{lem:ula} and Proposition~\ref{prop:dt-stability}).
\end{remark}

\begin{proof}[Proof of Corollary~\ref{cor:diminishing}]
Since $k_0\ge2C/(\bar\lambda\tau)$, the first step $\eta_0=2/(\bar\lambda k_0)$ is at most $\tau/C$. The steps are positive and decrease in $k$, so $0<\eta_k\le\tau/C$ for all $k$, and Proposition~\ref{prop:dt-stability} and~\eqref{eq:diminishing} apply. Write $n:=k+k_0$. Then $\bar\lambda\eta_k=2/n$, and the error term in~\eqref{eq:diminishing} is $\bar\lambda C\eta_k^2=4C/(\bar\lambda n^2)$. We prove $E_k\le C_2/n$, which is~\eqref{eq:dt-schedule}, by induction on $k$. For $k=0$, it reads $E_0\le C_2/k_0$, which holds because $C_2\ge k_0E_0$. Suppose that $E_k\le C_2/n$. Inserting this bound into~\eqref{eq:diminishing} gives
\[
E_{k+1}\le\Bigl(1-\frac2n\Bigr)\frac{C_2}{n}+\frac{4C}{\bar\lambda n^2}\le\frac{C_2}{n}-\frac{C_2}{n^2}\le\frac{C_2}{n+1}.
\]
In the middle inequality, the contraction lowers $C_2/n$ by $2C_2/n^2$, and the error $4C/(\bar\lambda n^2)$ is at most half of this decrease, because $C_2\ge4C/\bar\lambda$. The last inequality is equivalent to $(n-1)(n+1)\le n^2$. Since $n+1=(k+1)+k_0$, this completes the induction. As in the proof of Theorem~\ref{thm:discrete}, the performance gap and $\tau$ times the statewise KL divergence to $\pi^\star$ are at most $E_k$, so~\eqref{eq:dt-schedule} bounds them as well. Finally, for $k\ge C_2/\varepsilon$, \eqref{eq:dt-schedule} gives $E_k\le C_2/k\le\varepsilon$.
\end{proof}

\begin{remark}[Unbounded rewards]
The results also apply to certain unbounded rewards. Suppose that a measurable function $\Psi:\cS\to\R$ with $P|\Psi|<\infty$ makes Assumptions~\ref{ass:base} and~\ref{ass:smooth} hold with the shaped reward
\[
r_\Psi(s,a):=r(s,a)+\gamma(P\Psi)(s,a)-\Psi(s)
\]
in place of $r$, and that, for every $\pi\in\Pi$,
\[
\sum_{n\ge0}\gamma^n\E_{s,\pi}|\Psi(s_n)|<\infty,\qquad\sum_{n\ge0}\gamma^n\E_{\rho_0,\pi}|\Psi(s_n)|<\infty,
\]
where $\E_{s,\pi}:=\E_{\delta_s,\pi}$. Let $V^\pi_\Psi$ and $Q^\pi_\Psi$ denote the value and the soft Q-function of $\pi$ for $r_\Psi$. Since $r=r_\Psi+\Psi-\gamma P\Psi$ and $r_\Psi$ is bounded, these conditions make the discounted regularized reward of $\pi$ absolutely integrable, so its expectation $V^\pi$ is well defined. The terms $\Psi(s_n)-\gamma(P\Psi)(s_n,a_n)$ telescope in expectation, and the terminal term vanishes by the same conditions. Hence
\[
V^\pi_\Psi=V^\pi-\Psi,\qquad Q^\pi_\Psi(s,a)=Q^\pi(s,a)-\Psi(s).
\]
Consequently, the Gibbs policies, the optimal policy, the WPG dynamics from a given initial policy and the value gaps are the same for $r$ and $r_\Psi$. Theorems~\ref{thm:wellposed}--\ref{thm:dt-local} and their corollaries therefore apply with $r_\Psi$ in place of $r$, and their conclusions about the dynamics, the value gaps and the policies hold for $r$. Statements about the values themselves, such as bounds and $L^\infty$-regularity, refer to $V^\pi-\Psi$.
\end{remark}

\begin{remark}[General action penalties]\label{rem:confinement}
The quadratic penalty can be replaced by a state-independent action penalty $U\in C^1(\R^d)$ that is $m_U$-strongly convex, with $m_U>0$, and has an $L_U$-Lipschitz gradient. For $\tilde r(s,a)=r(s,a)-U(a)$, use $\rho_U(a):=Z_U^{-1}e^{-U(a)/\tau}$, with ${Z_U:=\int_{\R^d}e^{-U(a)/\tau}\,\dd a}$, in place of $\rho_\beta$ and $Z_\beta$, and define admissibility relative to $\rho_U$. By the Bakry--\'Emery criterion~\factref{F3} of Section~\ref{sec:local}, $\rho_U$ satisfies $\mathrm{LSI}(m_U/\tau)$, so $m_U$ replaces $\beta$ in the global LSI estimates. In the drift Lipschitz bounds, $L_U$ replaces $\beta$. The well-posedness and convergence arguments extend with these replacements (Appendix~\ref{app:confinement}).
\end{remark}

\section{Local Convergence under Optimal-Policy Geometry}\label{sec:local}

This section establishes faster local convergence rates using geometric properties of the optimal policy. We first show that, near the optimum, the current Gibbs policies inherit favorable LSI or curvature bounds. A refined Bellman comparison isolates the discrepancy between the current and optimal transition operators, and this discrepancy contributes only a higher-order error. These estimates allow the local convergence rates to avoid the factor $1-\gamma$ appearing in the global exponents. Section~\ref{sec:local-rate} develops the continuous-time results, including the case where $Q^\star$ is strongly concave. Section~\ref{sec:dt-local} extends the analysis to discrete-time WPG, accounting for discretization errors and ensuring that the iterates remain within the local neighborhood.

\subsection{Optimal-Policy Geometry}

The local rates use the geometry of the optimal policy rather than an LSI constant shared by all Gibbs policies along the trajectory. We consider two conditions: an LSI for $\pi^\star$, or strong concavity of $Q^\star$. Under Assumption~\ref{ass:smooth}, $Q^\pi(s,\cdot)$ is $C^1$ for every $\pi\in\Pi$ and a.e.\ $s$ by Lemma~\ref{lem:drift}. For $\mu>0$, a $C^1$ function $g$ on $\R^d$ is $\mu$-strongly concave if
\[
\langle\nabla g(a)-\nabla g(\bar a),a-\bar a\rangle\le-\mu|a-\bar a|^2
\qquad\text{for all }a,\bar a\in\R^d.
\]
In particular, this definition applies to $Q^\star(s,\cdot)$. Set $P^\star:=P^{\pi^\star}$.

Strong concavity provides the required LSI through the following classical criterion.
\begin{enumerate}
\item[\factref{F3}]\raisebox{\ht\strutbox}{\hypertarget{fact:F3}{}}\emph{Strongly log-concave measures}~\cite{bakry1985diffusions,bakry2014analysis}. If $p\propto e^{-\Phi}$ with $\Phi\in C^1(\R^d)$ and
\[
\langle\nabla\Phi(a)-\nabla\Phi(\bar a),a-\bar a\rangle\ge c_\Phi|a-\bar a|^2
\]
for some $c_\Phi>0$ and all $a,\bar a$, then $p$ satisfies $\mathrm{LSI}(c_\Phi)$.
\end{enumerate}
For $\Phi\in C^2$, \factref{F3} is the Bakry--\'Emery criterion. For $\Phi\in C^1$, the assumption means that $\Phi-\frac{c_\Phi}2|a|^2$ is convex, and \factref{F3} follows from~\cite[Proposition~3.1]{bobkov2000brunn}. Thus, if $Q^\star(s,\cdot)$ is $\mu$-strongly concave, the Gibbs policy $\pi^\star_s\propto e^{Q^\star(s,\cdot)/\tau}$ satisfies $\mathrm{LSI}(\mu/\tau)$.

We next show how these optimal-policy properties extend to the Gibbs policies of nearby policies. We also bound the cross term $(P^\pi-P^\star)W^\pi$ that enters the local value-gap comparison. The following lemma collects these estimates.

\begin{lemma}[Near-optimal comparisons]\label{lem:near-optimal}
Let Assumptions~\ref{ass:base} and~\ref{ass:smooth} hold and let $\pi\in\Pi$.
\begin{enumerate}
\renewcommand{\theenumi}{\roman{enumi}}
\itemlabel{itm:near-optimal-i}{i}
If $\pi^\star_s$ satisfies $\mathrm{LSI}(\alpha_\star)$ for a.e.\ $s$, then $p^\pi_s$ satisfies an LSI with constant $\alpha_\star e^{-\gamma E^\pi/\tau}$ for a.e.\ $s$.
\itemlabel{itm:near-optimal-ii}{ii}
If $Q^\star(s,\cdot)$ is $\mu$-strongly concave for a.e.\ $s$, then, for a.e.\ $s$ and all $a,\bar a\in\R^d$,
\[
\langle\nabla_aQ^\pi(s,a)-\nabla_aQ^\pi(s,\bar a),a-\bar a\rangle\le-\Bigl(\mu-\frac\gamma2L_pE^\pi\Bigr)|a-\bar a|^2.
\]
Consequently, if $\gamma L_pE^\pi<2\mu$, then $p^\pi_s$ satisfies $\mathrm{LSI}((\mu-\frac\gamma2L_pE^\pi)/\tau)$ for a.e.\ $s$. Moreover, $\pi^\star_s$ satisfies $\mathrm{LSI}(\mu/\tau)$ for a.e.\ $s$.
\itemlabel{itm:near-optimal-iii}{iii}
If $\pi^\star_s$ satisfies $\mathrm{LSI}(\alpha_\star)$ for a.e.\ $s$, then $W_2^2(\pi_s,\pi^\star_s)\le2E^\pi/(\tau\alpha_\star)$ for a.e.\ $s$, and the cross term satisfies
\[
\|(P^\pi-P^\star)W^\pi\|_\infty\le\frac{\min\{G_p/\sqrt{\alpha_\star},\,1\}}{\sqrt{2\tau}}\,(E^\pi)^{3/2}.
\]
\itemlabel{itm:near-optimal-iv}{iv}
$R^\pi\ge(I-\gamma P^\star)W^\pi$. Since $P^\star W^\pi\le E^\pi\mathbf 1$, this implies the Bellman comparison~\eqref{eq:residual-lower}.
\end{enumerate}
\end{lemma}

\begin{proof}
Fix $s$, and let $f(a):=(PW^\pi)(s,a)$ be the expected value gap at the next state after action~$a$. The function $f$ takes values in $[0,E^\pi]$ because ${0\le W^\pi\le E^\pi}$. Since $Q_V=\tilde r+\gamma PV$ is affine in $V$ and $V^\pi=V^\star-W^\pi$, we have $Q^\pi(s,\cdot)=Q^\star(s,\cdot)-\gamma f$, and~\eqref{eq:gibbs} gives
\begin{equation}\label{eq:gibbs-perturbation}
p^\pi_s\ \propto\ \pi^\star_s\,e^{-\gamma f/\tau}.
\end{equation}
\looseness=-1 As in the proof of Lemma~\ref{lem:drift}, Assumption~\ref{ass:smooth}\itemref{itm:smooth-b} justifies differentiation under the transition integral. Since $\int\nabla_ap(s'\mid s,a)\,\mathfrak m(\dd s')=0$, we may subtract the midpoint $\frac12E^\pi$ of $[0,E^\pi]$ from $W^\pi$, which gives
\[
\nabla f(a)=\int\Bigl(W^\pi(s')-\frac12E^\pi\Bigr)\nabla_ap(s'\mid s,a)\,\mathfrak m(\dd s').
\]
Since $|W^\pi-\frac12E^\pi|\le\frac12E^\pi$, Assumption~\ref{ass:smooth}\itemref{itm:smooth-b} gives
\begin{equation}\label{eq:continuation-gap}
\begin{aligned}
|\nabla f(a)|&\le\tfrac12G_pE^\pi,\\
|\nabla f(a)-\nabla f(\bar a)|&\le\tfrac12L_pE^\pi|a-\bar a|,\qquad a,\bar a\in\R^d.
\end{aligned}
\end{equation}

\emph{Part~\itemref{itm:near-optimal-i}.} By~\eqref{eq:gibbs-perturbation}, $p^\pi_s$ is the density $q$ of~\factref{F2} with $\pi^\star_s$ in place of $p$ and $\psi=-\gamma f/\tau$. Since $0\le f\le E^\pi$, the oscillation of $\psi$ is at most $\gamma E^\pi/\tau$, and~\factref{F2} gives $\mathrm{LSI}(\alpha_\star e^{-\gamma E^\pi/\tau})$.

\emph{Part~\itemref{itm:near-optimal-ii}.} Since $\nabla_aQ^\pi(s,\cdot)=\nabla_aQ^\star(s,\cdot)-\gamma\nabla f$, the strong concavity of $Q^\star(s,\cdot)$, the Cauchy--Schwarz inequality and the second bound in~\eqref{eq:continuation-gap} give, for all $a,\bar a$,
\[
\begin{aligned}
\langle\nabla_aQ^\pi(s,a)-\nabla_aQ^\pi(s,\bar a),a-\bar a\rangle
&\le-\mu|a-\bar a|^2+\gamma|\nabla f(a)-\nabla f(\bar a)|\,|a-\bar a|\\
&\le-\Bigl(\mu-\frac\gamma2L_pE^\pi\Bigr)|a-\bar a|^2.
\end{aligned}
\]
By~\eqref{eq:gibbs}, $\pi^\star_s$ is proportional to $e^{-\Phi}$ with $\Phi=-Q^\star(s,\cdot)/\tau$, and the strong concavity of $Q^\star(s,\cdot)$ gives the condition of~\factref{F3} with $c_\Phi=\mu/\tau$. Likewise, $p^\pi_s$ is proportional to $e^{-\Phi}$ with $\Phi=-Q^\pi(s,\cdot)/\tau$, and the last display gives this condition with $c_\Phi=(\mu-\frac\gamma2L_pE^\pi)/\tau$, which is positive if $\gamma L_pE^\pi<2\mu$. Hence~\factref{F3} gives the two LSIs.

\emph{Part~\itemref{itm:near-optimal-iii}.} By~\eqref{eq:optimal-gap},
\[
\tau\KL(\pi_s\,\|\,\pi^\star_s)=W^\pi(s)-\gamma(P^\pi W^\pi)(s)\le E^\pi.
\]
\looseness=-1 Since $\pi^\star_s$ satisfies $\mathrm{LSI}(\alpha_\star)$, it satisfies Talagrand's inequality~\cite[Theorem~1]{otto2000generalization}, so the last display gives
\[
W_2^2(\pi_s,\pi^\star_s)\le\frac2{\alpha_\star}\KL(\pi_s\,\|\,\pi^\star_s)\le\frac{2E^\pi}{\tau\alpha_\star}.
\]
The cross term at $s$ is the difference of the means of $f$ under $\pi_s$ and $\pi^\star_s$, and we bound it in two ways. First, the bound on $|\nabla f|$ in~\eqref{eq:continuation-gap} shows that $f$ is $\frac12G_pE^\pi$-Lipschitz. Hence, for every coupling $(X,Y)$ of $\pi_s$ and $\pi^\star_s$,
\[
\big|\bigl((P^\pi-P^\star)W^\pi\bigr)(s)\big|=\big|\E[f(X)-f(Y)]\big|\le\tfrac12G_pE^\pi\bigl(\E|X-Y|^2\bigr)^{1/2}.
\]
Taking the infimum over couplings and inserting the bound on $W_2^2$ gives
\[
\big|\bigl((P^\pi-P^\star)W^\pi\bigr)(s)\big|\le\tfrac12G_pE^\pi\,W_2(\pi_s,\pi^\star_s)\le\frac{G_p}{\sqrt{2\tau\alpha_\star}}\,(E^\pi)^{3/2}.
\]
Second, $\pi_s$ and $\pi^\star_s$ are probability densities, so subtracting the constant $\frac12E^\pi$ from $f$ does not change the difference of the means. Since $|f-\frac12E^\pi|\le\frac12E^\pi$, Pinsker's inequality and the bound on $\tau\KL(\pi_s\,\|\,\pi^\star_s)$ give
\[
\begin{aligned}
\big|\bigl((P^\pi-P^\star)W^\pi\bigr)(s)\big|
&=\Bigl|\int\bigl(f-\tfrac12E^\pi\bigr)(\pi_s-\pi^\star_s)\,\dd a\Bigr|
\le\tfrac12E^\pi\|\pi_s-\pi^\star_s\|_{L^1}\\
&\le\tfrac12E^\pi\bigl(2\KL(\pi_s\,\|\,\pi^\star_s)\bigr)^{1/2}
\le\frac{(E^\pi)^{3/2}}{\sqrt{2\tau}}.
\end{aligned}
\]
Taking the smaller of the two bounds proves part~\itemref{itm:near-optimal-iii}.

\emph{Part~\itemref{itm:near-optimal-iv}.} The operator $T^{\pi^\star}$ is affine with linear part $\gamma P^\star$, and $T^{\pi^\star}V^\star=V^\star$ by Lemma~\ref{lem:bellman-identities}\itemref{itm:bellman-identities-i}, so
\[
T^{\pi^\star}V^\pi=T^{\pi^\star}V^\star-\gamma P^\star(V^\star-V^\pi)=V^\star-\gamma P^\star W^\pi.
\]
By~\eqref{eq:bellman-gibbs} with $V=V^\pi$, $T^\star V^\pi=T^{\pi^\star}V^\pi+\tau\KL(\pi^\star\,\|\,p^\pi)$. Hence
\[
R^\pi=T^\star V^\pi-V^\pi=(I-\gamma P^\star)W^\pi+\tau\KL(\pi^\star\,\|\,p^\pi)\ge(I-\gamma P^\star)W^\pi.\qedhere
\]
\end{proof}

\subsection{Local Rates for \eqref{eq:WPGF}}\label{sec:local-rate}

\begin{theorem}[Fast local convergence of \eqref{eq:WPGF}]\label{thm:local-rate}
Let Assumptions~\ref{ass:base} and~\ref{ass:smooth} hold, let $(\pi_t)_{t\ge0}$ be the solution of~\eqref{eq:WPGF} from Theorem~\ref{thm:wellposed}, and suppose that, for some $\alpha_\star>0$, $\pi^\star_s$ satisfies $\mathrm{LSI}(\alpha_\star)$ for $\mathfrak m$-a.e.\ $s$. Let $\delta\in(0,1)$. If the initial value gap satisfies
\begin{equation}\label{eq:local-entry}
\frac{\gamma E(0)}{\tau}+B_\star E(0)^{1/2}\le\delta,
\end{equation}
where $B_\star:=\gamma G_p/((1-\gamma)\sqrt{2\tau\alpha_\star})$, then, for all $t\ge0$,
\begin{equation}\label{eq:local-rate}
E(t)\le e^{2\delta/(1-\delta)}E(0)\,e^{-2\tau\alpha_\star t}.
\end{equation}
\end{theorem}

By Lemma~\ref{lem:near-optimal}\itemref{itm:near-optimal-iii}, Theorem~\ref{thm:local-rate} also holds with $B_\star$ replaced by $\gamma/((1-\gamma)\sqrt{2\tau})$, which is smaller when $\alpha_\star<G_p^2$.

\begin{proof}
The argument follows the global convergence proof of Theorem~\ref{thm:continuous}: we derive and integrate a statewise differential inequality before taking the essential supremum. Here, the near-optimal LSI and refined Bellman comparison make the coefficients of the inequality depend on the current value gap. We show that this dependence causes only a finite loss in the exponent.

To derive the differential inequality, fix $t>0$. By Lemma~\ref{lem:near-optimal}\itemref{itm:near-optimal-i}, $p^{\pi_t}_s$ satisfies $\mathrm{LSI}(\omega(t)/(2\tau))$ with $\omega(t):=2\tau\alpha_\star e^{-\gamma E(t)/\tau}$, so Proposition~\ref{prop:dominance}\itemref{itm:dominance-i} gives
\[
g_t=\tau^2\mathcal I(\pi_t\,\|\,p^{\pi_t})\ge\omega(t)R_t.
\]
Unlike in the proof of Theorem~\ref{thm:continuous}, we keep the Bellman resolvent in~\eqref{eq:monotone}. Its positivity from Lemma~\ref{lem:soft-bellman}, together with Lemma~\ref{lem:near-optimal}\itemref{itm:near-optimal-iv}, gives
\begin{align*}
\dot V_t=(I-\gamma P^{\pi_t})^{-1}g_t&\ge\omega(t)(I-\gamma P^{\pi_t})^{-1}R_t\\
&\ge\omega(t)(I-\gamma P^{\pi_t})^{-1}(I-\gamma P^\star)W_t\\
&=\omega(t)\Bigl(W_t+\gamma(I-\gamma P^{\pi_t})^{-1}(P^{\pi_t}-P^\star)W_t\Bigr).
\end{align*}
The last equality follows from $I-\gamma P^\star=(I-\gamma P^{\pi_t})+\gamma(P^{\pi_t}-P^\star)$. In the last line, only the second term carries the resolvent, whose norm is $(1-\gamma)^{-1}$. Lemma~\ref{lem:near-optimal}\itemref{itm:near-optimal-iii} bounds the sup norm of the cross term $(P^{\pi_t}-P^\star)W_t$ by $G_pE(t)^{3/2}/\sqrt{2\tau\alpha_\star}$, so
\[
\gamma(I-\gamma P^{\pi_t})^{-1}(P^{\pi_t}-P^\star)W_t\ge-\frac{\gamma}{1-\gamma}\,\frac{G_pE(t)^{3/2}}{\sqrt{2\tau\alpha_\star}}\,\mathbf 1=-B_\star E(t)^{3/2}\mathbf 1.
\]
Inserting this bound into the lower bound on $\dot V_t$ gives
\begin{equation}\label{eq:local-ode}
\dot W_t=-\dot V_t\le-\omega(t)\bigl(W_t-\bar\gamma(t)E(t)\mathbf 1\bigr),\qquad\bar\gamma(t):=B_\star E(t)^{1/2}.
\end{equation}
Compared with~\eqref{eq:ct-ode}, the coefficient $\gamma$ is replaced by $\bar\gamma(t)$, which vanishes as $E(t)\to0$, so the error term is of order $E(t)^{3/2}$.

Integrating~\eqref{eq:local-ode} at each state before taking the essential supremum, as in the proof of Theorem~\ref{thm:continuous}, and using the integrating factor $e^{\int_0^q\omega}$ in place of $e^{2\tau\alpha q}$ gives, for $0<t_0<t$,
\[
u(t)\le u(t_0)+\int_{t_0}^t\omega(q)\bar\gamma(q)u(q)\,\dd q,\qquad u(q):=e^{\int_0^q\omega}E(q).
\]
Gr\"onwall's inequality gives
\[
E(t)\le E(t_0)\exp\Bigl(-\int_{t_0}^t\omega(q)\bigl(1-\bar\gamma(q)\bigr)\,\dd q\Bigr).
\]
By~\eqref{eq:monotone}, $E$ is nonincreasing, so $E(t_0)\le E(0)$. Since $\omega$ and $\bar\gamma$ are bounded near zero, letting $t_0\downarrow0$ gives
\begin{equation}\label{eq:local-gronwall}
E(t)\le E(0)\exp\Bigl(-\int_0^t\omega(q)\bigl(1-\bar\gamma(q)\bigr)\,\dd q\Bigr),\qquad t\ge0.
\end{equation}

To recover the full local exponent from this bound, we show that the loss in the exponent is integrable. Since $E$ is nonincreasing, \eqref{eq:local-entry} gives ${\gamma E(q)/\tau+\bar\gamma(q)\le\delta<1}$ for all $q\ge0$. Since $e^{-x}\ge1-x$ and $(1-x)(1-y)\ge1-x-y$ for $x,y\in[0,1]$,
\begin{equation}\label{eq:local-integrand}
\omega(q)\bigl(1-\bar\gamma(q)\bigr)\ge2\tau\alpha_\star\Bigl(1-\frac{\gamma E(q)}{\tau}-\bar\gamma(q)\Bigr)\ge2\tau\alpha_\star(1-\delta),\qquad q\ge0.
\end{equation}
Thus the integrand in~\eqref{eq:local-gronwall} is at least $2\tau\alpha_\star(1-\delta)$, which gives ${E(q)\le E(0)\,e^{-2\tau\alpha_\star(1-\delta)q}}$. Consequently, ${\bar\gamma(q)\le\bar\gamma(0)\,e^{-\tau\alpha_\star(1-\delta)q}}$. Integrating these two bounds over $[0,\infty)$ shows that the accumulated loss in the exponent is finite:
\[
2\tau\alpha_\star\int_0^\infty\Bigl(\frac{\gamma E(q)}{\tau}+\bar\gamma(q)\Bigr)\,\dd q\le\frac1{1-\delta}\Bigl(\frac{\gamma E(0)}{\tau}+2\bar\gamma(0)\Bigr)\le\frac{2\delta}{1-\delta}.
\]
The first inequality in~\eqref{eq:local-integrand} and~\eqref{eq:local-gronwall} therefore give
\[
E(t)\le E(0)\exp\Bigl(-2\tau\alpha_\star t+\frac{2\delta}{1-\delta}\Bigr),
\]
which is~\eqref{eq:local-rate}.
\end{proof}

Theorem~\ref{thm:local-rate} assumes an LSI for the optimal policy. Strong concavity of $Q^\star$ provides this LSI and also yields a curvature bound for $Q^\pi$ near the optimum. The following corollary uses that inherited curvature to give a local rate under strong concavity at the optimum alone.

\begin{corollary}[Continuous-time local rate under strong concavity of $Q^\star$]\label{cor:concave-optimum}
Let Assumptions~\ref{ass:base} and~\ref{ass:smooth} hold, let $(\pi_t)_{t\ge0}$ be the solution of~\eqref{eq:WPGF} from Theorem~\ref{thm:wellposed}, and suppose that, for some $\mu>0$, $Q^\star(s,\cdot)$ is $\mu$-strongly concave for $\mathfrak m$-a.e.\ $s$. Let $\delta\in(0,1)$. If the initial value gap satisfies
\begin{equation}\label{eq:concave-entry}
A_\mu E(0)+B_\mu E(0)^{1/2}\le\delta,
\end{equation}
where $A_\mu:=\gamma L_p/(2\mu)$ and $B_\mu:=\gamma G_p/((1-\gamma)\sqrt{2\mu})$, then, for all $t\ge0$,
\begin{equation}\label{eq:concave-rate}
E(t)\le e^{2\delta/(1-\delta)}E(0)\,e^{-2\mu t}.
\end{equation}
\end{corollary}

Corollary~\ref{cor:concave-optimum} imposes no concavity on $Q^\pi$ for $\pi\ne\pi^\star$. By Lemma~\ref{lem:near-optimal}\itemref{itm:near-optimal-ii}, the curvature of $Q^\star$ passes to $Q^{\pi_t}$ as soon as $\gamma L_pE(t)<2\mu$. With $\alpha_\star=\mu/\tau$, we have $B_\star=B_\mu$, so~\eqref{eq:concave-entry} is weaker than~\eqref{eq:local-entry} whenever $A_\mu\le\gamma/\tau$, that is, $\tau L_p\le2\mu$. For fixed $\mu$, $G_p$, $L_p$ and $\gamma$, the neighborhood~\eqref{eq:concave-entry}, the rate $2\mu$ and the prefactor in~\eqref{eq:concave-rate} do not depend on $\tau$.

\begin{proof}
We follow the proof of Theorem~\ref{thm:local-rate} with $\omega(t):=2\mu(1-A_\mu E(t))$ and $\bar\gamma(t):=B_\mu E(t)^{1/2}$. Since $E$ is nonincreasing, \eqref{eq:concave-entry} gives $A_\mu E(t)+\bar\gamma(t)\le\delta<1$ for all $t\ge0$; in particular, $\gamma L_pE(t)<2\mu$. For $t>0$, Lemma~\ref{lem:near-optimal}\itemref{itm:near-optimal-ii} therefore shows that $p^{\pi_t}_s$ satisfies $\mathrm{LSI}(\omega(t)/(2\tau))$ and that $\pi^\star_s$ satisfies $\mathrm{LSI}(\mu/\tau)$. By Proposition~\ref{prop:dominance}\itemref{itm:dominance-i}, the first LSI gives $g_t\ge\omega(t)R_t$. Lemma~\ref{lem:near-optimal}\itemref{itm:near-optimal-iii} with $\alpha_\star=\mu/\tau$ then bounds the cross term by $G_pE(t)^{3/2}/\sqrt{2\mu}$. These two bounds then give~\eqref{eq:local-ode} with the present $\omega$ and $\bar\gamma$. Since ${(1-x)(1-y)\ge1-x-y}$ for $x,y\ge0$, we obtain the counterpart of~\eqref{eq:local-integrand},
\[
\omega(q)\bigl(1-\bar\gamma(q)\bigr)\ge2\mu\bigl(1-A_\mu E(q)-\bar\gamma(q)\bigr)\ge2\mu(1-\delta),\qquad q\ge0.
\]
The rest of the proof of Theorem~\ref{thm:local-rate}, with $2\mu$ in place of $2\tau\alpha_\star$ and $A_\mu E$ in place of $\gamma E/\tau$, gives~\eqref{eq:concave-rate}.
\end{proof}

\begin{remark}[From global to local convergence]\label{rem:global-local}
By Theorem~\ref{thm:continuous}, $E(t)\to0$, so for any common LSI constant $\alpha_\star$ of $\pi^\star$ and any $\delta\in(0,1)$, \eqref{eq:local-entry} holds with $E(t_0)$ in place of $E(0)$ for all sufficiently large $t_0$. By~\eqref{eq:weak-class}, $\pi_{t_0}$ is admissible. The equation is autonomous, so $(\pi_{t_0+t})_{t\ge0}$ is an entropy-bounded weak solution from $\pi_{t_0}$ and is the solution identified by Theorem~\ref{thm:wellposed}. Applying Theorem~\ref{thm:local-rate} to it gives $\limsup_{t\to\infty}t^{-1}\log E(t)\le-2\tau\alpha_\star$. Corollary~\ref{cor:concave-optimum} applies in the same way when $Q^\star$ is strongly concave. Without further information on $\pi^\star$, Lemmas~\ref{lem:bellman-identities}\itemref{itm:bellman-identities-iii} and~\ref{lem:gibbs-bounds}\itemref{itm:gibbs-bounds-ii} give only the Holley--Stroock constant ${(\beta/\tau)\exp[-\Delta_r/(\tau(1-\gamma))]}$ of $\pi^\star$.
\end{remark}

\subsection{Local Rates for \eqref{eq:WPG}}\label{sec:dt-local}

We now establish local convergence guarantees for the discrete-time WPG iterates, complementing the continuous-time rates above.

For a radius $\bar E>0$, set
\[
L_{\bar E}:=\beta+L_r+\frac\gamma2L_p\Bigl(\frac{\Delta_r}{1-\gamma}+\bar E\Bigr).
\]
By Lemma~\ref{lem:bellman-identities}\itemref{itm:bellman-identities-ii} and~\itemref{itm:bellman-identities-iii}, every $\pi\in\Pi$ with $E^\pi\le\bar E$ has values in $[V_+-\Delta_r/(1-\gamma)-\bar E,V_+]$, so its drift $\nabla_aQ^\pi(s,\cdot)$ is $L_{\bar E}$-Lipschitz by Lemma~\ref{lem:drift}.

\begin{theorem}[Fast local convergence of \eqref{eq:WPG}]\label{thm:dt-local}
Let Assumptions~\ref{ass:base} and~\ref{ass:smooth} hold, and suppose that, for some $\alpha_\star>0$, $\pi^\star_s$ satisfies $\mathrm{LSI}(\alpha_\star)$ for $\mathfrak m$-a.e.\ $s$. Let $\delta\in(0,1)$, set $\kappa_\star:=\tau\alpha_\star$, and let $\bar E>0$ satisfy $\gamma\bar E/\tau+B_\star\bar E^{1/2}\le\delta/2$, with $B_\star$ as in Theorem~\ref{thm:local-rate}. Run~\eqref{eq:WPG} with a constant step size
\begin{equation}\label{eq:dt-local-step}
0<\eta\le\min\Bigl\{\frac{\delta}{8L_{\bar E}},\ \frac{\delta}{4\kappa_\star},\ \frac{\delta(1-\delta)\kappa_\star(1-\gamma)\bar E}{3\tau dL_{\bar E}^2}\Bigr\}.
\end{equation}
If $E_0\le\bar E$, then every iterate is admissible and satisfies $E_k\le\bar E$, and, for all $k\ge0$,
\begin{equation}\label{eq:dt-local-rate}
E_k\le e^{-2\kappa_\star(1-\delta)\eta k}E_0+\frac{3\tau dL_{\bar E}^2}{\delta(1-\delta)\kappa_\star(1-\gamma)}\,\eta.
\end{equation}
\end{theorem}

Writing $t=\eta k$, the exponential term in~\eqref{eq:dt-local-rate} decays at rate $2\kappa_\star(1-\delta)$, which approaches the continuous-time rate $2\tau\alpha_\star$ as $\delta\to0$. This sharper rate comes with a trade-off: the condition on $\bar E$ requires a smaller neighborhood, the step-size restriction tightens, and the coefficient of the $O(\eta)$ error grows like $1/[\delta(1-\delta)]$. Compared with the global bound~\eqref{eq:dt-rate}, whose error coefficient contains $(1-\gamma)^{-2}$, the local bound has only one explicit factor $(1-\gamma)^{-1}$. For fixed $\delta$ and $\bar E$, the step-size restriction and the constants of Theorem~\ref{thm:dt-local} depend only on the model and the geometry of $\pi^\star$. The following corollary replaces the optimal-policy LSI assumption with strong concavity of $Q^\star$.

\begin{corollary}[Discrete-time local rate under strong concavity of $Q^\star$]\label{cor:dt-concave}
Let Assumptions~\ref{ass:base} and~\ref{ass:smooth} hold, and suppose that, for some $\mu>0$, $Q^\star(s,\cdot)$ is $\mu$-strongly concave for $\mathfrak m$-a.e.\ $s$. Let $\delta\in(0,1)$, set $\kappa_\star:=\mu$, and let $\bar E>0$ satisfy $A_\mu\bar E+B_\mu\bar E^{1/2}\le\delta/2$, with $A_\mu$ and $B_\mu$ as in Corollary~\ref{cor:concave-optimum}. Run~\eqref{eq:WPG} with a constant step size satisfying~\eqref{eq:dt-local-step}. If $E_0\le\bar E$, then every iterate is admissible and satisfies $E_k\le\bar E$, and~\eqref{eq:dt-local-rate} holds for all $k\ge0$.
\end{corollary}

By Lemma~\ref{lem:near-optimal}\itemref{itm:near-optimal-ii}, $\pi^\star_s$ satisfies $\mathrm{LSI}(\mu/\tau)$, so Theorem~\ref{thm:dt-local} with $\alpha_\star=\mu/\tau$ already gives the rate $2\mu(1-\delta)$ in $t=\eta k$. Corollary~\ref{cor:dt-concave} replaces $\gamma/\tau$ by $A_\mu$ in the condition on $\bar E$. As in Corollary~\ref{cor:concave-optimum}, this condition does not depend on $\tau$ and is weaker whenever $\tau L_p\le2\mu$.

Under the assumptions of Theorem~\ref{thm:discrete}, the global analysis also provides entry into the neighborhood. Run~\eqref{eq:WPG} first with a constant step $\eta'$ such that $\eta'\le\tau/C$ and $C\eta'\le\bar E/2$, and let $K$ be the smallest integer $k\ge0$ with $e^{-\bar\lambda\eta'k}E_0\le\bar E/2$. Then~\eqref{eq:dt-rate} gives $E_K\le\bar E$, and $\pi_K$ is admissible by Theorem~\ref{thm:discrete}. From $\pi_K$ on, run~\eqref{eq:WPG} with a constant step $\eta$ satisfying~\eqref{eq:dt-local-step}. Theorem~\ref{thm:dt-local}, applied from $\pi_K$, gives~\eqref{eq:dt-local-rate} with $k-K$ and $E_K$ in place of $k$ and $E_0$. Unlike $\tau/C$, the restriction~\eqref{eq:dt-local-step} does not contain the worst-case constant $\bar\alpha$, so $\eta$ may be much larger than $\eta'$. Corollary~\ref{cor:dt-concave} applies in the same way.

The proof of Theorem~\ref{thm:dt-local} needs a finite-step counterpart of the resolvent comparison leading to~\eqref{eq:local-ode}. In the soft performance difference lemma (Lemma~\ref{lem:pdl}), the resolvent belongs to the new policy~$\pi^+$ while the residual is evaluated at the old policy~$\pi$, so the cross term becomes ${(P^{\pi^+}-P^\star)W^\pi}$. Lemma~\ref{lem:local-recursion} controls this cross term and, since a discrete-time WPG step need not decrease the value gap, also shows that the new policy stays in the neighborhood. Combining the recursion of Lemma~\ref{lem:local-recursion} with the one-step estimate of Lemma~\ref{lem:ula} gives Theorem~\ref{thm:dt-local}. Lemma~\ref{lem:near-optimal}\itemref{itm:near-optimal-ii} then yields Corollary~\ref{cor:dt-concave}.

\begin{lemma}[Local value-gap recursion]\label{lem:local-recursion}
Let Assumptions~\ref{ass:base} and~\ref{ass:smooth} hold, suppose that $\pi^\star_s$ satisfies $\mathrm{LSI}(\alpha_\star)$ for $\mathfrak m$-a.e.\ $s$, and let $B_\star$ be as in Theorem~\ref{thm:local-rate}. Let $\pi,\pi^+\in\Pi$ satisfy
\[
\KL(\pi^+\,\|\,p^\pi)\le(1-\theta)\KL(\pi\,\|\,p^\pi)+\varepsilon_{\mathrm{KL}}\qquad\text{for a.e.\ }s,
\]
for some $\theta\in[0,1]$ and $\varepsilon_{\mathrm{KL}}\ge0$.
\begin{enumerate}
\renewcommand{\theenumi}{\roman{enumi}}
\itemlabel{itm:local-recursion-i}{i} $E^{\pi^+}\le(1-\theta)E^\pi+\theta B_\star E^\pi(E^{\pi^+})^{1/2}+\dfrac{\tau\varepsilon_{\mathrm{KL}}}{1-\gamma}$.
\itemlabel{itm:local-recursion-ii}{ii} If $E^\pi\le\bar E$ for some $\bar E>0$ with $\bar\gamma:=B_\star\bar E^{1/2}<1$, and $\tau\varepsilon_{\mathrm{KL}}/(1-\gamma)\le(1-\bar\gamma)\theta\bar E$, then $E^{\pi^+}\le\bar E$ and
\begin{equation}\label{eq:local-recursion}
E^{\pi^+}\le\bigl[1-(1-\bar\gamma)\theta\bigr]E^\pi+\frac{\tau\varepsilon_{\mathrm{KL}}}{1-\gamma}.
\end{equation}
\end{enumerate}
\end{lemma}

\begin{proof}
To prove~\itemref{itm:local-recursion-i}, the soft performance difference lemma (Lemma~\ref{lem:pdl}), the residual identity~\eqref{eq:residual-kl} and the assumption imply $(I-\gamma P^{\pi^+})(V^{\pi^+}-V^\pi)\ge\theta R^\pi-\tau\varepsilon_{\mathrm{KL}}\mathbf 1$. Lemma~\ref{lem:soft-bellman} shows that the Bellman resolvent $(I-\gamma P^{\pi^+})^{-1}$ is positive and maps $\mathbf 1$ to $(1-\gamma)^{-1}\mathbf 1$. Applying it to this inequality and inserting the refined Bellman comparison $R^\pi\ge(I-\gamma P^\star)W^\pi$ of Lemma~\ref{lem:near-optimal}\itemref{itm:near-optimal-iv} give
\[
\begin{aligned}
V^{\pi^+}-V^\pi&\ge\theta(I-\gamma P^{\pi^+})^{-1}(I-\gamma P^\star)W^\pi-\frac{\tau\varepsilon_{\mathrm{KL}}}{1-\gamma}\mathbf 1\\
&=\theta W^\pi+\theta\gamma(I-\gamma P^{\pi^+})^{-1}(P^{\pi^+}-P^\star)W^\pi-\frac{\tau\varepsilon_{\mathrm{KL}}}{1-\gamma}\mathbf 1,
\end{aligned}
\]
that is, in terms of the value gaps,
\begin{equation}\label{eq:local-statewise}
W^{\pi^+}\le(1-\theta)W^\pi-\theta\gamma(I-\gamma P^{\pi^+})^{-1}(P^{\pi^+}-P^\star)W^\pi+\frac{\tau\varepsilon_{\mathrm{KL}}}{1-\gamma}\mathbf 1.
\end{equation}
We bound the cross term as in the proof of Lemma~\ref{lem:near-optimal}\itemref{itm:near-optimal-iii}, now with the new policy in the Wasserstein distance and the old gap in the integrand. For a.e.\ $s$, the function $a\mapsto(PW^\pi)(s,a)$ is $\frac12G_pE^\pi$-Lipschitz by~\eqref{eq:continuation-gap}, and $W_2(\pi^+_s,\pi^\star_s)\le(2E^{\pi^+}/(\tau\alpha_\star))^{1/2}$ by Lemma~\ref{lem:near-optimal}\itemref{itm:near-optimal-iii} applied to $\pi^+$. Hence
\[
\|(P^{\pi^+}-P^\star)W^\pi\|_\infty\le\frac{G_p}{\sqrt{2\tau\alpha_\star}}\,E^\pi(E^{\pi^+})^{1/2}.
\]
The resolvent norm bound in Lemma~\ref{lem:soft-bellman}, together with the definition of $B_\star$, shows that the second term on the right of~\eqref{eq:local-statewise} is at most $\theta B_\star E^\pi(E^{\pi^+})^{1/2}\mathbf 1$. Since $\theta\le1$, $W^\pi\le E^\pi\mathbf 1$ and $W^{\pi^+}\ge0$, taking the essential supremum in~\eqref{eq:local-statewise} proves part~\itemref{itm:local-recursion-i}.

To prove~\itemref{itm:local-recursion-ii}, part~\itemref{itm:local-recursion-i} and $E^\pi\le\bar E$ imply that the convex quadratic
\[
q(x):=x^2-\theta B_\star\bar E\,x-(1-\theta)\bar E-\frac{\tau\varepsilon_{\mathrm{KL}}}{1-\gamma}
\]
satisfies $q((E^{\pi^+})^{1/2})\le0$. Moreover, $q(0)\le0$, and the error condition gives
\[
q(\bar E^{1/2})=(1-\bar\gamma)\theta\bar E-\frac{\tau\varepsilon_{\mathrm{KL}}}{1-\gamma}\ge0.
\]
Hence the larger root of $q$ lies in $[0,\bar E^{1/2}]$, and $E^{\pi^+}\le\bar E$. Substituting $(E^{\pi^+})^{1/2}\le\bar E^{1/2}$ into part~\itemref{itm:local-recursion-i} gives~\eqref{eq:local-recursion}.
\end{proof}

\begin{proof}[Proof of Theorem~\ref{thm:dt-local}]
The proof closes an induction on the neighborhood $E_k\le\bar E$. Lemma~\ref{lem:near-optimal}\itemref{itm:near-optimal-i} gives a common local LSI constant, and Lemma~\ref{lem:drift} gives the corresponding drift Lipschitz bound. Lemma~\ref{lem:ula} then provides a one-step KL estimate. The step-size restrictions in~\eqref{eq:dt-local-step} make its error small enough for Lemma~\ref{lem:local-recursion}\itemref{itm:local-recursion-ii} to preserve the neighborhood and produce a contraction. Iterating that contraction yields~\eqref{eq:dt-local-rate}.

We first fix constants that are uniform over the neighborhood. Set $\zeta:=\delta/4$, $A:=\gamma/\tau$ and $B:=B_\star$, so that the condition on $\bar E$ reads $A\bar E+B\bar E^{1/2}\le\delta/2$. Let $\alpha_{\mathrm{loc}}:=\alpha_\star e^{-A\bar E}$. Since $\tau\alpha_{\mathrm{loc}}=\kappa_\star e^{-A\bar E}$ and $e^{-x}\ge1-x$,
\[
\kappa_\star(1-A\bar E)\le\tau\alpha_{\mathrm{loc}}\le\kappa_\star.
\]
For every $\pi\in\Pi$ with $E^\pi\le\bar E$, the Gibbs policy $p^\pi_s$ satisfies $\mathrm{LSI}(\alpha_{\mathrm{loc}})$ by Lemma~\ref{lem:near-optimal}\itemref{itm:near-optimal-i}. Since $\pi^\star_s$ satisfies $\mathrm{LSI}(\alpha_\star)$, Lemma~\ref{lem:local-recursion} applies with the constant $B$. Write $\bar\gamma:=B\bar E^{1/2}$, so that $\bar\gamma+A\bar E\le\delta/2$. Since ${(1-x)(1-y)\ge1-x-y}$ for $x,y\ge0$,
\begin{equation}\label{eq:dt-local-a}
(1-\bar\gamma)\,\tau\alpha_{\mathrm{loc}}\ge\kappa_\star(1-\bar\gamma)(1-A\bar E)\ge\kappa_\star(1-\bar\gamma-A\bar E)\ge\Bigl(1-\frac\delta2\Bigr)\kappa_\star.
\end{equation}

We now show by induction that $\pi_k\in\Pi$ and $E_k\le\bar E$ for all $k\ge0$. Both hold for $k=0$. Suppose that they hold for some $k$. Lemma~\ref{lem:bellman-identities}\itemref{itm:bellman-identities-ii}--\itemref{itm:bellman-identities-iii} and Lemma~\ref{lem:drift} give that $\nabla_aQ^{\pi_k}(s,\cdot)$ is $L_{\bar E}$-Lipschitz. Lemma~\ref{lem:near-optimal}\itemref{itm:near-optimal-i} gives $\mathrm{LSI}(\alpha_{\mathrm{loc}})$ for $p^{\pi_k}_s$. Since $\eta\le\zeta/(2L_{\bar E})$, Lemma~\ref{lem:ula} with $\alpha_{\mathrm{loc}}$ and $L_{\bar E}$ in place of $\alpha$ and $L$ gives
\[
\KL(\pi_{k+1}\,\|\,p^{\pi_k})\le(1-\theta)\KL(\pi_k\,\|\,p^{\pi_k})+\varepsilon_{\mathrm{loc}},
\]
where $\theta:=1-e^{-2(1-\zeta)\alpha_{\mathrm{loc}}\tau\eta}$ and $\varepsilon_{\mathrm{loc}}:=6dL_{\bar E}^2\eta^2/\delta$.
The Gaussian-mixture representation of~\eqref{eq:WPG}, the change-of-reference identity~\eqref{eq:change-reference}, and Lemma~\ref{lem:gibbs-bounds}\itemref{itm:gibbs-bounds-i}, applied as in the proof of Proposition~\ref{prop:dt-stability}, show that the next iterate is admissible: $\pi_{k+1}\in\Pi$.

It remains to show that $E_{k+1}\le\bar E$, for which we apply Lemma~\ref{lem:local-recursion}\itemref{itm:local-recursion-ii}. To this end, we bound $(1-\bar\gamma)\theta$ from below. Since $\tau\alpha_{\mathrm{loc}}\le\kappa_\star$ and $\eta\le\delta/(4\kappa_\star)$, we have $(1-\zeta)\alpha_{\mathrm{loc}}\tau\eta\le\kappa_\star\eta\le\delta/4$. The inequality $1-e^{-u}\ge u(1-u/2)$ with $u=2(1-\zeta)\alpha_{\mathrm{loc}}\tau\eta$ and~\eqref{eq:dt-local-a} give
\[
\begin{aligned}
(1-\bar\gamma)\theta&\ge2(1-\zeta)\bigl(1-(1-\zeta)\alpha_{\mathrm{loc}}\tau\eta\bigr)(1-\bar\gamma)\tau\alpha_{\mathrm{loc}}\,\eta\ge2\kappa_\star\eta\Bigl(1-\frac\delta4\Bigr)^2\Bigl(1-\frac\delta2\Bigr)\ge2\kappa_\star(1-\delta)\eta,
\end{aligned}
\]
where the last step uses $(1-\frac\delta4)^2(1-\frac\delta2)\ge1-\delta$ for $\delta\in(0,1)$. The last restriction in~\eqref{eq:dt-local-step} is equivalent to $\tau\varepsilon_{\mathrm{loc}}/(1-\gamma)\le2\kappa_\star(1-\delta)\eta\bar E$, so $\tau\varepsilon_{\mathrm{loc}}/(1-\gamma)\le(1-\bar\gamma)\theta\bar E$. Lemma~\ref{lem:local-recursion}\itemref{itm:local-recursion-ii} with $\varepsilon_{\mathrm{KL}}=\varepsilon_{\mathrm{loc}}$ therefore gives $E_{k+1}\le\bar E$, which completes the induction, and
\[
E_{k+1}\le\bigl(1-2\kappa_\star(1-\delta)\eta\bigr)E_k+\frac{\tau\varepsilon_{\mathrm{loc}}}{1-\gamma}.
\]
Finally, since $0<2\kappa_\star(1-\delta)\eta\le\delta/2<1$, iterating this recursion and summing the geometric series give
\[
\begin{aligned}
E_k&\le\bigl(1-2\kappa_\star(1-\delta)\eta\bigr)^kE_0+\frac{\tau\varepsilon_{\mathrm{loc}}}{2\kappa_\star(1-\delta)(1-\gamma)\eta}\le e^{-2\kappa_\star(1-\delta)\eta k}E_0+\frac{3\tau dL_{\bar E}^2}{\delta(1-\delta)\kappa_\star(1-\gamma)}\,\eta,
\end{aligned}
\]
which is~\eqref{eq:dt-local-rate}.
\end{proof}

\begin{proof}[Proof of Corollary~\ref{cor:dt-concave}]
We follow the proof of Theorem~\ref{thm:dt-local} with $\alpha_{\mathrm{loc}}:=(\mu-\frac\gamma2L_p\bar E)/\tau$, $A:=A_\mu$ and $B:=B_\mu$. For every $\pi\in\Pi$ with $E^\pi\le\bar E$, we have $\gamma L_pE^\pi\le2\mu A_\mu\bar E<2\mu$, so Lemma~\ref{lem:near-optimal}\itemref{itm:near-optimal-ii} shows that $p^\pi_s$ satisfies $\mathrm{LSI}(\alpha_{\mathrm{loc}})$. The same lemma shows that $\pi^\star_s$ satisfies $\mathrm{LSI}(\mu/\tau)$. For $\alpha_\star=\mu/\tau$, we have $B_\star=B_\mu$, so Lemma~\ref{lem:local-recursion} applies with the constant $B$. Since $\tau\alpha_{\mathrm{loc}}=\kappa_\star(1-A\bar E)$, the two-sided bound on $\tau\alpha_{\mathrm{loc}}$ in the proof of Theorem~\ref{thm:dt-local} holds. The rest of that proof is unchanged, with Lemma~\ref{lem:near-optimal}\itemref{itm:near-optimal-ii} in place of Lemma~\ref{lem:near-optimal}\itemref{itm:near-optimal-i}.
\end{proof}

\section{Concluding Remarks}\label{sec:conclusion}

We establish global and local convergence guarantees for Wasserstein policy gradient in entropy-regularized discounted Markov decision processes with general state spaces and unbounded continuous actions. Our analysis exploits the Bellman structure and log-Sobolev inequalities to relate the Wasserstein policy gradient to the value gap, enabling convergence guarantees despite the nonconvexity of policy optimization. The results establish well-posedness and exponential convergence of the continuous-time flow, stability and convergence of the discrete-time update, and faster local rates governed by the geometry of the optimal policy. This theory provides a foundation for future analyses of approximate WPG methods.

\appendix
\section{Entropy and Soft Bellman Operators}\label{app:entropy}

This appendix provides three analytic ingredients used throughout the paper. Lemma~\ref{lem:entropy-integrability} derives moment and entropy consequences of finite relative entropy and the change-of-reference identity~\eqref{eq:change-reference}. Lemma~\ref{lem:soft-bellman} proves the contraction and positivity properties of the soft Bellman operators and the resolvent~\eqref{eq:resolvent}. Lemma~\ref{lem:entropy-calculus} gives the entropy subgradient inequality~\eqref{eq:hwi-two-laws}, the chain rule~\eqref{eq:weak-entropy-chain}, and the tangent-space property used in Appendices~\ref{app:metric} and~\ref{app:wellposed}.

\subsection{Integrability and the Bellman Resolvent}

\begin{lemma}[Elementary entropy facts]\label{lem:entropy-integrability}
Let $p$ be a positive probability density and let $\mu\ll p\,\dd a$ with $\KL(\mu\,\|\,p)<\infty$.
\begin{enumerate}
\itemlabel{itm:entropy-integrability-i}{i} $\int|\log(\dd\mu/(p\,\dd a))|\,\dd\mu\le\KL(\mu\,\|\,p)+2/e$. If $\int|\log p|\,\dd\mu<\infty$, the entropy integral of $\mu$ is absolutely convergent.
\itemlabel{itm:entropy-integrability-ii}{ii} For $p=\rho_\beta$,
\begin{equation}\label{eq:entropy-moment}
\frac{\beta}{4\tau}\int|a|^2\,\mu(\dd a)\le\KL(\mu\,\|\,\rho_\beta)+\frac d2\log2,
\end{equation}
and $\mu$ has a finite second moment and an absolutely convergent entropy integral.
\itemlabel{itm:entropy-integrability-iii}{iii} If $p_1,p_2$ are positive probability densities such that $\log(p_1/p_2)$ is essentially bounded, then $\KL(\mu\,\|\,p_1)<\infty$ if and only if ${\KL(\mu\,\|\,p_2)<\infty}$, and
\begin{equation}\label{eq:change-reference}
\KL(\mu\,\|\,p_2)=\KL(\mu\,\|\,p_1)+\int\log(p_1/p_2)\,\dd\mu.
\end{equation}
\end{enumerate}
All bounds are uniform over a family of laws with a common bound on the relative entropy.
\end{lemma}

\begin{proof}
\emph{Part~\itemref{itm:entropy-integrability-i}.} Write $f:=\dd\mu/(p\,\dd a)$. Since $-x\log x\le1/e$ for $0\le x\le1$, the negative part of $f\log f$ has integral at most $1/e$ against $p$. The signed integral of $f\log f$ against $p$ equals $\KL(\mu\,\|\,p)$, and the absolute integral exceeds it by twice the negative part, which proves the bound. The identity $\log(\dd\mu/\dd a)=\log f+\log p$, valid $\mu$-a.e., gives the second assertion.

\emph{Part~\itemref{itm:entropy-integrability-ii}.} For $N\in\mathbb N$ let $g_N(a):=\min\{\beta|a|^2/(4\tau),N\}$. Since $g_N$ is bounded, part~\itemref{itm:entropy-integrability-iii} gives the equality below. The last inequality uses $\int e^{\beta|a|^2/(4\tau)}\rho_\beta(a)\,\dd a=2^{d/2}$.
\[
0\le\KL\Bigl(\mu\,\Big\|\,\frac{e^{g_N}\rho_\beta}{\int e^{g_N}\rho_\beta\,\dd a}\Bigr)=\KL(\mu\,\|\,\rho_\beta)-\int g_N\,\dd\mu+\log\int e^{g_N}\rho_\beta\,\dd a\le\KL(\mu\,\|\,\rho_\beta)-\int g_N\,\dd\mu+\frac d2\log2.
\]
Rearranging and applying monotone convergence prove~\eqref{eq:entropy-moment}. Part~\itemref{itm:entropy-integrability-i} shows that the entropy integral is absolutely convergent, since ${|\log\rho_\beta(a)|\le|\log Z_\beta|+\beta|a|^2/(2\tau)}$.

\emph{Part~\itemref{itm:entropy-integrability-iii}.} Write $\log(\dd\mu/(p_2\,\dd a))=\log(\dd\mu/(p_1\,\dd a))+\log(p_1/p_2)$. The second term is bounded, and part~\itemref{itm:entropy-integrability-i} shows that the first is $\mu$-integrable whenever $\KL(\mu\,\|\,p_1)<\infty$. Integrating against $\mu$ gives~\eqref{eq:change-reference}, and exchanging $p_1$ and $p_2$ gives the equivalence.
\end{proof}

\begin{lemma}[Soft Bellman operators]\label{lem:soft-bellman}
Let $r$ be bounded and $\pi\in\Pi$. Then $T^\pi$ and $T^\star$ map $L^\infty_{\cS}$ into itself and are monotone $\gamma$-contractions: for $T\in\{T^\pi,T^\star\}$ and $V,V'\in L^\infty_{\cS}$, $V\le V'$ implies $TV\le TV'$, and $\|TV-TV'\|_\infty\le\gamma\|V-V'\|_\infty$. Hence $T^\pi$ and $T^\star$ have unique fixed points $V^\pi$ and $V^\star$ in $L^\infty_{\cS}$. Moreover, the Bellman resolvent~\eqref{eq:resolvent} is a positive operator on $L^\infty_{\cS}$ with ${(I-\gamma P^\pi)^{-1}\mathbf 1=(1-\gamma)^{-1}\mathbf 1}$ and $(I-\gamma P^\pi)^{-1}f\ge f$ for $f\ge0$, and $V^\pi=(I-\gamma P^\pi)^{-1}c^\pi$.
\end{lemma}

\begin{proof}
Let $V\in L^\infty_{\cS}$. Since $P(\dd s'\mid s,a)$ has a density with respect to $\mathfrak m$, $(PV)(s,a)$ does not depend on the representative of $V$. It is jointly measurable in $(s,a)$ and satisfies $|PV|\le\|V\|_\infty$. By~\eqref{eq:onestep-reward}, for $\mathfrak m$-a.e.\ $s$,
\[
|c^\pi(s)|\le\|r\|_\infty+\tau|\log Z_\beta|+\tau\esssup_{s'}\KL(\pi_{s'}\,\|\,\rho_\beta),
\]
so $T^\pi V=c^\pi+\gamma P^\pi V$ is essentially bounded. Since $\tilde r(s,a)=r(s,a)-\frac\beta2|a|^2$,
\[
(T^\star V)(s)=\tau\log Z_\beta+\tau\log\int e^{(r(s,a)+\gamma(PV)(s,a))/\tau}\rho_\beta(a)\,\dd a,
\]
and the exponent is bounded in absolute value by $(\|r\|_\infty+\gamma\|V\|_\infty)/\tau$. Since $\int\rho_\beta\,\dd a=1$, this bounds $|(T^\star V)(s)-\tau\log Z_\beta|$ by $\|r\|_\infty+\gamma\|V\|_\infty$. Since $\pi$ and $PV$ are jointly measurable, the Fubini--Tonelli theorem shows that $T^\pi V$ and $T^\star V$ are measurable.

If $V\le V'$, then $PV\le PV'$, so $Q_V\le Q_{V'}$, which gives $T^\pi V\le T^\pi V'$ and $T^\star V\le T^\star V'$. For a constant $c$, ${P(V+c\mathbf 1)=PV+c}$, so $Q_{V+c\mathbf 1}=Q_V+\gamma c$. Inserting this into the definitions, and using $\int\pi_s\,\dd a=1$ for $T^\pi$, gives $T(V+c\mathbf 1)=TV+\gamma c\mathbf 1$. Since $V\le V'+\|V-V'\|_\infty\mathbf 1$, monotonicity and the translation identity give $TV\le TV'+\gamma\|V-V'\|_\infty\mathbf 1$. Exchanging $V$ and $V'$ gives the contraction bound. Since $L^\infty_{\cS}$ is complete and $\gamma<1$, the Banach fixed-point theorem gives the unique fixed points.

Since $P^\pi$ is positive and $P^\pi\mathbf 1=\mathbf 1$, we have $\|P^\pi f\|_\infty\le\|f\|_\infty$. Hence the series~\eqref{eq:resolvent} converges in operator norm on $L^\infty_{\cS}$. Every term is a positive operator, the terms applied to $\mathbf 1$ sum to $(1-\gamma)^{-1}\mathbf 1$, and the term $n=0$ gives $(I-\gamma P^\pi)^{-1}f\ge f$ for $f\ge0$. Applying the resolvent to $V^\pi=c^\pi+\gamma P^\pi V^\pi$ gives $V^\pi=(I-\gamma P^\pi)^{-1}c^\pi$.
\end{proof}

\subsection{Entropy Calculus on Wasserstein Space}\label{app:entropy-calculus}

Write $\mathcal H_\beta(\mu):=\KL(\mu\,\|\,\rho_\beta)$ and $\mathcal I_\beta(\mu):=\mathcal I(\mu\,\|\,\rho_\beta)$. The following facts, used for individual action laws, concern the entropy functional relative to $\rho_\beta$.

\begin{lemma}[Entropy subgradient and chain rule]\label{lem:entropy-calculus}
For laws in $\mathcal P_2(\R^d)$ the following statements hold.
\begin{enumerate}
\itemlabel{itm:entropy-calculus-i}{i} If $\mathcal H_\beta(\mu),\mathcal H_\beta(\nu)<\infty$ and $\mathcal I_\beta(\mu)<\infty$, then
\begin{equation}\label{eq:hwi-two-laws}
\mathcal H_\beta(\mu)-\mathcal H_\beta(\nu)
\le \sqrt{\mathcal I_\beta(\mu)}\,W_2(\mu,\nu)
-\frac\beta{2\tau}W_2^2(\mu,\nu).
\end{equation}
\itemlabel{itm:entropy-calculus-ii}{ii} Let $(\mu_t)_{t\in J}$ be narrowly continuous on a compact interval $J$, with finite entropy at every time, and suppose that $\partial_t\mu_t+\diver(\mu_tw_t)=0$, $\sup_{t\in J}\mathcal I_\beta(\mu_t)<\infty$ and $\int_J\int |w_t|^2\,\dd\mu_t\,\dd t<\infty$. Then $t\mapsto\mu_t$ is $W_2$-absolutely continuous with square-integrable metric speed, $\mathcal H_\beta(\mu_t)$ is absolutely continuous, and, for almost every $t\in J$,
\begin{equation}\label{eq:weak-entropy-chain}
\frac{\dd}{\dd t}\mathcal H_\beta(\mu_t)
=\int\nabla\log(\mu_t/\rho_\beta)\cdot w_t\,\dd\mu_t.
\end{equation}
\itemlabel{itm:entropy-calculus-iii}{iii} When $\mathcal I_\beta(\mu)<\infty$, the relative score belongs to the tangent space $\operatorname{Tan}_\mu\mathcal P_2(\R^d)$, the $L^2(\mu)$-closure of $\{\nabla\psi:\psi\in C_c^\infty(\R^d)\}$.
\end{enumerate}
\end{lemma}

\begin{proof}
The relative entropy is a proper lower semicontinuous, $(\beta/\tau)$-geodesically convex functional. Its proximal minimization problems attain their minima by tightness and lower semicontinuity, and geodesic convexity gives the regularity required by the chain rule cited below. Its minimal Wasserstein subgradient is the weak relative score, and its squared slope is $\mathcal I_\beta$; see~\cite[Theorem~10.4.13 and equations~(10.4.67)--(10.4.69)]{ambrosio2005gradientflows}, with internal energy $x\log x$, potential $\beta|a|^2/(2\tau)$ and zero interaction. These choices satisfy the hypotheses of that theorem, and finite relative Fisher information gives the required local weak derivative. The subgradient inequality gives part~\itemref{itm:entropy-calculus-i}, also stated as the two-measure HWI inequality in~\cite[Proposition~7.18]{ambrosio2013guide}. The minimal subgradient lies in the tangent space, giving part~\itemref{itm:entropy-calculus-iii}.

For part~\itemref{itm:entropy-calculus-ii}, the converse direction of~\cite[Theorem~8.3.1]{ambrosio2005gradientflows} applies to the continuity equation and shows that the metric speed $\lim_{h\to0}W_2(\mu_{t+h},\mu_t)/|h|$ is at most $\|w_t\|_{L^2(\mu_t)}$. Thus the product of the entropy slope and the metric speed is integrable on $J$, so the convex chain rule~\cite[Proposition~10.3.18]{ambrosio2005gradientflows} applies. Its velocity is the minimal tangent velocity, and replacing it by $w_t$ leaves the pairing unchanged, because their difference has zero weighted divergence and the entropy subgradient is tangent~\cite[Lemma~8.4.2]{ambrosio2005gradientflows}. This proves~\eqref{eq:weak-entropy-chain}. The values at the endpoints agree with the absolutely continuous representative: applying part~\itemref{itm:entropy-calculus-i} in both directions bounds their entropy difference by a common constant times $W_2$.
\end{proof}

\section{Wasserstein-Gradient Representation}\label{app:metric}

This appendix makes the gradient interpretation in Section~\ref{sec:wpg} rigorous for policies whose scores are defined through weak derivatives. Proposition~\ref{prop:first-variation} differentiates the policy value and the performance objective along smooth action transports. Proposition~\ref{prop:weighted-gradient} then identifies $v^\pi$ as the minimum-energy velocity representing this first variation in the occupancy-weighted Wasserstein metric. Corollary~\ref{cor:gradient-along-flow} applies that representation to the WPG solution and proves the performance identity~\eqref{eq:performance-energy}.

\subsection{First Variation along Action Transports}

\begin{proposition}[First variation of policy performance]\label{prop:first-variation}
Let Assumptions~\ref{ass:base} and~\ref{ass:smooth} hold, and let $\pi\in\Pi$. Let $u(s,a)$ be jointly measurable, $C^2$ in the action, with common compact action support and uniformly bounded derivatives through order two. For $|\varepsilon|\|D_au\|_\infty<1$, define $F_{\varepsilon,s}(a)=a+\varepsilon u(s,a)$ and $\pi_{\varepsilon,s}=(F_{\varepsilon,s})_\#\pi_s$. Then $\pi_\varepsilon\in\Pi$, and the value derivative $\delta V:=\frac{\dd}{\dd\varepsilon}V^{\pi_\varepsilon}\big|_{\varepsilon=0}$ exists in $L^\infty_{\cS}$ and satisfies
\begin{equation}\label{eq:transport-value-derivative}
(I-\gamma P^\pi)\delta V=g_u,
\qquad g_u(s)=\int(\nabla_aQ^\pi\cdot u+\tau\diver_a u)\pi_s\,\dd a.
\end{equation}
If $\esssup_s\mathcal I(\pi_s\,\|\,p_s^\pi)<\infty$, then $g_u=\int v^\pi\cdot u\,\pi\,\dd a$, and~\eqref{eq:transport-derivative} holds, with weak scores in $v^\pi$.
\end{proposition}

\begin{proof}
The proof has three steps. We first use the transport change of variables to verify admissibility of $\pi_\varepsilon$. We then expand the one-step Bellman improvement to first order and apply the Bellman resolvent to obtain~\eqref{eq:transport-value-derivative}. Finally, weak integration by parts converts this derivative into the velocity pairing~\eqref{eq:transport-derivative}.

The inverse of $F_{\varepsilon,s}$ solves $x=y-\varepsilon u(s,x)$ by contraction, so the maps are global diffeomorphisms with positive Jacobians. Their inverses are jointly measurable by Picard iteration. The change-of-variables formula gives
\[
\mathcal H_\beta(\pi_{\varepsilon,s})-\mathcal H_\beta(\pi_s)
=\int\left[-\log\det(I+\varepsilon D_au_s)
+\frac\beta{2\tau}(|F_{\varepsilon,s}(a)|^2-|a|^2)\right]\pi_s(a)\,\dd a.
\]
The integrand is bounded uniformly in the state on compact parameter intervals, so the transported policies are admissible. With the value $V^\pi$ held fixed, the one-step Bellman improvement is exactly
\[
T^{\pi_\varepsilon}V^\pi-V^\pi
=\int\left[Q^\pi(s,F_{\varepsilon,s}(a))-Q^\pi(s,a)
+\tau\log\det(I+\varepsilon D_au_s(a))\right]\pi_s(a)\,\dd a.
\]
The action gradient of $Q^\pi$ is Lipschitz by Lemma~\ref{lem:drift}. Uniform Taylor expansion on the common compact support therefore gives $\varepsilon g_u+O(\varepsilon^2)$ in $L^\infty_{\cS}$. Also,
$\|P^{\pi_\varepsilon}-P^\pi\|_{\mathrm{op}}
\le G_p|\varepsilon|\|u\|_\infty$.
The first equality in~\eqref{eq:finite-difference}, applied with $\pi^+=\pi_\varepsilon$, and continuity of the resolvent give~\eqref{eq:transport-value-derivative}.

Finite relative Fisher information and the moment bound imply $\nabla_a\pi_s=\pi_s\nabla_a\log\pi_s\in L^1_{\mathrm{loc}}$. Weak integration by parts against the compactly supported $u$ gives $g_u=\int v^\pi\cdot u\,\pi\,\dd a$. Integrating the norm-convergent resolvent series against $\rho_0$ yields ${\delta J_{\rho_0}(\pi)=(1-\gamma)^{-1}\int g_u\,\dd d^\pi_{\rho_0}}$, which is~\eqref{eq:transport-derivative}. The associated density direction is $-\diver_a(\pi u)$ in distributions, so no classical density derivative is needed.
\end{proof}

\subsection{Minimum-Energy Velocity and the Identity along the Flow}

Recall $\mu_\pi$ from Section~\ref{sec:wpg}, and let $\mathsf H_\pi=L^2(\mu_\pi;\R^d)$. For a density direction $\xi$ representable by a velocity in $\mathsf H_\pi$, define its squared norm by
\begin{equation}\label{eq:weighted-metric}
\|\xi\|_{\pi,d^\pi_{\rho_0}}^2
:=\inf\left\{\|u\|_{\mathsf H_\pi}^2:
-\diver_a(\pi_su_s)=\xi_s\ \text{for }d^\pi_{\rho_0}\text{-a.e. }s\right\}.
\end{equation}
The constraint is distributional in the action. A velocity is tangent if its component at a fixed state belongs to $\operatorname{Tan}_{\pi_s}\mathcal P_2(\R^d)$ for almost every state with respect to the occupancy measure.

\begin{proposition}[Wasserstein ascent velocity]\label{prop:weighted-gradient}
Under the hypotheses of Proposition~\ref{prop:first-variation}, suppose $\esssup_s\mathcal I(\pi_s\,\|\,p_s^\pi)<\infty$. Then $v^\pi=-\tau\nabla_a\log(\pi/p^\pi)$ belongs to $\mathsf H_\pi$ and is tangent statewise. For the smooth transports of Proposition~\ref{prop:first-variation}, it represents the first variation by ${\delta J_{\rho_0}(\pi)[-\diver_a(\pi u)]=\langle v^\pi,u\rangle_{\mathsf H_\pi}}$. This pairing extends continuously to $\mathsf H_\pi$. The velocity $v^\pi$ has minimum energy among velocities with the same density direction, and
\begin{equation}\label{eq:metric-fisher}
\|-\diver_a(\pi v^\pi)\|_{\pi,d^\pi_{\rho_0}}^2
=\frac{\tau^2}{1-\gamma}\int\mathcal I(\pi_s\,\|\,p_s^\pi)\,d^\pi_{\rho_0}(\dd s).
\end{equation}
\end{proposition}

\begin{proof}
Write
$v^\pi=\nabla_a(r+\gamma PV^\pi)-\tau\nabla_a\log(\pi/\rho_\beta)$.
The field $\nabla_a(r+\gamma PV^\pi)$ is bounded, and the Fisher assumption implies $\esssup_s\|\nabla_a\log(\pi_s/\rho_\beta)\|_{L^2(\pi_s)}<\infty$.
Lemma~\ref{lem:entropy-calculus}\itemref{itm:entropy-calculus-iii} makes $\nabla_a\log(\pi_s/\rho_\beta)$ tangent. The bounded $C^{1,1}$ potential $r+\gamma PV^\pi$ can be cut off and mollified in the action, and the resulting gradients converge to $\nabla_a(r+\gamma PV^\pi)(s,\cdot)$ in $L^2(\pi_s)$. Thus $v_s^\pi$ is tangent as well. Proposition~\ref{prop:first-variation} gives the pairing, which extends by Cauchy--Schwarz. The minimum-norm characterization~\cite[Lemma~8.4.2]{ambrosio2005gradientflows} applies at each state, and integrating it against the occupancy measure proves~\eqref{eq:metric-fisher}. The minimizing field is explicit, so no measurable selection is involved.
\end{proof}

\begin{remark}[States with zero occupancy]\label{rem:zero-occupancy}
The square root of the energy in~\eqref{eq:weighted-metric} is a norm on the quotient that identifies velocities with the same weighted divergence and ignores states of zero occupancy. On full statewise density directions it is only a seminorm. WPG prescribes~\eqref{eq:velocity} also at those states. The convergence analysis uses this extension for $\mathfrak m$-almost every state.
\end{remark}

\begin{corollary}[Gradient representation along the flow]\label{cor:gradient-along-flow}
Under Assumptions~\ref{ass:base} and~\ref{ass:smooth}, the solution of Theorem~\ref{thm:wellposed} has the statewise minimum-energy velocity $v_t$ at every positive time. The performance is absolutely continuous on every $[0,T]$, and~its~derivative~satisfies
\begin{equation}\label{eq:performance-energy}
\frac{\dd}{\dd t}J_{\rho_0}(\pi_t)
=\frac{\tau^2}{1-\gamma}\int\mathcal I(\pi_t(\cdot\mid s)\,\|\,p_s^{\pi_t})\,d^{\pi_t}_{\rho_0}(\dd s)
=\|\partial_t\pi_t\|_{\pi_t,d^{\pi_t}_{\rho_0}}^2
\end{equation}
for almost every $t>0$, where $\partial_t\pi_t=-\diver_a(\pi_tv_t)$ is distributional.
\end{corollary}

\begin{proof}
Since $\nabla_a\log(p^{\pi_t}_s/\rho_\beta)=\tau^{-1}\nabla_a(r+\gamma PV_t)(s,\cdot)$ is bounded, the Fisher-information bound in Theorem~\ref{thm:wellposed}\itemref{itm:wellposed-ii} verifies Proposition~\ref{prop:weighted-gradient} at every positive time. By~\eqref{eq:monotone} and the initial-time convergence~\eqref{eq:initial-traces}, the statewise values are absolutely continuous on compact time intervals bounded away from zero, nondecreasing, bounded, and converge to $V_0$ as $t\downarrow0$. Integrating their derivative and then taking this limit gives
$V_t(s)-V_0(s)=\int_0^t\dot V_q(s)\,\dd q$.
Tonelli's theorem permits integration against $\rho_0$. Expanding the resolvent in~\eqref{eq:monotone} by the Neumann series~\eqref{eq:resolvent} gives
\[
\int\dot V_t\,\dd\rho_0
=\sum_{n\ge0}\gamma^n\int(P^{\pi_t})^n\bigl[\tau^2\mathcal I(\pi_t\,\|\,p^{\pi_t})\bigr]\,\dd\rho_0
=\frac{\tau^2}{1-\gamma}\int\mathcal I(\pi_t\,\|\,p^{\pi_t})\,\dd d^{\pi_t}_{\rho_0}.
\]
The time integral is finite because the total performance improvement is bounded, and~\eqref{eq:metric-fisher} gives the last equality in~\eqref{eq:performance-energy}.
\end{proof}

\section{Well-Posedness and Differentiation of Policy Evaluation}\label{app:wellposed}\label{app:chain}

As discussed in Section~\ref{sec:ct}, the entropy dependence of Bellman evaluation makes it challenging to establish Wasserstein stability of the policy-dependent drift. We prove Theorem~\ref{thm:wellposed} in three steps. First, Lemma~\ref{lem:prescribed} establishes regularity estimates for the diffusion associated with a given value-function path, while Lemma~\ref{lem:bellman-stability} uses the entropy subgradient inequality~\eqref{eq:hwi-two-laws} to control differences in policy evaluation. Second, we combine these estimates with synchronous coupling and the positive-time Fisher-information bound~\eqref{eq:prescribed-moments} to obtain the value-function stability estimate~\eqref{eq:value-feedback-estimate}. This allows us to construct a fixed point on value-function paths, initially clipping the value input to ensure uniform drift bounds. Finally, the initial value bounds~\eqref{eq:initial-unclipping} show that clipping is inactive near time zero. The Bellman differentiation identity~\eqref{eq:weak-bellman-chain}, established in Lemma~\ref{lem:value-derivative}, then yields monotone value improvement, which keeps the values within the clipping range for all time. The fixed point therefore solves the original WPG equation globally. Appendix~\ref{app:fixed-point} completes the proof by establishing uniqueness among weak solutions and continuous dependence on the initial policy.

Throughout this appendix, $\mathcal H_\beta$ and $\mathcal I_\beta$ denote the relative entropy and Fisher information with respect to $\rho_\beta$ defined in Appendix~\ref{app:entropy-calculus}. Constants in statewise estimates are uniform outside a state null set. Time derivatives of values are understood statewise. The proof does not require differentiability in the $L^\infty_{\cS}$ norm.

\subsection{Diffusions with a Given Drift}

The fixed-point construction begins by treating the function that enters the drift as given. For a bounded measurable function $\bar V:\cS\to\R$, define
\begin{equation}\label{eq:prescribed-drift}
\begin{aligned}
h_{\bar V}(s,a)&:=\nabla_ar(s,a)+\gamma\int\bar V(s')\nabla_ap(s'\mid s,a)\,\mathfrak m(\dd s'),\\
b_{\bar V}(s,a)&:=-\beta a+h_{\bar V}(s,a).
\end{aligned}
\end{equation}
The $L^1$ Taylor estimate in the proof of Lemma~\ref{lem:drift} shows that $a\mapsto p(\cdot\mid s,a)$ is $C^{1,1}$ as a map into $L^1(\mathfrak m)$. Consequently, $b_{\bar V}=\nabla_aQ_{\bar V}$. By Fubini's theorem, this drift is jointly measurable. The next lemma gives the estimates needed when this input varies with time. More generally, it treats any drift of the form $-\beta a+h$, where $h$ is bounded.

\begin{lemma}[Entropy, moment, and Fisher-information bounds]\label{lem:prescribed}
Fix $T<\infty$. Let $h(t,s,a)$ be jointly measurable, with $|h|\le H$, and suppose that $b_t(s,a):=-\beta a+h(t,s,a)$ is globally $L$-Lipschitz in $a$, uniformly in $(t,s)$. Let the initial policy satisfy $\esssup_s\mathcal H_\beta(\pi_{0,s})\le K$. For almost every state $s$, the SDE
\[
\dd X_t=b_t(s,X_t)\,\dd t+\sqrt{2\tau}\,\dd B_t,
\qquad X_0\sim\pi_0(\cdot\mid s),
\]
where $X_0$ is independent of $B$, has a unique strong solution on $[0,T]$. Its marginal laws ${\pi^h_{t,s}:=\pi^h_t(\cdot\mid s)}$ have jointly measurable densities and solve the corresponding linear Fokker--Planck equation. They satisfy
\begin{align}
\mathcal H_\beta(\pi^h_{t,s})
&\le\mathcal H_\beta(\pi_{0,s})+\frac{H^2t}{4\tau},\qquad 0\le t\le T,\label{eq:prescribed-entropy}\\
\mathcal J(\pi^h_{t,s})
&:=4\int|\nabla\sqrt{\pi^h_{t,s}}|^2\,\dd a
\le\frac d{2\tau}\left(\frac1t+L+\frac{L^2t}{3}\right),\qquad 0<t\le T,\label{eq:prescribed-fisher}\\
\sup_{0\le t\le T}\esssup_s\int|a|^2\pi^h_{t,s}(a)\,\dd a
&\le C_T,\qquad
\mathcal I_\beta(\pi^h_{t,s})\le C_T(1+t^{-1}),\qquad 0<t\le T,\label{eq:prescribed-moments}\\
\wdist(\pi^h_t,\pi^h_q)&\le C_T|t-q|^{1/2},\qquad 0\le q,t\le T.\label{eq:prescribed-time}
\end{align}
For almost every state, $\mathcal H_\beta(\pi^h_{t,s})\to\mathcal H_\beta(\pi_{0,s})$ as $t\downarrow0$. The score functions, interpreted through weak derivatives, can be chosen jointly measurable. The constants depend only on $T,H,L,K,d,\tau,\beta$.
\end{lemma}

\begin{proof}
\emph{Existence and the entropy bound.}
Fix a state for which the assumptions hold. Standard SDE theory gives strong existence and uniqueness because the drift is measurable in time, Lipschitz in the action, and has linear growth~\cite[Chapter~5]{karatzas1991brownian}. Picard iteration also gives measurable dependence on the state, the initial action, and the Brownian path. After integrating over the initial policy, the resulting family of marginal laws is jointly measurable.

To control the entropy, compare this process with the stationary Ornstein--Uhlenbeck process. Let $\mathbb Q$ denote the path law of that process with initial law $\rho_\beta$, and let $\mathbb P$ denote the path law of the process above. Replacing the initial law $\rho_\beta$ by $\pi_{0,s}$ contributes $\mathcal H_\beta(\pi_{0,s})$ to the relative entropy. Girsanov's theorem handles the change in drift from $-\beta a$ to $-\beta a+h$; boundedness of $h$ verifies Novikov's condition~\cite[Section~3.5]{karatzas1991brownian}. On paths restricted to $[0,t]$, this gives
\[
\KL(\mathbb P\,\|\,\mathbb Q)
=\mathcal H_\beta(\pi_{0,s})+
\frac1{4\tau}\mathbb E_{\mathbb P}\int_0^t|h(q,s,X_q)|^2\,\dd q.
\]
The time-$t$ marginal of $\mathbb Q$ is $\rho_\beta$. Projecting a path to its endpoint at time $t$ cannot increase relative entropy, so the preceding identity proves~\eqref{eq:prescribed-entropy}. In particular, the marginal under $\mathbb P$ has a density. A jointly measurable version can be obtained from limits of density averages over dyadic action cubes. It\^o's formula then gives the Fokker--Planck equation in the distributional sense.

\emph{Moments and continuity at the initial time.}
The entropy bound and Lemma~\ref{lem:entropy-integrability}\itemref{itm:entropy-integrability-ii} give the uniform second-moment bound in~\eqref{eq:prescribed-moments}. For $0\le q\le t\le T$, the SDE gives
\[
\mathbb E|X_t-X_q|^2
\le2(t-q)\int_q^t\mathbb E|b_u(s,X_u)|^2\,\dd u+4\tau d(t-q).
\]
Since $|b_u(s,a)|\le\beta|a|+H$, the moment bound makes the right-hand side at most $C_T(t-q)$. Since the joint law of $(X_q,X_t)$ is a coupling of the two marginals, this bound proves~\eqref{eq:prescribed-time}. In particular, the marginal at time $t$ converges to the initial law in $W_2$ as $t\downarrow0$. Lower semicontinuity of relative entropy, together with~\eqref{eq:prescribed-entropy}, now gives
\[
\mathcal H_\beta(\pi_{0,s})
\le\liminf_{t\downarrow0}\mathcal H_\beta(\pi^h_{t,s})
\le\limsup_{t\downarrow0}\mathcal H_\beta(\pi^h_{t,s})
\le\mathcal H_\beta(\pi_{0,s}).
\]

\emph{Fisher information for $t>0$.}
First suppose that the drift is $C^1$ in the action. We use Wang's integration-by-parts formula to bound the score of the marginal density. In~\cite[Theorem~3.1(1)]{wang2014shift}, take $m=0$, $Z=b$, $\sigma=\sqrt{2\tau}I_d$, and $\psi(q)=q/t$. Then
\[
\mathbb E[\partial_e f(X_t)]=\mathbb E[f(X_t)M_e],
\qquad
M_e:=\frac1{\sqrt{2\tau}}\int_0^t
\left(\frac et-\frac qt Db_q(s,X_q)e\right)\cdot\dd B_q
\]
for each unit coordinate vector $e$ and $f\in C_c^\infty(\R^d)$. Here the rank condition is automatic, linear growth prevents explosion, and $\|Db_q\|\le L$ verifies the derivative condition in that theorem. The theorem allows measurable dependence on time. Conditioning on $X_0$ extends the formula to the random initial action used here.

By~\cite[Theorem~2.4]{wang2014shift}, the directional score, interpreted in the weak sense, is $-\mathbb E[M_e\mid X_t]$. Conditional Jensen's inequality and It\^o's isometry therefore give
\[
\int|\partial_e\log\pi^h_{t,s}|^2\pi^h_{t,s}\,\dd a
\le\frac1{2\tau}\int_0^t\left(\frac1t+\frac{Lq}{t}\right)^2\dd q
=\frac1{2\tau}\left(\frac1t+L+\frac{L^2t}{3}\right).
\]
Summing over coordinates gives~\eqref{eq:prescribed-fisher}.

For a drift that is only Lipschitz, smooth $h$ by convolution in the action variable with a centered mollifier. The smoothed drifts retain the bounds $|h|\le H$ and $\Lip_a(b)\le L$, because centered convolution leaves the linear term $-\beta a$ unchanged. Moreover, $h$ is $(L+\beta)$-Lipschitz, so the smoothed drifts converge uniformly to $b$. Driving each approximating SDE and the limiting SDE with the same Brownian motion shows that their marginals converge in $W_2$. The variational formula
\[
\mathcal J(\mu)=\sup_{\zeta\in C_c^\infty(\R^d;\R^d)}
\int(-2\diver\zeta-|\zeta|^2)\,\dd\mu
\]
shows that $\mathcal J$ is lower semicontinuous under weak convergence, so~\eqref{eq:prescribed-fisher} passes to the limit. Finally,
\[
\mathcal I_\beta(\pi^h_{t,s})
\le2\mathcal J(\pi^h_{t,s})+2(\beta/\tau)^2\int|a|^2\pi^h_{t,s}(a)\,\dd a
\]
gives the relative Fisher-information bound in~\eqref{eq:prescribed-moments}. Finite Fisher information implies that the density has a weak derivative in $L^1_{\mathrm{loc}}$ with respect to the action. Mollifying the joint density and passing to the limit yields a jointly measurable version of this derivative. Dividing by the density where it is positive, and setting the result to zero elsewhere, gives a jointly measurable score.
\end{proof}

\subsection{Stability and Differentiation of Policy Evaluation}

The proof of Theorem~\ref{thm:wellposed} requires two properties of policy evaluation. First, if two policies are close in $W_2$ and both have bounded relative entropy and Fisher information, then their value functions are close in the sup norm. Second, along a policy path with sufficient regularity, the Bellman equation can be differentiated. The next two lemmas establish these facts.

\begin{lemma}[Policy-evaluation stability]\label{lem:bellman-stability}
Let Assumptions~\ref{ass:base} and~\ref{ass:smooth} hold. If $\pi,\widehat\pi\in\Pi$ satisfy, for some finite constants $K,F\ge0$,
\[
\esssup_s\max\{\mathcal H_\beta(\pi_s),\mathcal H_\beta(\widehat\pi_s)\}\le K,
\qquad
\esssup_s\max\{\mathcal I_\beta(\pi_s),\mathcal I_\beta(\widehat\pi_s)\}\le F^2,
\]
then, with $B_K:=(\|r\|_\infty+\tau|\log Z_\beta|+\tau K)/(1-\gamma)$,
\begin{equation}\label{eq:weak-evaluation-stability}
\|V^\pi-V^{\widehat\pi}\|_\infty
\le\frac{G_r+\tau F+\gamma G_pB_K}{1-\gamma}\,\wdist(\pi,\widehat\pi).
\end{equation}
\end{lemma}

\begin{proof}
We first control the entropy term. Applying~\eqref{eq:hwi-two-laws} in both directions and dropping the nonpositive quadratic term~gives
\[
|\mathcal H_\beta(\pi_s)-\mathcal H_\beta(\widehat\pi_s)|
\le F W_2(\pi_s,\widehat\pi_s).
\]
The reward expectation changes by at most $G_rW_2(\pi_s,\widehat\pi_s)$ because $r(s,\cdot)$ is $G_r$-Lipschitz. Inserting this bound and the entropy estimate into~\eqref{eq:onestep-reward} gives
\[
\|c^\pi-c^{\widehat\pi}\|_\infty\le(G_r+\tau F)\wdist(\pi,\widehat\pi).
\]
We next compare the transition operators. For bounded measurable $f$, differentiation under the transition integral gives $\|\nabla_a(Pf)(s,\cdot)\|_\infty\le G_p\|f\|_\infty$. Comparing the expectations of $(Pf)(s,\cdot)$ under a coupling of $\pi_s$ and $\widehat\pi_s$, and then minimizing over couplings, yields
\[
\|P^\pi-P^{\widehat\pi}\|_{\mathrm{op}}\le G_p\wdist(\pi,\widehat\pi).
\]
The entropy bound implies $\|V^\pi\|_\infty,\|V^{\widehat\pi}\|_\infty\le B_K$. Subtracting the two Bellman equations gives
\[
(I-\gamma P^\pi)(V^\pi-V^{\widehat\pi})
=c^\pi-c^{\widehat\pi}+\gamma(P^\pi-P^{\widehat\pi})V^{\widehat\pi}.
\]
Applying $(I-\gamma P^\pi)^{-1}$ and using its operator norm $(1-\gamma)^{-1}$ proves~\eqref{eq:weak-evaluation-stability}.
\end{proof}

The next lemma gives the derivative formula used later to prove value improvement. It applies to any policy path that satisfies a continuity equation and uniform bounds on relative entropy, Fisher information, and velocity energy.

\begin{lemma}[Differentiation of policy evaluation along a policy path]\label{lem:value-derivative}
Let Assumptions~\ref{ass:base} and~\ref{ass:smooth} hold. Let $\pi_t$ be jointly measurable and, for almost every state, continuous in time in the weak topology of probability measures on an open interval $(t_1,t_2)$. Suppose that
$\partial_t\pi_t+\diver_a(\pi_tw_t)=0$
with a jointly measurable velocity $w$. Assume that, on every compact interval $J\subset(t_1,t_2)$,
\begin{equation}\label{eq:chain-hypotheses}
\sup_{t\in J}\esssup_s
\left[\mathcal H_\beta(\pi_{t,s})+\mathcal I_\beta(\pi_{t,s})
+\int|w_t(s,a)|^2\pi_t(a\mid s)\,\dd a\right]<\infty.
\end{equation}
\looseness=-1 Then $c_t:=c^{\pi_t}$ and $V_t:=V^{\pi_t}$ are locally Lipschitz as maps into $L^\infty_{\cS}$, and $P_t:=P^{\pi_t}$ is locally Lipschitz in operator norm. Moreover, one can choose a version of $V_t(s)$ that is locally Lipschitz in time for every state outside a single null set. For almost every $t\in(t_1,t_2)$, the statewise derivatives satisfy
\begin{align}
\dot c_t(s)&=\int\bigl(\nabla_ar-\tau\nabla_a\log(\pi_t/\rho_\beta)\bigr)\cdot w_t\,\pi_t\,\dd a,\label{eq:weak-c-derivative}\\
\bigl[(I-\gamma P_t)\dot V_t\bigr](s)
&=\int\bigl(\nabla_ar+\gamma\nabla_a(PV_t)-\tau\nabla_a\log(\pi_t/\rho_\beta)\bigr)\cdot w_t\,\pi_t\,\dd a.
\label{eq:weak-bellman-chain}
\end{align}
These derivatives are jointly measurable and essentially bounded over states and compact time intervals.
\end{lemma}

\begin{proof}
Fix a compact interval $J\subset(t_1,t_2)$. Throughout the proof, all bounds are uniform over states, and their constants depend only on the bound in~\eqref{eq:chain-hypotheses}.

\emph{The one-step reward.}
We first choose a single state null set outside which the required bounds hold throughout $J$. For the entropy and Fisher-information bounds, check a countable dense set of times and use weak continuity and lower semicontinuity. Fubini's theorem gives the velocity-energy bound for almost every time outside the same null set. Lemma~\ref{lem:entropy-calculus}\itemref{itm:entropy-calculus-ii} then applies statewise and yields both $W_2$-absolute continuity and the entropy derivative~\eqref{eq:weak-entropy-chain}. The uniform velocity bound also gives
\[
W_2(\pi_{t,s},\pi_{q,s})
\le\int_q^t\|w_\sigma(s,\cdot)\|_{L^2(\pi_{\sigma,s})}\,\dd\sigma
\le C_J|t-q|,\qquad q\le t,\quad q,t\in J.
\]
Testing the continuity equation against $r(s,\cdot)$ gives the derivative of the expected reward. A cutoff and smooth approximation justify this test because both $r$ and $\nabla_ar$ are bounded and $w$ has finite $L^2(\pi)$ norm. Combining the reward derivative with the entropy derivative and~\eqref{eq:onestep-reward} proves~\eqref{eq:weak-c-derivative}. Its right-hand side is uniformly bounded by the Cauchy--Schwarz inequality, so $c_t$ is Lipschitz on $J$ in the $L^\infty_{\cS}$ norm.

\emph{The transition operator.}
Write $P_t(s,s'):=\int p(s'\mid s,a)\pi_t(a\mid s)\,\dd a$ for the density of the transition kernel under $\pi_t$, and define
\begin{equation}\label{eq:transition-derivative}
\begin{aligned}
\dot P_t(s,s')&:=\int\nabla_ap(s'\mid s,a)\cdot w_t(s,a)\pi_t(a\mid s)\,\dd a,\\
\|\dot P_t(s,\cdot)\|_{L^1(\mathfrak m)}
&\le G_p\int|w_t(s,a)|\pi_t(a\mid s)\,\dd a\le C_J.
\end{aligned}
\end{equation}
The integrals are well defined by Fubini's theorem and the $L^1$ bound on $\nabla_ap$ in Assumption~\ref{ass:smooth}\itemref{itm:smooth-b}.

For bounded measurable $f$, the function $(Pf)(s,\cdot)$ is bounded and $C^{1,1}$ with bounded gradient, so a cutoff and smooth approximation make it an admissible test function. To obtain an identity for the entire transition kernel, first use indicators from a countable algebra that generates the state space. The integrated identities hold first at rational endpoints and then at all endpoints by continuity. Since this countable family determines finite measures, we obtain
\begin{equation}\label{eq:kernel-ac}
P_t(s,\cdot)-P_q(s,\cdot)=\int_q^t\dot P_\sigma(s,\cdot)\,\dd\sigma
\quad\text{in }L^1(\mathfrak m),\qquad q,t\in J.
\end{equation}
Because $L^1(\mathfrak m)$ is separable and the kernel in~\eqref{eq:transition-derivative} has integrable norm, the right-hand side is a well-defined $L^1(\mathfrak m)$-valued integral. It follows that $P_t$ is Lipschitz in operator norm. Equation~\eqref{eq:kernel-ac} also shows that, for almost every $t$ and almost every $s$, the difference quotients converge in $L^1(\mathfrak m)$. For every bounded $f$, the derivative is
\[
\dot P_t f(s)=\int f(s')\dot P_t(s,s')\,\mathfrak m(\dd s')
=\int\nabla_a(Pf)(s,a)\cdot w_t(s,a)\pi_t(a\mid s)\,\dd a.
\]
Because the argument began with a countable family that determines finite measures, the state null set can be chosen independently of $f$.

\emph{The value function.}
Subtracting the Bellman equations at $q,t\in J$ gives
\[
(I-\gamma P_t)(V_t-V_q)=c_t-c_q+\gamma(P_t-P_q)V_q.
\]
The uniform entropy bound gives a common bound on $\|V_q\|_\infty$. Together with the preceding estimates and the resolvent bound in Lemma~\ref{lem:soft-bellman}, this shows that $V_t$ is Lipschitz on $J$ in the $L^\infty_{\cS}$ norm. Choose compatible statewise versions at rational times and extend them using this common Lipschitz bound. This produces a version of $V_t(s)$ that is scalar Lipschitz in time outside one state null set. Continuity of $c_t$, equation~\eqref{eq:kernel-ac}, and norm continuity of $V_t$ extend the Bellman equation from rational times to all times. The Neumann series~\eqref{eq:resolvent} then confirms that this version is the policy value function.

The scalar derivatives $\dot V_t(s)$ are jointly measurable, uniformly bounded on $J$, and satisfy the fundamental theorem of calculus statewise. By Fubini's theorem, for almost every $t$ the scalar difference quotients converge for almost every state. At such a time, the $L^1$ derivative of the transition kernel also exists, and we may write
\[
\begin{aligned}
\frac{P_{t+h}V_{t+h}-P_tV_t}{h}
&=\frac{P_{t+h}-P_t}{h}V_{t+h}
 +P_t\frac{V_{t+h}-V_t}{h}.
\end{aligned}
\]
The first term converges statewise to $\dot P_tV_t$ by the $L^1$ derivative of the kernel and norm continuity of $V_t$. The second converges to $P_t\dot V_t$ by dominated convergence, because the value difference quotients are uniformly bounded and $P_t(s,\cdot)$ has a density with respect to $\mathfrak m$. Differentiating $V_t=c_t+\gamma P_tV_t$ therefore gives
\[
(I-\gamma P_t)\dot V_t=\dot c_t+\gamma\dot P_tV_t.
\]
Substituting~\eqref{eq:weak-c-derivative} and the formula for $\dot P_tV_t$ proves~\eqref{eq:weak-bellman-chain}. Applying the argument on a countable sequence of compact intervals that covers $(t_1,t_2)$ gives the stated version of the value function on the entire interval.
\end{proof}

\subsection{Fixed-Point Construction of the WPG Flow}\label{app:fixed-point}

We now combine Lemmas~\ref{lem:prescribed}--\ref{lem:value-derivative} to construct the nonlinear WPG flow and prove all parts of Theorem~\ref{thm:wellposed}.

\begin{proof}[Proof of Theorem~\ref{thm:wellposed}]
\emph{Step 1: define the fixed-point map.}
Fix $T<\infty$ and $M>\max\{|V_-|,|V_+|\}$. Define the scalar clipping map
$\mathcal C_M(x):=\max\{-M,\min\{x,M\}\}$
and the space
\[
\mathcal U_T:=L^\infty((0,T)\times\cS,\dd t\otimes\mathfrak m).
\]
Elements of $\mathcal U_T$ are equivalence classes, so we work with jointly measurable representatives and do not require strong measurability as $L^\infty_{\cS}$-valued paths. Given an input $\bar V\in\mathcal U_T$, define the drift
$b_t^{\bar V}:=b_{\mathcal C_M(\bar V_t)}$
through~\eqref{eq:prescribed-drift}. Clipping and Assumption~\ref{ass:smooth} give
\[
\begin{aligned}
|h_{\mathcal C_M(\bar V_t)}|&\le H_M:=G_r+\gamma G_pM,\\
\Lip_a(b_t^{\bar V})&\le L_M:=\beta+L_r+\gamma L_pM,\\
\|b_t^{\bar V}-b_t^{\bar V'}\|_\infty
&\le\gamma G_p\|\bar V_t-\bar V'_t\|_\infty
\quad\text{for a.e. }t.
\end{aligned}
\]
Lemma~\ref{lem:prescribed}, with $H=H_M$, $L=L_M$ and initial policy $\pi_0$, then gives a policy path $\pi^{\bar V}$. Define
\[
\mathcal F_M(\bar V)_t:=V^{\pi^{\bar V}_t}.
\]
At a fixed point, the input $\bar V_t$ equals the value function $V^{\pi^{\bar V}_t}$ of the policy it generates. The corresponding drift therefore uses the clipped value function of its own policy.

This construction does not depend on the chosen representative of $\bar V$. If two representatives agree outside a product null set, Fubini's theorem and domination of the transition kernel by $\mathfrak m$ show that the resulting drifts agree for almost every time and almost every state. Changing a drift on a null set of times does not change the SDE laws. Hence $\mathcal F_M$ is well defined on $\mathcal U_T$.

By Lemma~\ref{lem:prescribed}, every policy path generated by $\bar V$ has entropy at most
${K_T:=\RelEnt_0+H_M^2T/(4\tau)}$.
As in the proof of Lemma~\ref{lem:bellman-stability}, the corresponding value function is therefore bounded by $B_{K_T}$. The Bellman Neumann series~\eqref{eq:resolvent} also gives joint measurability, so $\mathcal F_M$ maps $\mathcal U_T$ into itself.

To compare two inputs $\bar V$ and $\bar V'$, drive the two SDEs with the same initial action and the same Brownian motion. The drift bounds above and Gr\"onwall's inequality give
\[
\wdist(\pi^{\bar V}_t,\pi^{\bar V'}_t)
\le\gamma G_p e^{L_MT}\int_0^t\|\bar V_q-\bar V'_q\|_\infty\,\dd q.
\]
Combining this estimate with~\eqref{eq:prescribed-moments} and Lemma~\ref{lem:bellman-stability} yields
\begin{equation}\label{eq:value-feedback-estimate}
\|\mathcal F_M(\bar V)_t-\mathcal F_M(\bar V')_t\|_\infty
\le C_{M,T}(1+t^{-1/2})\int_0^t\|\bar V_q-\bar V'_q\|_\infty\,\dd q.
\end{equation}
Although the coefficient in~\eqref{eq:value-feedback-estimate} behaves like $t^{-1/2}$ near zero, it multiplies an integral over $[0,t]$, which is of order~$t$. The right-hand side therefore remains bounded as $t\downarrow0$. To express this in a contraction estimate, equip $\mathcal U_T$ with~the~equivalent~norm
\[
\|\bar V\|_\omega:=\esssup_{0<t<T}e^{-\omega t}\|\bar V_t\|_\infty,
\qquad \omega>0.
\]
\looseness=-1 Since $\frac{(1+t^{-1/2})(1-e^{-\omega t})}{\omega}\le\omega^{-1}+\omega^{-1/2}$, as follows from $1-e^{-x}\le\min\{x,1\}\le\sqrt{x}$, equation~\eqref{eq:value-feedback-estimate} gives
\[
\|\mathcal F_M(\bar V)-\mathcal F_M(\bar V')\|_\omega
\le C_{M,T}(\omega^{-1}+\omega^{-1/2})\|\bar V-\bar V'\|_\omega.
\]
For sufficiently large $\omega$, $\mathcal F_M$ is a contraction on the complete space $\mathcal U_T$. Its unique fixed point defines a policy path $\pi$ that solves the WPG equation with the clipped value function in the drift.

\emph{Step 2: establish time regularity and the initial-time limit.}
For a policy path generated by an input $\bar V$, let $\xi_t:=\nabla_a\log(\pi^{\bar V}_t/\rho_\beta)$ denote its weak relative score. Its Fokker--Planck equation can be written as a continuity equation with velocity
\[
w_t=b_t^{\bar V}-\tau\nabla_a\log\pi^{\bar V}_t
=h_{\mathcal C_M(\bar V_t)}-\tau\xi_t,
\qquad
\int|w_t|^2\pi^{\bar V}_t\,\dd a
\le2H_M^2+2\tau^2\mathcal I_\beta(\pi^{\bar V}_{t,s}).
\]
Lemma~\ref{lem:prescribed} verifies the hypotheses of Lemma~\ref{lem:value-derivative} on every compact time interval bounded away from zero. Hence the policy value is locally Lipschitz in $L^\infty_{\cS}$ on $(0,T]$. At the fixed point, we choose this Lipschitz version and write $V_t:=V^{\pi_t}$, $c_t:=c^{\pi_t}$, and $P_t:=P^{\pi_t}$. This change affects the input only on a null set and therefore leaves the policy laws unchanged.

Lemma~\ref{lem:prescribed} gives convergence to the initial policy in $W_2$ and convergence of the relative entropy. Since $r$ is bounded and continuous in the action, the one-step reward satisfies $c_t(s)\to c_0(s)$ outside one state null set. The functions $c_t$ are uniformly bounded, and
\[
\|P_t-P_0\|_{\mathrm{op}}\le G_p\wdist(\pi_t,\pi_0)\longrightarrow0.
\]
For each fixed $n\ge0$, decompose
\[
P_t^nc_t-P_0^nc_0=(P_t^n-P_0^n)c_t+P_0^n(c_t-c_0).
\]
The first term tends to zero in norm, since
$\|P_t^n-P_0^n\|_{\mathrm{op}}\le n\|P_t-P_0\|_{\mathrm{op}}$.
The second term tends to zero statewise by dominated convergence. For $n=0$, this is exactly the convergence of $c_t$ established above. For $n\ge1$, it follows because the kernel $P_0^n(s,\cdot)$ is absolutely continuous with respect to $\mathfrak m$. The tails of the Neumann series are uniformly bounded by
$\sup_{q\le T}\|c_q\|_\infty\gamma^{N+1}/(1-\gamma)$.
Apply the limit first to a finite partial sum and then let the truncation level tend to infinity. This proves $V_t(s)\to V_0(s)$, as required in~\eqref{eq:initial-traces}.

\emph{Step 3: show that clipping is never active and extend the solution globally.}
We first show that clipping is inactive on a common initial interval. By~\eqref{eq:prescribed-entropy} and~\eqref{eq:onestep-reward},
\[
c_t\ge r_-+\tau\log Z_\beta-\tau\RelEnt_0-\frac{H_M^2t}{4}
=(1-\gamma)V_--\frac{H_M^2t}{4}.
\]
Applying the positive Bellman resolvent from Lemma~\ref{lem:soft-bellman} to the lower bound on $c_t$, and using Lemma~\ref{lem:bellman-identities}\itemref{itm:bellman-identities-ii}--\itemref{itm:bellman-identities-iii} for the upper bounds, gives
\begin{equation}\label{eq:initial-unclipping}
V_- -\frac{H_M^2t}{4(1-\gamma)}\le V_t\le V^\star\le V_+.
\end{equation}
Because $V_->-M$ and $V_+<M$, these bounds place $V_t$ strictly inside $[-M,M]$ for all sufficiently small $t$, uniformly over states. Thus clipping is initially inactive. This conclusion comes directly from the entropy estimate and does not require convergence of $V_t$ to $V_0$ in the sup norm.

On any interval starting at zero on which clipping is inactive, the velocity from Step~2 becomes
\[
w_t=\nabla_ar+\gamma\nabla_a(PV_t)-\tau\nabla_a\log(\pi_t/\rho_\beta).
\]
Applying Lemma~\ref{lem:value-derivative} on compact time subintervals bounded away from zero gives
\[
(I-\gamma P_t)\dot V_t=g_t,
\qquad g_t(s):=\int|w_t(s,a)|^2\pi_t(a\mid s)\,\dd a,
\qquad \dot V_t=(I-\gamma P_t)^{-1}g_t\ge g_t\ge0.
\]
Integrating this inequality and using the statewise limit at time zero gives $V_t\ge V_0\ge V_-$. Lemma~\ref{lem:bellman-identities}\itemref{itm:bellman-identities-ii}--\itemref{itm:bellman-identities-iii} gives the upper bound $V_t\le V^\star\le V_+$. Suppose, for contradiction, that the maximal interval on which clipping is inactive ended at some positive time before $T$. By continuity in the sup norm, the value at that endpoint would still lie in $[V_-,V_+]$, which remains a positive distance from the clipping thresholds $\pm M$. The same continuity would then extend the unclipped interval beyond that endpoint. Hence clipping is inactive on all of $[0,T]$, and the constructed path solves~\eqref{eq:WPGF}.

For $0<\varepsilon<T$, the inequality $g_t\le\dot V_t$ gives
\[
\int_\varepsilon^T\int|w_t(s,a)|^2\pi_t(a\mid s)\,\dd a\,\dd t
\le V_T(s)-V_\varepsilon(s).
\]
Letting $\varepsilon\downarrow0$ and using the initial-time limit yields
\[
\int_0^T\int|w_t(s,a)|^2\pi_t(a\mid s)\,\dd a\,\dd t
\le V_T(s)-V_0(s)\le V_+-V_-.
\]
This establishes the integrability of the velocity down to time zero, which does not follow from the pointwise Fisher-information bound $\mathcal I_\beta(\pi_{t,s})\le C_T(1+t^{-1})$ alone. Together with the $W_2$ convergence at the initial time, the continuity-equation characterization~\cite[Theorem~8.3.1]{ambrosio2005gradientflows} now gives statewise $W_2$-absolute continuity on $[0,T]$, with square-integrable metric speed and velocity $w_t=v_t$.

The same clipping level $M$ works for every finite horizon. Moreover, the policy laws up to time $T$ depend only on the input up to time $T$. Thus the restriction of a fixed point constructed on a longer interval is the unique fixed point on $[0,T]$. Constructing the solution successively on $[0,n]$, $n\in\mathbb N$, therefore produces a consistent global solution. The entropy estimate~\eqref{eq:prescribed-entropy} verifies condition~\eqref{eq:weak-class} on every finite time interval.

\emph{Step 4: prove uniqueness among weak solutions.}
Let $\widehat\pi$ be any weak solution in Definition~\ref{def:weak-solution} with initial policy $\pi_0$, and fix $T<\infty$. Its entropy bound gives uniform bounds on its values and second moments. Set $\widehat V_t:=V^{\widehat\pi_t}$. The drift $b_{\widehat V_t}$ determined by this value function is jointly measurable, globally Lipschitz in the action, and of linear growth, with constants uniform on $[0,T]$. In particular, for almost every state,
\[
\int_0^T\int\bigl(2\tau d+|b_{\widehat V_t}(s,a)|\bigr)
\widehat\pi_t(a\mid s)\,\dd a\,\dd t<\infty.
\]
This is the coefficient-integrability condition required by the superposition principle~\cite[Theorem~2.5]{trevisan2016wellposedness}, with diffusion matrix $2\tau I_d$. Continuity in the weak topology gives the initial condition. The superposition principle then represents $\widehat\pi_t(\cdot\mid s)$ as the marginal laws of a solution to the martingale problem for this drift. Because the drift is Lipschitz, the corresponding SDE is unique in law, so these marginals coincide with the laws constructed in Lemma~\ref{lem:prescribed}.

Choose $M$ large enough that both $\widehat V$ and the value path constructed above lie in $[-M,M]$ on $[0,T]$. Both value paths are then fixed points of $\mathcal F_M$ for the same initial policy. Uniqueness of the fixed point from Step~1 gives equality for almost every time and state, and therefore equality of the corresponding policy laws. Continuity in time in the weak topology extends this equality to every time outside one state null set. Applying the argument on a countable sequence of horizons proves global uniqueness.

\emph{Step 5: prove continuous dependence on the initial policy.}
Consider two initial policies with a common upper bound on their relative entropy with respect to $\rho_\beta$. Their value functions lie in a common bounded interval, so the drift, entropy, and Fisher-information estimates above hold with common constants. Let $M$ bound the absolute values of both value functions. Then $L_M$ from Step~1 is a common Lipschitz constant of the two drifts in the action variable. At each state, choose an optimal coupling of the initial laws and drive the two SDEs with the same Brownian motion. This comparison is statewise, so no measurable choice of the couplings across states is needed. The coupling estimate and Lemma~\ref{lem:bellman-stability} give
\[
\begin{aligned}
\wdist(\pi_t,\widehat\pi_t)
&\le e^{L_MT}\wdist(\pi_0,\widehat\pi_0)
 +C_T\int_0^t\|V^{\pi_q}-V^{\widehat\pi_q}\|_\infty\,\dd q,\\
\|V^{\pi_t}-V^{\widehat\pi_t}\|_\infty
&\le C_T(1+t^{-1/2})\wdist(\pi_t,\widehat\pi_t),\qquad t>0.
\end{aligned}
\]
Substituting the second inequality into the first gives a Gr\"onwall inequality with the integrable coefficient $C_T(1+q^{-1/2})$. Therefore
\[
\wdist(\pi_t,\widehat\pi_t)
\le e^{L_MT+C_T(t+2\sqrt t)}\wdist(\pi_0,\widehat\pi_0),\qquad 0\le t\le T.
\]
Absorbing the exponential factor into $C_T$ proves~\eqref{eq:flow-stability}, and the second comparison inequality then proves~\eqref{eq:value-stability-initial}. Equation~\eqref{eq:prescribed-time} gives the time-continuity estimate, equation~\eqref{eq:prescribed-moments} gives the Fisher-information bound for $t>0$, and Step~2 gives the stated regularity of the value function and the measurable velocity. Steps~1--4 establish existence and uniqueness and prove parts~\itemref{itm:wellposed-i} and~\itemref{itm:wellposed-ii}; Step~5 proves the continuous-dependence bounds in part~\itemref{itm:wellposed-iii}. This completes the proof.
\end{proof}

\section{The Euler Interpolation}\label{app:ula}

This appendix provides the interpolation facts used in the proof of Lemma~\ref{lem:ula}. Lemma~\ref{lem:ula-chain} establishes the density and moment bounds, the continuity equation, the entropy derivative, and the integration-by-parts identity needed to derive~\eqref{eq:kprime2}. Lemma~\ref{lem:ula-endpoint} proves continuity of the relative entropy at $t=0$, which permits integration of that differential inequality from the initial law.

Throughout, $\rho$ is a positive $C^1$ probability density on $\R^d$ such that $b:=\tau\nabla\log \rho$ is $L$-Lipschitz, $X_0\sim\mu$ with $\int|x|^2\,\mu(\dd x)<\infty$ and $\KL(\mu\,\|\,\rho)<\infty$, $B$ is a standard Brownian motion independent of $X_0$, and $C$ denotes a generic positive constant. Let $\mathbb B_R:=\{a\in\R^d:|a|<R\}$ be the open ball and $\phi_{\sigma^2}$ the density of $\mathcal N(0,\sigma^2I_d)$. Set $F_t(x):=x+tb(x)$, $X_t:=F_t(X_0)+\sqrt{2\tau}\,B_t$, $\mu_t:=\Law(X_t)$ and $\mathcal H(t):=\KL(\mu_t\,\|\,\rho)$.

\subsection{Entropy along the Interpolation}\label{app:ula-interpolation}

\begin{lemma}[Regularity, entropy identity and integration by parts]\label{lem:ula-chain}
\leavevmode
\begin{enumerate}
\itemlabel{itm:ula-chain-i}{i} For every $T>0$, $\sup_{t\le T}\int|y|^2\,\mu_t(\dd y)<\infty$. For every $t>0$, $\mu_t$ has the strictly positive smooth density
\begin{equation}\label{eq:mixture}
\mu_t(y)=\int\phi_{2\tau t}\big(y-F_t(x)\big)\,\mu(\dd x).
\end{equation}
For $0<\varepsilon<T$ there are constants $c_{\varepsilon,T},C_{\varepsilon,T}\in(0,\infty)$ such that, for $t\in[\varepsilon,T]$ and $y\in\R^d$,
\[
c_{\varepsilon,T}e^{-C_{\varepsilon,T}(1+|y|^2)}\le\mu_t(y)\le(4\uppi\tau\varepsilon)^{-d/2},\qquad\Big|\log\frac{\mu_t(y)}{\rho(y)}\Big|\le C_{\varepsilon,T}(1+|y|^2).
\]
\itemlabel{itm:ula-chain-ii}{ii} Let $u_t(y):=\E[b(X_0)\mid X_t=y]$ and $w_t:=u_t-\tau\nabla\log\mu_t$. On $(0,\infty)\times\R^d$,
\[
\partial_t\mu_t=-\diver(\mu_tw_t)=-\diver(\mu_tu_t)+\tau\Delta\mu_t,
\]
and, for every $t>0$,
\[
\int|u_t|^2\mu_t\,\dd y\le\E|b(X_0)|^2,\quad\int|\nabla\log\mu_t|^2\mu_t\,\dd y\le\frac{d}{2\tau t},\quad
\mathcal I(\mu_t\,\|\,\rho)\le\frac{d}{\tau t}+\frac{2}{\tau^2}\E|b(X_t)|^2<\infty,
\]
and, for $0<\varepsilon<T$,
\[
\int_\varepsilon^T\int\bigl(|w_t|^2+|\nabla\log(\mu_t/\rho)|^2+|u_t-b|^2+|y|^2\bigr)\mu_t\,\dd y\,\dd t<\infty.
\]
\itemlabel{itm:ula-chain-iii}{iii} $\mathcal H$ is locally absolutely continuous on $(0,\infty)$ and, for a.e.\ $t>0$,
\[
\mathcal H'(t)=\int\mu_tw_t\cdot\nabla\log(\mu_t/\rho)\,\dd y=-\tau \mathcal I(\mu_t\,\|\,\rho)+\int\mu_t(u_t-b)\cdot\nabla\log(\mu_t/\rho)\,\dd y.
\]
\itemlabel{itm:ula-chain-iv}{iv} For every $t>0$, $\int|b|^2\mu_t\,\dd y<\infty$, and the weak divergence of $b$ satisfies
\[
\int\mu_t\,\diver b\,\dd y=-\int b\cdot\nabla\mu_t\,\dd y.
\]
\end{enumerate}
\end{lemma}

\begin{proof}
Parts~\itemref{itm:ula-chain-i}--\itemref{itm:ula-chain-ii} establish the regularity and integrability needed for the later calculations. Part~\itemref{itm:ula-chain-iii} proves the entropy identity by localizing the continuity equation and removing the cutoff. Part~\itemref{itm:ula-chain-iv} then justifies the integration by parts used in the drift-moment estimate of Lemma~\ref{lem:ula}.

\emph{Part~\itemref{itm:ula-chain-i}.} Linear growth of $b$ gives $\sup_{t\le T}\E|X_t|^2<\infty$ and $\E|b(X_0)|^2<\infty$. For $t>0$, conditioning on $X_0$ gives~\eqref{eq:mixture}. The density is strictly positive, smooth in $y$, and $C^1$ in $t$ on compact time intervals bounded away from zero: derivatives of the Gaussian kernel are bounded on such intervals, and the time derivative adds only the integrable factor $|b(x)|$. Fix $0<\varepsilon<T$, choose $R_0$ with $\mu(\mathbb B_{R_0})>0$, and set $R_1:=R_0+T\sup_{|x|\le R_0}|b(x)|$, so that $|F_t(x)|\le R_1$ for $x\in\mathbb B_{R_0}$ and $t\le T$. For $t\in[\varepsilon,T]$, restricting the integral in~\eqref{eq:mixture} to $\mathbb B_{R_0}$ and using $|y-F_t(x)|^2\le2|y|^2+2R_1^2$ gives
\[
\mu_t(y)\ge\mu(\mathbb B_{R_0})(4\uppi\tau T)^{-d/2}e^{-(|y|^2+R_1^2)/(2\tau\varepsilon)}.
\]
The maximum of the Gaussian kernel gives $\mu_t(y)\le(4\uppi\tau\varepsilon)^{-d/2}$. Hence $|\log\mu_t(y)|\le C_{\varepsilon,T}(1+|y|^2)$. The logarithmic bound also uses $|\log \rho(y)|\le C(1+|y|^2)$, a consequence of the case $x=0$ of the segment bound
\[
|\log \rho(x+z)-\log \rho(x)|\le\tau^{-1}|b(x)|\,|z|+(2\tau)^{-1}L|z|^2,\qquad x,z\in\R^d,
\]
which follows by integrating $\nabla\log \rho=b/\tau$ along the segment from $x$ to $x+z$.
Thus all entropy integrals on compact time intervals bounded away from zero are finite. The constants may depend on $\mu$, but they only justify limits and do not enter the one-step error constant.

\emph{Part~\itemref{itm:ula-chain-ii}.} Conditioning and differentiation of the Gaussian kernel give the measurable versions
\[
u_t(y)=\frac{\int b(x)\,\phi_{2\tau t}(y-F_t(x))\,\mu(\dd x)}{\mu_t(y)},\qquad\nabla\log\mu_t(y)=-\frac{1}{2\tau t}\,\E\big[\sqrt{2\tau}B_t\,\big|\,X_t=y\big].
\]
Differentiating $k_t(y;x):=\phi_{2\tau t}(y-F_t(x))$ and using the heat equation $\partial_t\phi_{2\tau t}=\tau\Delta\phi_{2\tau t}$ and $\partial_tF_t(x)=b(x)$ gives
\[
\partial_tk_t(y;x)=\tau\Delta_yk_t(y;x)-b(x)\cdot\nabla_yk_t(y;x).
\]
On compact time intervals bounded away from zero these derivatives are dominated by ${C(1+|y|^2+|x|^2+|b(x)|^2)\,k_t}$, which is integrable against $\mu(\dd x)$ locally uniformly in $y$, so integration against $\mu(\dd x)$ gives
\[
\partial_t\mu_t=\tau\Delta\mu_t-\diver\Bigl(\int b(x)k_t(\cdot;x)\,\mu(\dd x)\Bigr)=\tau\Delta\mu_t-\diver(\mu_tu_t).
\]
Conditional Jensen's inequality gives
\[
\int|u_t|^2\mu_t\le\E|b(X_0)|^2,\qquad\int|\nabla\log\mu_t|^2\mu_t\le\frac{\E|\sqrt{2\tau}B_t|^2}{4\tau^2t^2}=\frac{d}{2\tau t}.
\]
These bounds and the moment bound in part~\itemref{itm:ula-chain-i} give the space--time integrability. Finally, ${|\nabla\log(\mu_t/\rho)|^2\le2|\nabla\log\mu_t|^2+2\tau^{-2}|b|^2}$ gives the stated bound on $\mathcal I(\mu_t\,\|\,\rho)$.

\emph{Part~\itemref{itm:ula-chain-iii}.} To localize the continuity equation, let $\chi_R(y):=\chi(y/R)$ for a smooth radial cutoff $\chi$ that is nonincreasing in the radius, equals one on $\mathbb B_1$ and vanishes outside $\mathbb B_2$. Then $0\le\chi_R\uparrow1$ and $|\nabla\chi_R|\le C/R$. Differentiation on the compact support of each cutoff is legitimate by~\eqref{eq:mixture} and its positive local lower bound, and gives
\[
\frac{\dd}{\dd t}\int\chi_R\,\mu_t\log\frac{\mu_t}{\rho}\,\dd y=\int\chi_R\,\mu_tw_t\cdot\nabla\log\frac{\mu_t}{\rho}\,\dd y+\mathcal E_R(t),\qquad \mathcal E_R(t):=\int\Bigl(1+\log\frac{\mu_t}{\rho}\Bigr)\mu_tw_t\cdot\nabla\chi_R\,\dd y.
\]
For $R\ge1$, on the annulus supporting $\nabla\chi_R$ we have $(1+|y|^2)^2/R^2\le C(1+|y|^2)$. Hence part~\itemref{itm:ula-chain-i} and Cauchy--Schwarz give
\[
\int_\varepsilon^T|\mathcal E_R(t)|\,\dd t\le C_{\varepsilon,T}\Bigl(\int_\varepsilon^T\int|w_t|^2\mu_t\,\dd y\,\dd t\Bigr)^{1/2}\Bigl(\int_\varepsilon^T\int_{|y|>R}(1+|y|^2)\mu_t\,\dd y\,\dd t\Bigr)^{1/2}\longrightarrow0,
\]
since the first factor is finite and the second tends to zero by part~\itemref{itm:ula-chain-ii}. Integrating the localized identity between two positive times and letting $R\to\infty$, using part~\itemref{itm:ula-chain-i} at the endpoints and part~\itemref{itm:ula-chain-ii} for the interior term, gives
\[
\mathcal H(t)-\mathcal H(t_0)=\int_{t_0}^t\int\mu_qw_q\cdot\nabla\log(\mu_q/\rho)\,\dd y\,\dd q,\qquad\varepsilon\le t_0<t\le T.
\]
The integrand is integrable by Cauchy--Schwarz and part~\itemref{itm:ula-chain-ii}, so $\mathcal H$ is locally absolutely continuous and satisfies the first equality in part~\itemref{itm:ula-chain-iii}. Since $\nabla\log \rho=b/\tau$, we have $w_t=u_t-b-\tau\nabla\log(\mu_t/\rho)$, which gives the second equality.

\emph{Part~\itemref{itm:ula-chain-iv}.} Linear growth of $b$ and the moment bound in part~\itemref{itm:ula-chain-i} give $\int|b|^2\mu_t\,\dd y<\infty$. By part~\itemref{itm:ula-chain-ii} and Cauchy--Schwarz, $b\cdot\nabla\mu_t$ is integrable, and so is $\mu_t\diver b$, since $|\diver b|\le dL$. For the cutoffs $\chi_R$ from part~\itemref{itm:ula-chain-iii}, the definition of the weak divergence, applied to $\chi_R\mu_t$, which lies in $C^\infty_c(\R^d)$ by part~\itemref{itm:ula-chain-i}, gives
\[
\int\chi_R\mu_t\,\diver b\,\dd y=-\int\chi_R\,b\cdot\nabla\mu_t\,\dd y-\int\mu_t\,b\cdot\nabla\chi_R\,\dd y.
\]
As $R\to\infty$, dominated convergence gives the limits of the first two integrals, and the last integral tends to zero, because linear growth of $b$ makes $b\cdot\nabla\chi_R$ bounded uniformly in $R$, and $\nabla\chi_R=0$ on $\mathbb B_R$.
\end{proof}

\subsection{Continuity at the Initial Time}\label{app:ula-initial}

\begin{lemma}[Initial-time continuity]\label{lem:ula-endpoint}
$\lim_{t\downarrow0}\mathcal H(t)=\KL(\mu\,\|\,\rho)$.
\end{lemma}

\begin{proof}
For the upper bound, we treat the entropy and the potential term of $\mathcal H(t)$ separately. Write $\nu_t:=(F_t)_\#\mu$. For $tL<1$, $F_t$ is a bi-Lipschitz bijection, since its inverse solves $x=y-tb(x)$ by contraction. The weak Jacobian $Db$ exists almost everywhere, and $DF_t=I+tDb$ has singular values in $[1-tL,1+tL]$. By continuity in $t$, its determinant is positive, and
\[
|\log\det DF_t|\le d\,|\log(1-tL)|\le2dLt\qquad\text{for }tL\le\tfrac12.
\]
Lemma~\ref{lem:entropy-integrability}\itemref{itm:entropy-integrability-i}, applied to $\rho$, shows that the entropy integral $\operatorname{Ent}(\mu):=\int\mu\log\mu\,\dd x$ is absolutely convergent, because $\log \rho$ grows at most quadratically, $\mu$ has a finite second moment and ${\KL(\mu\,\|\,\rho)<\infty}$. The change-of-variables formula for bi-Lipschitz maps gives
\[
\operatorname{Ent}(\nu_t)=\operatorname{Ent}(\mu)-\int\log\det DF_t\,\dd\mu=\operatorname{Ent}(\mu)+O(t),
\]
and the entropy integral of the transported distribution is absolutely convergent. Convexity of $x\log x$ gives
\[
\mu_t(y)\log\mu_t(y)\le\int\nu_t(y-z)\log\nu_t(y-z)\,\phi_{2\tau t}(z)\,\dd z.
\]
The integrand on the right-hand side is absolutely integrable in $(y,z)$ because the integral defining $\operatorname{Ent}(\nu_t)$ is absolutely convergent, and the left-hand side is integrable by Lemma~\ref{lem:ula-chain}\itemref{itm:ula-chain-i}. Fubini's theorem and translation invariance give ${\operatorname{Ent}(\mu_t)\le\operatorname{Ent}(\nu_t)}$.

For the potential term, apply the segment bound in the proof of Lemma~\ref{lem:ula-chain}\itemref{itm:ula-chain-i} with $x=X_0$ and $z=X_t-X_0$. Since $\E|X_t-X_0|^2=t^2\E|b(X_0)|^2+2\tau dt$ tends to zero, the Cauchy--Schwarz inequality gives
\[
\E\big|\log \rho(X_t)-\log \rho(X_0)\big|\le\tau^{-1}\big(\E|b(X_0)|^2\big)^{1/2}\big(\E|X_t-X_0|^2\big)^{1/2}+(2\tau)^{-1}L\,\E|X_t-X_0|^2\longrightarrow0.
\]
Since $\mathcal H(t)=\operatorname{Ent}(\mu_t)-\E\log \rho(X_t)$, these two estimates give $\limsup_{t\downarrow0}\mathcal H(t)\le\KL(\mu\,\|\,\rho)$. Conversely, $X_t\to X_0$ in $L^2$, so $\mu_t$ converges weakly to $\mu$, and lower semicontinuity of relative entropy gives $\liminf_{t\downarrow0}\mathcal H(t)\ge\KL(\mu\,\|\,\rho)$.
\end{proof}

\section{Strongly Convex Action Penalties}\label{app:confinement}

We prove the extension stated in Remark~\ref{rem:confinement} by checking the four properties of the quadratic penalty used in the main analysis: the Bellman identities and value bounds, the LSI for the reference and Gibbs policies, the Lipschitz bound on the drift, and the moment and entropy estimates. Items~(a)--(d) verify these properties for a strongly convex action penalty $U$. The final paragraph checks that the fixed-point construction of Appendix~\ref{app:wellposed} also carries over.

Let $a_U$ be the minimizer of $U$. Strong convexity and the Lipschitz gradient give
\[
U(a_U)+\tfrac{m_U}2|a-a_U|^2\le U(a)\le U(a_U)+\tfrac{L_U}2|a-a_U|^2,\qquad a\in\R^d.
\]
\begin{enumerate}
\item[(a)] With $\rho_U$ and $Z_U$ in place of $\rho_\beta$ and $Z_\beta$, \eqref{eq:onestep-reward}, \eqref{eq:value-range} and the proof of Lemma~\ref{lem:soft-bellman} hold verbatim, and the proof of Lemma~\ref{lem:bellman-identities}\itemref{itm:bellman-identities-iii} uses the penalty only through~\eqref{eq:onestep-reward}. The proof of Lemma~\ref{lem:improvement}\itemref{itm:improvement-ii} and the derivation of the lower value bound~\eqref{eq:lower-preserved} in the proof of Proposition~\ref{prop:dt-stability} use the penalty only through $\int e^{-U/\tau}\,\dd a=Z_U$.
\item[(b)] By~\eqref{eq:gibbs}, $G[V](\cdot\mid s)\propto\exp\bigl((r(s,\cdot)+\gamma(PV)(s,\cdot))/\tau\bigr)\rho_U$, so Lemma~\ref{lem:gibbs-bounds}\itemref{itm:gibbs-bounds-i} holds with $\rho_U$ in place of $\rho_\beta$. By~\factref{F3} with $\Phi=U/\tau$, $\rho_U$ satisfies $\mathrm{LSI}(m_U/\tau)$, which replaces~\factref{F1} in Lemma~\ref{lem:gibbs-bounds}\itemref{itm:gibbs-bounds-ii}. Hence $m_U$ replaces $\beta$ in $\alpha$, $\lambda$, $\bar\alpha$ and $\bar\lambda$, with $\RelEnt_0$ defined relative to $\rho_U$.
\item[(c)] Since $\nabla_aQ_V=-\nabla U+\nabla_ar+\gamma\nabla_a(PV)$, the proof of Lemma~\ref{lem:drift} gives the Lipschitz bound with $L_U$ in place of $\beta$, so $L_U$ replaces $\beta$ in $L$ and $L_{\bar E}$. Since $\bar\alpha\tau\le m_U\le L_U\le L$, the bounds~\eqref{eq:step-bounds} remain valid.
\item[(d)] The quadratic bounds on $U$ give ${\int e^{m_U|a-a_U|^2/(4\tau)}\rho_U(a)\,\dd a\le(2L_U/m_U)^{d/2}}$. Hence the proof of~\eqref{eq:entropy-moment}, with $\beta|a|^2$ replaced by $m_U|a-a_U|^2$ in the truncation $\min\{\beta|a|^2/(4\tau),N\}$, gives
\begin{equation}\label{eq:general-moment}
\frac{m_U}{4\tau}\int|a-a_U|^2\,\mu(\dd a)\le\KL(\mu\,\|\,\rho_U)+\frac d2\log\frac{2L_U}{m_U},
\end{equation}
and $|U|$ and $|\log\rho_U|$ grow at most quadratically. These bounds replace~\eqref{eq:entropy-moment} and the quadratic growth of the penalty in Section~\ref{sec:mdp}, in the proof of Lemma~\ref{lem:ula}, and in Appendix~\ref{app:chain}.
\end{enumerate}
The remaining convergence arguments, including Lemmas~\ref{lem:ula} and~\ref{lem:near-optimal} and Appendix~\ref{app:ula}, use the penalty only through these properties. For $U(a)=\frac\beta2|a|^2$, the constants above reduce to those used in the paper.

To complete the well-posedness extension, we use the diffusion $\dd X_t=-\nabla U(X_t)\,\dd t+\sqrt{2\tau}\,\dd B_t$ as a reference process, which has a unique law. The density $\rho_U$ is stationary: ${\tau\nabla\rho_U+\rho_U\nabla U=0}$ weakly, and its second moment is finite, so uniqueness for the linear Fokker--Planck equation obtained from~\cite[Theorem~2.5]{trevisan2016wellposedness} identifies the stationary curve with its SDE marginals. Girsanov comparison with this stationary path law gives~\eqref{eq:prescribed-entropy} with $\KL(\cdot\,\|\,\rho_U)$ in place of $\mathcal H_\beta$. The Fisher information bound~\eqref{eq:prescribed-fisher} depends only on the Lipschitz constant of the full drift and remains valid. Since ${|\nabla U(a)|\le L_U|a-a_U|}$, equation~\eqref{eq:general-moment} implies
\[
\mathcal I(\pi^h_{t,s}\,\|\,\rho_U)
\le8\int|\nabla\sqrt{\pi^h_{t,s}}|^2\,\dd a+2\tau^{-2}\int|\nabla U|^2\,\dd\pi^h_{t,s}
\le C_T(1+t^{-1}).
\]
The relative entropy $\KL(\cdot\,\|\,\rho_U)$ is $(m_U/\tau)$-geodesically convex, so the proof of Lemma~\ref{lem:entropy-calculus} applies with potential $U/\tau$ and gives the same subgradient and chain-rule properties. Consequently the policy-evaluation estimate~\eqref{eq:weak-evaluation-stability}, contraction estimate~\eqref{eq:value-feedback-estimate}, initial bound~\eqref{eq:initial-unclipping}, and Bellman differentiation identity~\eqref{eq:weak-bellman-chain} all apply with this reference, which proves the well-posedness extension and its convergence consequences.
\bibliographystyle{plainnat}
\bibliography{references}
\end{document}